\documentclass[twoside,11pt]{article}

\usepackage{blindtext}

\usepackage[abbrvbib]{jmlr2e}

\usepackage[english]{babel}
\usepackage[normalem]{ulem}

\usepackage{enumerate}

\usepackage{amsmath}
\usepackage{mathrsfs}
\usepackage{bm}
\usepackage{amsfonts}
\usepackage{algorithm}
\usepackage{algorithmic}

\usepackage{comment}
\usepackage{color}
\DeclareMathOperator*{\Real}{Re}
\DeclareMathOperator*{\Imag}{Im}
\DeclareMathOperator*{\argmin}{arg\,min}
\DeclareMathOperator*{\argmax}{arg\,max}

\DefineNamedColor{named}{forestgreen} {rgb}{0.13,0.54,0.13}

\DefineNamedColor{named}{oceanfoam} {rgb}{0.34, 0.52, 0.54}

\DefineNamedColor{named}{burgundy}{rgb}{0.61, 0, 0}

\newcommand{\vertiii}[1]{{\left\vert\kern-0.25ex\left\vert\kern-0.25ex\left\vert #1 
    \right\vert\kern-0.25ex\right\vert\kern-0.25ex\right\vert}}

\usepackage{mathtools}
\usepackage{etoolbox}

\DeclareFontFamily{U}{matha}{\hyphenchar\font45}
\DeclareFontShape{U}{matha}{m}{n}{
<-6> matha5 <6-7> matha6 <7-8> matha7
<8-9> matha8 <9-10> matha9
<10-12> matha10 <12-> matha12
}{}
\DeclareSymbolFont{matha}{U}{matha}{m}{n}

\DeclareFontFamily{U}{mathx}{\hyphenchar\font45}
\DeclareFontShape{U}{mathx}{m}{n}{
<-6> mathx5 <6-7> mathx6 <7-8> mathx7
<8-9> mathx8 <9-10> mathx9
<10-12> mathx10 <12-> mathx12
}{}
\DeclareSymbolFont{mathx}{U}{mathx}{m}{n}

\DeclareMathDelimiter{\vvvert} {0}{matha}{"7E}{mathx}{"17}%

\DeclarePairedDelimiterX{\normIII}[1]
{\vvvert}
{\vvvert}
{\ifblank{#1}{\:\cdot\:}{#1}}

\usepackage{lastpage}
\jmlrheading{}{2025}{1-\pageref{LastPage}}{6/24; Revised 11/25}{TBD}{21-0000}{S.~Brugiapaglia, N.~Dexter, S.~Karam and W.~Wang}

\ShortHeadings{Physics-informed deep learning and compressive collocation}{S.~Brugiapaglia, N.~Dexter, S.~Karam and W.~Wang}
\firstpageno{1}

\begin{document}

\title{Physics-informed deep learning and compressive collocation for high-dimensional diffusion-reaction equations: \\practical existence theory and numerics}

\author{\name Simone Brugiapaglia \email simone.brugiapaglia@concordia.ca \\
       \addr Department of Mathematics and Statistics\\
       Concordia University\\
       Montr\'eal, QC, Canada
       \AND
       \name Nick Dexter \email nick.dexter@fsu.edu \\
       \addr Department of Scientific Computing\\
       Florida State University\\
       Tallahassee, FL, USA
       \AND
       \name Samir Karam \email samir.karam@mail.concordia.ca \\
       \addr Department of Mathematics and Statistics\\
       Concordia University\\
       Montr\'eal, QC, Canada
       \AND
       \name Weiqi Wang \email weiqiwang@uvic.ca \\
       \addr Department of Physics \& Astronomy\\
       University of Victoria\\ 
       Victoria, BC, Canada
       }

\editor{Anima Anandkumar}

\maketitle

\begin{abstract}
On the forefront of scientific computing, Deep Learning (DL), i.e., machine learning with Deep Neural Networks (DNNs), has emerged a powerful new tool for solving Partial Differential Equations (PDEs). It has been observed that DNNs are particularly well suited to weakening the effect of the \emph{curse of dimensionality}, a term coined by Richard E. Bellman in the late `50s to describe challenges such as the exponential dependence of the sample complexity, i.e., the number of samples required to solve an approximation problem, on the dimension of the ambient space. 
However, although DNNs have been used to solve PDEs since the `90s, the literature underpinning their mathematical efficiency in terms of numerical analysis (i.e., stability, accuracy, and sample complexity), is only recently beginning to emerge. In this paper, we leverage recent advancements in function approximation using sparsity-based techniques and random sampling to develop and analyze an efficient high-dimensional PDE solver based on DL. We show, both theoretically and numerically, that it can compete with a novel stable and accurate compressive spectral collocation method for the solution of high-dimensional, steady-state diffusion-reaction equations with periodic boundary conditions. In particular, we demonstrate a new \emph{practical existence theorem}, which establishes the existence of a class of trainable DNNs with suitable bounds on the network architecture and a sufficient condition on the sample complexity, with logarithmic or, at worst, linear scaling in dimension, such that the resulting networks stably and accurately approximate a diffusion-reaction PDE with high probability.   
\end{abstract}


\section{Introduction}

PDEs over high-dimensional domains are a powerful mathematical modelling tool adopted in a variety of applications including molecular dynamics, computational finance, optimal control, and statistical mechanics. Important high-dimensional Partial Differential Equation (PDE) models in these areas are the many-electron Schr\"odinger equation, the Hamilton–Jacobi–Bellman equation, the Fokker-Planck equation and the Black-Scholes model. Analytic solutions to these equations are in general not available and, hence, it is necessary to design efficient numerical PDE solvers to approximate their solutions. A crucial challenge that immediately arises in this context is the \emph{curse of dimensionality}, see \cite{bellman1957dynamic, bellman1961adaptive}. This refers to the tendency of numerical methods for solving high-dimensional problems to exhibit a computational cost or require an amount of data that scales exponentially with the problem's dimension.

Recent work has shown that \emph{compressive sensing} and \emph{Deep Learning (DL)} are promising techniques to develop efficient high-dimensional PDE solvers and lessen the curse. This success is part of a larger research trend in the area of \emph{scientific machine learning} \citep{baker2019workshop}, where state-of-the-art techniques from machine learning are applied to solve challenging scientific computing problems, including the numerical solution of PDEs. One of the most popular recent examples in this area are \emph{Physics-Informed Neural Networks (PINNs)}, see \cite{Raissi2019} and earlier studies on the topic by \cite{Lagaris1998}, which recently gained an impressive amount of attention in the scientific computing community. In particular, DL based methods have shown great promise for high-dimensional PDEs, see \cite{han2018solving} and the recent review paper by \cite{weinan2021algorithms}, and PDEs on domains with complex geometries, see, e.g., \cite{chen2022bridging}. 

Concurrently, recent advancements involving compressive sensing in scientific computing include the adoption of sparsity-based techniques for function approximation from random samples, see \cite{rauhut2012sparse}, whose initial success was due to their application in the field of \emph{Uncertainty Quantification (UQ)} of parametric PDEs by \cite{doostan2011non}, see also \cite{adcock2022sparse} for a comprehensive review of the topic. Here we focus on \emph{Compressive Fourier Collocation} (\emph{CFC}), a method proposed in \cite{wang2022compressive} as an improvement of the compressive spectral collocation method from \cite{brugiapaglia2020compressivespectral} and able to lessen the curse of dimensionality in the number of collocation points. A detailed literature review on compressive sensing and DL methods for PDEs can be found in \S\ref{sec:literature}.

Motivated by these recent advances, in this paper we study and compare numerical solvers for high-dimensional PDEs based on compressive sensing and DL from both the theoretical and the numerical viewpoint, focusing on steady-state diffusion-reaction equations with periodic boundary conditions. The choice of the time-independent setting introduces a gap with real-world applications. Yet, it allows us to carry out a rigorous analysis while keeping the level of technical difficulty moderate and focus on the impact of the physical domain's dimension, which is the main subject of this work. Our methodological approach is inspired by the recent paper by \cite{adcock2021gap}, where a similar practical and theoretical study was made in the context of high-dimensional function approximation from pointwise samples.

\subsection{Main contributions}
Our main contributions, of both theoretical and computational nature, are summarized below.

\paragraph{i) New convergence theorem for high-dimensional periodic PINNs.} The main theoretical contribution of the paper is a new convergence theorem for periodic PINNs applied to (possibly high-dimensional) steady-state diffusion-reaction problems with periodic boundary conditions. Our theoretical guarantee is a convergence result in the form of \emph{practical existence theorem}, a theoretical approach recently developed in the context of scalar-valued function approximation via DL in \cite{adcock2021gap} and Hilbert-valued function approximation relevant to parametric PDEs and UQ in \cite{adcock2022deep}. In contrast with better-known universal approximation theorems (see, e.g., \cite{elbrachter2021} and references therein) that prove the existence of networks with suitable approximation properties, a \emph{practical} existence theorem also addresses the issues of network training and sample complexity. 
Our result, stated in Theorem~\ref{thm:PET} and proved in Section~\ref{sec:proofs}, establishes the existence of a class of trainable periodic PINNs with explicit architecture bounds and shows that networks in this class can achieve near-optimal approximation rates for PDE solutions that are sparse with respect to the Fourier basis through training using a number of samples that is only mildly affected by the curse of dimensionality (i.e., that scales logarithmically or, at worst, linearly with the domain's dimension $d$). Despite the presence of the adjective ``practical'' in the terminology used to refer to Theorem~\ref{thm:PET}, there are still gaps between this result and common implementation practice. For example, Theorem~\ref{thm:PET} is a result about networks obtained as minimizers of a certain regularized training loss. However, it does not involve any specific optimization method that may be used to approximate such minimizers (e.g., stochastic gradient descent or Adam). Further limitations of Theorem~\ref{thm:PET} are discussed extensively in \S\ref{sec:conclusion}.

\paragraph{ii) Numerical study of CFC and periodic PINNs in high dimensions.} Our second contribution is of computational nature and it is the implementation and numerical study of CFC and periodic PINNs for high-dimensional, steady-state diffusion-reaction problems on the torus $\mathbb{T}^d$, up to dimension $d=30$. We also compare these two methods by studying their accuracy as a function of the training set size (i.e., the number of collocation points). The code to reproduce our experiments can be found in the GitHub repository \url{https://github.com/WeiqiWangMath/PINN_high_dimensional_PDE}.

\paragraph{iii) Improved implementation of CFC.}  We propose a new variant of CFC, based on  \emph{adaptive lower Orthogonal Matching Pursuit (OMP)} recovery (see Algorithm~\ref{algo:lower_OMP}). This variant improves the accuracy of the CFC method from \cite{wang2022compressive} under suitable structural assumptions on the PDE solution and allows the method to scale in higher dimensions. This is due to the way it constructs the approximation support set by adaptively exploring the set of possible candidates to be added in an iterative procedure, thereby avoiding dealing with a large \emph{a priori} truncation set. 

\paragraph{iv) CFC convergence theorem for diffusion-reaction problems.} We show the convergence of CFC for steady-state, periodic diffusion-reaction problems in Theorem~\ref{thm:CFC}. Although this is mainly an auxiliary step needed to prove our main theoretical result (Theorem~\ref{thm:PET}) it is an extension of the analysis in \cite{wang2022compressive} (relying on the same techniques employed therein) of independent interest.

\subsection{Contextualization of our contributions}
\label{sec:literature}

Numerical methods for PDEs based on compressive sensing and sparse recovery have been considered in tandem with different discretization techniques. This includes Galerkin \citep{jokar2010sparse}, Petrov-Galerkin \citep{brugiapaglia2021wavelet,brugiapaglia2015compressed}, Fourier-Galerkin \citep{gross2023sparse}, isogeometric analysis \citep{brugiapaglia2020compressive}, spectral collocation approaches  \citep{brugiapaglia2020compressivespectral, wang2022compressive} and methods based on the sparse Fourier transform \citep{daubechies2007sparse}. In this paper, we are interested in methods that can be applied to high-dimensional domains. In particular, we focus on \emph{Compressive Fourier Collocation} (\emph{CFC}) \citep[see][]{wang2022compressive}, which aims to compute a sparse approximation to the PDE solution with respect to the Fourier basis from random collocation points. This method will be reviewed in detail in \S\ref{sec:CFC}. The Wavelet-Fourier CORSING method \citep{brugiapaglia2021wavelet} can in principle be implemented in domains in arbitrary dimension, although making practical numerical implementations scale in high dimensions is nontrivial. The sparse spectral method proposed in \cite{gross2023sparse}, based on Fourier-Galerkin discretization and sublinear time algorithms, can scale to extremely high-dimensional problems. 

The design of numerical PDE solvers based on neural networks dates back to the 1990s \citep{Lagaris1998,lee1990neural}. More recently, this field became extremely popular thanks to the introduction of \emph{Physics-Informed Neural Networks (PINNs)} \citep[see][]{Raissi2019,karniadakis2021physics}. Solvers based on DL have shown great promise specifically in the case of high-dimensional PDEs. In this direction, approaches proposed in the literature include the Deep Galerkin Method \citep{sirignano2018dgm}, methods based on the reformulation of high-dimensional PDEs as Stochastic Differential Equations (SDEs) \citep{han2017deep, han2018solving}, PINNs for high-dimensional problems \citep{hu2023tackling,zeng2022adaptive}, and deep genetic algorithms \citep{putri2024deep}. In this paper, we will consider PINNs combined with a periodic layer \citep{dong2021method} to solve high-dimensional PDEs on the $d$-dimensional torus. This approach will be presented in detail in \S\ref{sec:PINNs}.

Other approaches for the numerical solution of PDEs able to scale to moderately high dimensions include sparse grid methods (see, e.g., \cite{shen2010efficient, shen2012efficient}), methods based on tensor-based approximation and low rank structures \citep{bachmayr2015adaptive, bachmayr2016tensor, dahmen2016tensor}, and  sparse grid spectral methods \citep{kupka1997sparse}. 

Our main theoretical result (Theorem~\ref{thm:PET}) is a convergence theorem for PINNs over high-dimensional periodic domains. Currently, the convergence analysis of PINNs is an active research area and several studies have appeared in the literature. However, the theory is arguably far from being fully developed. The analysis in \cite{shin2020convergence} shows the convergence of PINNs for linear second-order elliptic problems, but it is based on H\"older-type regularization that is in general not implementable. The study in \cite{shin2023error} provides \emph{asymptotic} convergence results (i.e., with training set size $m\to \infty$) relying on assumptions in terms of Bernstein-type inequalities or Rademacher's complexity that could be challenging to verify in practice. The analysis in \cite{doumeche2023convergence} shows convergence results for PINNs of asymptotic type and nonasymptotic convergence rates for the expected squared $L^2$-error of the form $\log^2(m)/m^{1/2}$ for Sobolev-type regularized loss functions, where $m$ is the number of training data points used to collocate the PDE in the PINN's loss function. The work \citep{de2022error} proves generalization bounds for PINNs for Kolmogorov equations showing that PINNs can lessen the curse of dimensionality, under the assumption that the network's weights are bounded.
Compared to these results, our practical existence theorem (Theorem~\ref{thm:PET}) has the advantage of being of nonasymptotic type and the corresponding error bound leads to fast convergence rates with respect to the size of the training set when the PDE solution is sparse or compressible with respect to the Fourier basis.

Finally, the framework of practical existence theorems developed in \cite{adcock2021gap, adcock2022deep} for scalar- and Hilbert-valued high-dimensional function approximation was recently extended to Banach-valued function approximation in \cite{adcock2023near} and reduced-order modelling of parametric PDEs based on convolutional autoencoders in \cite{franco2024practical}. This framework relies on recently proposed convergence results for compressive sensing-based approximation in high dimensions by \cite{adcock2022sparse}. For a review of practical existence theory, see \cite{adcock2024learning}. All the practical existence theorems developed so far rely on the emulation of orthogonal polynomials with neural networks \citep[see][]{daws2019analysis,de-ryck2021approximation,opschoor2022exponential} and are only applied to approximating Hilbert or Banach-valued functions relevant to parametric PDEs or scalar-valued functions. In this paper, we extend the scope of practical existence theorems by emulating Fourier basis functions and considering the case of (non-parametric) PDE solvers.

\subsection{Outline of the paper}

We briefly outline the organization of the paper. We start by illustrating the model problem (a high-dimensional periodic diffusion-reaction equation), periodic PINNs and our main theoretical result (Theorem~\ref{thm:PET}) in \S\ref{sec:problem_setting}. Then, we present CFC and  adaptive lower OMP in \S\ref{sec:CFC}.  \S\ref{sec:numerics} contains an extensive numerical study of CFC and periodic PINNs. The proof of our convergence result is presented in \S\ref{sec:proofs}. Finally, we draw some conclusions and describe possible directions of future work in \S\ref{sec:conclusion}. Additional experiments investigating the numerical impact of certain components of Theorem~\ref{thm:PET} are presented in Appendix~\ref{sec:additional_numerics}.

\section{Problem setting}
\label{sec:problem_setting}

\paragraph{Notation.} We start by recalling some standard mathematical notation that will be employed throughout the paper. We denote the $d$-dimensional torus by $\mathbb{T}^d$, with $d \in \mathbb{N}$, where $\mathbb{T}:=[0,1]/\sim$ and $\sim$ is the equivalence relation on $[0,1]$ defined by $x\sim y$ if and only if $x-y \in \mathbb{Z}$. $L^2(\mathbb{T}^d)$ denotes the space of square-integrable functions, equipped with inner product $\langle v,w\rangle:= \int_{\mathbb{T}^d} v(\bm{x}) \overline{w(\bm{x})}\, \mathrm{d}\bm{x}$ and norm $\|v\|_{L^2} = \langle v, v \rangle^{1/2}$. Moreover, $L^\infty(\mathbb{T}^d)$ denotes the space of functions such that $\|v\|_{L^{\infty}}:=\text{ess}\sup_{\bm{x} \in \mathbb{T}^d}|v(\bm{x})| <\infty$. We also consider the Sobolev spaces $H^k(\mathbb{T}^d)$, $k =1,2$ equipped with norms $\|v\|_{H^1}:= (\|v\|_{L^2}^2 + \|\nabla v\|_{L^2}^2)^{1/2}$ and $\|v\|_{H^2}:= (\|v\|_{H^1}^2 + \|\nabla^2 v\|_{L^2}^2)^{1/2}$, respectively. Here $\nabla$ and $\nabla^2$ denote the gradient and the Hessian operators, respectively. Moreover, $\|\nabla v\|_{L^2}^2 = \int_{\mathbb{T}^d} \|\nabla v(\bm{x})\|_2^2 \, \mathrm{d}\bm{x}$ and $\|\nabla^2 v\|_{L^2}^2 = \int_{\mathbb{T}^d} \|\nabla^2 v(\bm{x})\|_F^2 \, \mathrm{d}\bm{x}$, where $\|\cdot\|_2$ is the discrete $2$-norm and $\|\cdot\|_{F}$ is the Frobenius norm. We will also consider the Sobolev spaces $W^{k,\infty}(\mathbb{T}^d)$, $k=1, 2$, equipped with norms $\|v\|_{W^{1,\infty}}:=\max\{\|v\|_{L^\infty}, \|\nabla v\|_{L^\infty}\}$ and $\|v\|_{W^{2,\infty}}:=\max\{\|v\|_{W^{1,\infty}}, \|\nabla^2 v\|_{L^\infty}\}$, respectively. Here $\|\nabla v\|_{L^\infty} = \mathrm{ess}\sup_{\bm{x} \in \mathbb{T}^d} \|\nabla v(\bm{x})\|_{\infty}$ and $\|\nabla^2 v\|_{L^\infty} = \mathrm{ess}\sup_{\bm{x} \in \mathbb{T}^d} \max_{i,j}|(\nabla^2 v(\bm{x}))_{ij}|$. $C^k(\mathbb{T}^d)$ denotes the space of $k$-times continuously differentiable functions over $\mathbb{T}^d$. We denote the set of first $n$ positive integers by $[n]:=\{1,\ldots,n\}$. The cardinality of a set $X$ is denoted as $|X|$. For a vector $\bm{z} \in\mathbb{C}^N$, we define its support as $\mathrm{supp}(\bm{z}) = \{j \in [N] : z_j \neq 0\}$.

\paragraph{Model problem.} Our model problem is a periodic diffusion-reaction equation over the $d$-dimensional torus $\mathbb{T}^d$. 
In this paper, we are interested in the scenario where $d \gg 1$. Moreover, we consider the following steady-state, periodic diffusion-reaction equation as our model problem:
\begin{equation}
- \nabla \cdot (a(\bm{x}) \nabla u(\bm{x})) + \rho u(\bm{x}) = f(\bm{x}), \quad \forall \bm{x} \in \mathbb{T}^d, \label{eq:diffusion_eq_periodic_1}
\end{equation}
where $u:\mathbb{T}^d \to \mathbb{R}$ is the PDE solution, and the diffusion coefficient $a:\mathbb{T}^d \to \mathbb{R}$, the reaction term $\rho \in \mathbb{R}$ and the forcing term $f:\mathbb{T}^d \to \mathbb{R}$ are assumed to satisfy 
\begin{equation}
\label{eq:suff_cond_a_rho_f} a \in C^1(\mathbb{T}^d), \quad
\min_{\bm{x} \in \mathbb{T}^d}a(\bm{x}) \geq a_{\min} >0, \quad 
\rho >0, \quad \text{and} \quad f \in L^2(\mathbb{T}^d). 
\end{equation} 
These conditions are sufficient for the problem \eqref{eq:diffusion_eq_periodic_1} to be well-posed (see, e.g., \cite[Proposition 2.1]{brugiapaglia2021wavelet}) and guarantee that its weak solutions belong to the Sobolev space $H^2(\mathbb{T}^d)$ (see, e.g., \cite[\S6.3]{evans2010partial}). Throughout the paper, we also assume $f$ to be regular enough for pointwise evaluations to be well-defined. Moreover, we define the PDE operator
\begin{equation}
\label{eq:def_operator_L}
\mathscr{L}[u](\bm{x}) := - \nabla \cdot (a(\bm{x}) \nabla u(\bm{x})) + \rho u(\bm{x}), \quad \forall \bm{x} \in \mathbb{T}^d, \quad \forall u \in H^2(\mathbb{T}^d).
\end{equation}
Despite its simplicity, \eqref{eq:diffusion_eq_periodic_1} is an interesting model problem since it shares the same second-order diffusion term with more complex PDEs such as the Black-Scholes, Schr\"odinger and Fokker-Planck models.

\subsection{Physics-Informed Neural Networks (PINNs) with periodic layer}
\label{sec:PINNs}

We start by illustrating the PINN setting adopted to solve the periodic high-dimensional diffusion-reaction problem \eqref{eq:diffusion_eq_periodic_1}.
First, we describe the Deep Neural Network (DNN) architecture employed in the method, then illustrate the training strategy. The framework presented here will encompass both our main theoretical result (presented in \S\ref{sec:main_result}) and the numerical experiments in \S\ref{sec:numerics}.

To enforce periodic boundary conditions, we consider an approach proposed in \cite{dong2021method}. This is achieved by adding a $C^{\infty}$ \emph{periodic layer}  to the DNN as the first layer. The periodic layer contains, in turn, two layers of width $dl$, for some $l\in \mathbb{N}$, denoted as $\bm{q}^{(1)} = (q_{ij}^{(1)})_{i \in [d], j \in [l]}$ and $\bm{v}^{(2)} = (v_{ij}^{(2)})_{i \in [d], j \in [l]}$, respectively. The neuron $q_{ij}^{(1)}$ operates on the $i$th component of the input vector $\bm{x} \in \mathbb{R}^d$ and applies a cosine transformation to enforce periodicity while adding a phase shift parameter $\phi_{ij}$:
\begin{equation}
\label{eq:def_periodic_layer}
q_{ij}^{(1)}(\bm{x}) = \cos(2\pi x_i+\phi_{ij}), \quad \forall \bm{x} \in \mathbb{R}^{d}, \quad \forall i \in [d], \forall j\in [l],
\end{equation}
where $\bm{\phi} = (\phi_{ij})_{i \in [d], j \in [l]} \in \mathbb{R}^{dl}$ are trainable parameters and
$l$ controls the number of neurons per dimension.  
Then, the neurons in $\bm{v}^{(2)}$ collect the outputs of the corresponding nodes in $\bm{q}^{(1)}$, and apply an affine transformation and (possibly nonlinear) activation function $\bm{\sigma}^{(2)}:\mathbb{R}^{dl} \to \mathbb{R}^{dl}$, i.e., 
\begin{equation}
\bm{v}^{(2)}(\bm{x}) = \bm{\sigma}^{(2)}\left(\text{diag}(\bm{w}^{(2)}) \bm{x}+ \bm{b}^{(2)}\right), \qquad \forall \bm{x} \in \mathbb{R}^{dl},
\end{equation}
where $\bm{w}^{(2)}$ and $\bm{b}^{(2)}$ are $dl$-dimensional vectors of trainable parameters and $\text{diag}(\bm{w}^{(2)})$ is a $dl \times dl$ matrix with $\bm{w}^{(2)}$ on the main diagonal and zeros elsewhere. 
After the periodic layer, the DNN has $h \in \mathbb{N}$ traditional hidden layers having width $w \in \mathbb{N}$. Each of these takes the output of the previous layer as input and applies a trainable affine transformation and a componentwise activation function $\sigma$ to it. These hidden layers are denoted by 
\begin{equation}
\label{eq:def_hidden_layer}
\bm{v}^{(k)}(\bm{x}) = \bm{\sigma}^{(k)}(W^{(k)} \bm{x} + \bm{b}^{(k)}), \quad \forall\bm{x} \in\begin{cases}\mathbb{R}^{dl} & \text{if } k =3\\\mathbb{R}^{w} & \text{otherwise}\end{cases}, \quad\forall k = 3, \ldots, h+2,
\end{equation}
where the weight matrix $W^{(k)}$ is $w\times dl$ if $k=3$ and $w\times w$ otherwise, $\bm{b}^{(k)}\in\mathbb{R}^w$ and $\bm{\sigma}^{(k)}:\mathbb{R}^w \to \mathbb{R}^w$. \S\ref{Sec:architecture} discusses the selection of the hyper-parameters $l$ (number of neurons per dimension in the periodic layer), $h$ and $w$ defining the DNN architecture. 
The last hidden layer of the network $\bm{v}^{(h+2)}$ activates linearly into one output neuron so that the output of the network is scalar-valued. That is, 
$$
v^{(h+3)}(\bm{x})=W^{(h+3)}\bm{x},\quad\forall\bm{x}\in\mathbb{R}^{w},
$$
where $W^{(h+3)}$ is $1 \times w$. In this paper, we consider componentwise activations $\bm{\sigma}^{(k)}$ of the form $\bm{\sigma}^{(k)}(\bm{x}) = (\sigma^{(k)}_j(x_j))$, where $\sigma^{(k)}_j$ could be either a linear activation (i.e., $\sigma^{(k)}_j(x)=x$) or a nonlinear activation of the following two types: the hyperbolic tangent (i.e., $\sigma^{(k)}_j(x)=\tanh(x)$) or the \emph{Rectified Power Unit (RePU)}, defined by
\begin{equation}
\label{eq:def_RePU}
\textnormal{RePU}_\ell(x) := \max\{0, x^\ell\},\quad \ell\in \mathbb{N}.
\end{equation}
Note that when $\ell=1$ the function $\textnormal{RePU}_\ell$ is  the Rectified Linear Unit (ReLU). In our numerical experiments (see \S\ref{sec:numerics}) we will consider tanh activations. However, RePU and linear activations will be used to derive our theoretical result (see \S\ref{sec:main_result}).

In summary, we consider DNNs $\psi:\mathbb{R}^d \to \mathbb{R}$ of the form
$$
\psi = v^{(h+3)} \circ  \bm{v}^{(h+2)}\cdots \circ\bm{v}^{(2)} \circ \bm{q}^{(1)}.
$$ 
Fig.\ref{Fig:NN_architecture} depicts the architecture of a DNN with periodic layer in dimension $d=2$.
\begin{figure}[!t]
\centering
\includegraphics[width=9 cm]{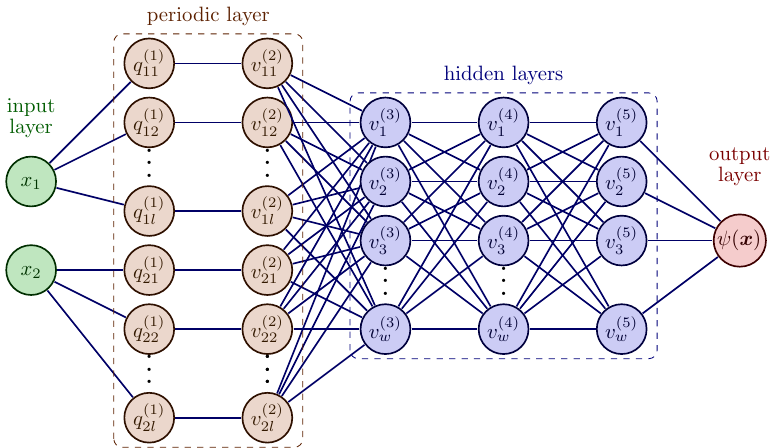}
\caption{Architecture of the neural network with the periodic layer ($d=2$). The number in superscript represents the layer number. In this case, $h = 3$ and depth$(\psi) = 7$. }  \label{Fig:NN_architecture}
\end{figure}
We also define the depth and width of the network as
\begin{equation}
\label{eq:def_width_depth}
\text{width}(\psi) = \max\{dl, w\}
\quad \text{and}\quad
\text{depth}(\psi) = h+4,
\end{equation}
where we included the input and the output layer in the depth count.

The DNN $\psi(\bm{x})$ is then trained to approximate the solution $u(\bm{x})$ of the high-dimensional PDE. Given collocation points drawn independently and uniformly at random from $\mathbb{T}^d$, i.e., 
\begin{equation}
\label{eq:def_collocation_points}
\bm{x}_1,\ldots, \bm{x}_m \stackrel{\text{i.i.d.}}{\sim} \text{Uniform}(\mathbb{T}^d),
\end{equation}
the parameters $\bm{\phi}$, $\bm{w}^{(2)}$, $W^{(k)}$, $\bm{b}^{(k)}$ of the DNN are learned by minimizing a regularized \emph{Root Mean Squared Error (RMSE)} loss, i.e., 
\begin{equation}
\label{eq:def_loss}
\min_{\psi} \sqrt{\frac{1}{m} \sum_{i=1}^{m} | \mathscr{L}[\psi](\bm{x}_i) - f(\bm{x}_i) |^2} + \lambda \mathcal{R}(\psi),
\end{equation}
where $\mathscr{L}$ is the PDE operator defined in \eqref{eq:def_operator_L}, $\lambda \geq 0$ is a tuning parameter, and $\mathcal{R}$ is a regularization term that usually involves the $\ell^2$- or $\ell^1$-norm of the networks' weights (corresponding to the \emph{weight decay} or \emph{sparse regularization} strategies, respectively). Minimizing the loss function \eqref{eq:def_loss} corresponds to finding the network $\psi$ that minimizes the PDE residual at the collocation points in a (regularized) least-squares sense. In the general PINN setting for solving stationary PDEs, the loss function usually consists of the sum of two components: the first one aims at minimizing the PDE residual (like in \eqref{eq:def_loss}) and the second component enforces boundary conditions. However, in our setting the periodic layer forces $\psi$ to be periodic, hence automatically enforcing boundary conditions. For this reason, the training loss does not contain a boundary condition term. The loss is then minimized by stochastic gradient descent methods. More technical details on the training procedure can be found in \S\ref{Sec:NN_performance}. 
We conclude by noting that other losses different from the RMSE can be considered for the PDE residual minimization. These include losses based on, e.g., $L^p$-norms \citep{wang20222} or Sobolev norms \citep{son2021sobolev}. Here we limit our attention to the regularized RMSE loss for the sake of simplicity and also because our main theoretical result holds for this loss. The presence of the regularization term $\mathcal{R}$ will be crucial for our convergence theorem, illustrated in the next subsection. In our numerical results we will train by simply minimizing the unregularized RMSE (or, equivalently, unregularized MSE) loss.

\subsection{A practical existence theorem for periodic PINNs}
\label{sec:main_result}

Before presenting our main result, namely a convergence theorem for periodic PINNs based on the framework of practical existence theory (Theorem~\ref{thm:PET}), we  need to introduce some definitions and further technical ingredients. In a nutshell, our main result shows that trained periodic PINNs are able to achieve the same accuracy as a sparse Fourier approximation of the PDE solution using a training set of collocation points whose size scales logarithmically or, at worst, linearly with the dimension $d$. This is a natural choice given the presence of periodic boundary conditions in \eqref{eq:diffusion_eq_periodic_1}. In addition, we will require some technical conditions on the PDE coefficients $a$ and $\rho$.

\paragraph{Target accuracy: sparse Fourier approximation.} The elements of the $L^2$-orthonormal Fourier basis are defined as
\begin{equation}
\label{eq:Fourier_system}
F_{\bm{\nu}}(\bm{x}) = \exp(2 \pi \mathrm{i} \, \bm{\nu} \cdot \bm{x}), \quad \forall \bm{\nu} \in \mathbb{Z}^d, \; \forall \bm{x} \in \mathbb{T}^d.
\end{equation}
In particular, we will focus on \emph{sparse} Fourier approximations supported on \emph{hyperbolic crosses} (see, e.g., \cite{dung2018hyperbolic,temlyakov2018multivariate} and references therein). The \emph{hyperbolic cross} of order $n$ is a multi-index set of $\mathbb{Z}^d$ defined as 
\begin{equation}
\label{eq:def_HC}
\Lambda^{\mathrm{HC}}_{d,n} = \left\{\bm{\nu}\in \mathbb{Z}^d : \prod_{k = 1}^d (|\nu_k| +1) \leq n  \right\}.
\end{equation}
The hyperbolic cross is a convenient choice in high-dimensional approximation since its cardinality grows moderately with respect to $d$ and $n$ when compared to other standard multi-index set choices such as the \emph{tensor product} and the \emph{total degree} sets, see, e.g., \cite[\S2.3]{adcock2022sparse}. Specifically, the following upper bound on the cardinality of the hyperbolic cross holds (see \cite{kuhn2015approximation,chernov2016new}, and also Eq.~(2.10) in \citep{wang2022compressive}):
\begin{equation}
    \label{eq:card_bound_HC}
    |\Lambda^{\mathrm{HC}}_{d,n}|\leq\min\{4n^516^d,e^2n^{2+\log_2(d)}\}.
\end{equation}

In order to leverage the CFC convergence theory needed for our practical existence theorem, we consider a rescaled version of the system $\{F_{\bm{\nu}}\}_{\bm{\nu} \in \mathbb{Z}^d}$, defined by
\begin{equation}
\label{eq:def_Psi}
\Psi_{\bm{\nu}} = \frac{1}{4\pi^2 \|\bm{\nu}\|_2^2+\rho/a_{\bm{0}}} F_{\bm{\nu}}, \quad \forall \bm{\nu} \in \mathbb{Z}^d,
\end{equation}
where $a_{\bm{0}} = \langle a, F_{\bm{0}} \rangle = \int_{\mathbb{T}^d} a(\bm{x})\,\mathrm{d}\bm{x}$. As explained in detail in \S\ref{sec:CFC_theory}, this rescaling ensures that $\{\mathscr{L}[\Psi_{\bm{\nu}}]\}_{\bm{\nu} \in \mathbb{Z}^d}$ is a \emph{bounded Riesz system} \citep{brugiapaglia2021sparse} of $L^2(\mathbb{T}^d)$ under sufficient conditions on $a$ and $\rho$, see \eqref{eq:diff_expansion} and \eqref{eq:suff_cond_PET} below, which is a crucial property needed to show the convergence of the CFC method. To gain some intuition about this rescaling, consider the simple case of constant diffusion $a \equiv 1$. Applying the PDE operator to the rescaled Fourier system yields 
$\mathscr{L}[\Psi_{\bm{\nu}}] = (-\Delta+\rho) \Psi_{\bm{\nu}} = F_{\bm{\nu}}$ for all $\bm{\nu} \in \mathbb{Z}^d$. In this simple scenario, $\{\mathscr{L}[\Psi_{\bm{\nu}}]\}_{\bm{\nu} \in \mathbb{Z}^d}$ is a bounded orthonormal system and, as such, is ideally suited for compressive sensing, \citep[see][\S12]{foucart2013mathematical}.

In this setting, we consider a finite-dimensional truncation of the solution $u$ to a finite multi-index set $\Lambda \subseteq \mathbb{Z}^d$ (expressed with respect to the rescaled Fourier basis), i.e.,
\begin{equation}
\label{eq:def_u_Lambda}
u_\Lambda = \sum_{\bm{\nu}\in\Lambda}c_{\bm{\nu}}\Psi_{\bm{\nu}}, \quad \text{with } c_{\bm{\nu}} = (4\pi^2 \|\bm{\nu}\|_2^2+\rho/a_{\bm{0}})\cdot\langle u, F_{\bm{\nu}}\rangle,
\end{equation}
and let $\bm{c}_{\Lambda} =  (c_{\bm{\nu}})_{\bm{\nu} \in \Lambda} \in \mathbb{C}^N$, where $N = |\Lambda|$. An \emph{$s$-sparse} approximation to $u$ is then obtained by only keeping $s$ terms in the expansion \eqref{eq:def_u_Lambda}. In general, we recall that a vector is said to be $s$-sparse if it has at most $s$ nonzero entries. The best possible accuracy of such an approximation is measured by the \emph{best $s$-term approximation error}, see \citep{cohen2009compressed} and references therein, defined as
$$
\sigma_s(\bm{c}_{\Lambda})_p=\min_{\bm{z}\in \mathbb{C}^N}\left\{\|\bm{c}_{\Lambda}-\bm{z}\|_p : \bm{z} \text{ is $s$-sparse}\right\}.
$$

\paragraph{Sufficient conditions on the PDE coefficients.} 
We also introduce a technical condition on the PDE coefficients $a$ and $\rho$ necessary for the convergence result.
We consider diffusion coefficients $a:\mathbb{T}^d \to \mathbb{R}$ having a sparse Fourier expansion, namely,
\begin{align}
\label{eq:diff_expansion}
a = a_{\bm{0}} + \sum_{\bm{\nu} \in T} a_{\bm{\nu}} F_{\bm{\nu}}, \quad \text{with } a_{\bm{\nu}} = \langle a, F_{\bm{\nu}}\rangle, \quad \forall \bm{\nu} \in T \cup \{\bm{0}\},
\end{align} 
for some $T \subseteq \mathbb{Z}^d \setminus \{\bm{0}\}$. In our main result, we will assume that $a$ and $\rho$ satisfy the following assumption:
\begin{align}
\label{eq:suff_cond_PET}
    &\sqrt{|T|} \cdot \|a - a_{\bm{0}}\|_{H^1} < 
    \sqrt{\left(a_{\bm{0}}+ \frac{\rho}{4\pi^2}\right)^2 + a_{\bm{0}}^2} -  \left(a_{\bm{0}} + \frac{\rho}{4\pi^2}\right).
\end{align}
Note that in the case of constant diffusion (i.e., $a\equiv a_{\bm{0}}$) the condition above is always satisfied under \eqref{eq:suff_cond_a_rho_f}. In general, condition \eqref{eq:suff_cond_PET} controls the oscillatory behaviour of $a$. We emphasize that the sparsity of $a$ and condition \eqref{eq:suff_cond_PET} are sufficient for our convergence theorem to hold, but we do not claim (nor believe) they are necessary. We also observe that our analysis does not take into account the potential sample complexity needed to estimate the term $a_{\bm{0}} = \int_{\mathbb{T}^d} a(\bm{x})\,\mathrm{d}\bm{x}$. In order to check the validity of condition \eqref{eq:suff_cond_PET}, one needs direct access to this quantity. If it is not already available, the extra samples needed to compute it via, e.g., Monte Carlo quadrature or other high-dimensional integration methods would have to be taken into account. \\

We are now in a position to state our main result, which provides the existence of a class of trained neural networks with architecture of the form described in \S\ref{sec:PINNs} able to approximate the solution to \eqref{eq:diffusion_eq_periodic_1} with accuracy comparable to that of an $s$-sparse Fourier approximation with high probability.

\begin{theorem}[Practical existence theorem for high-dimensional periodic PINNs]
\label{thm:PET}
Given a dimension $d\in \mathbb{N}$, target sparsity $s\in\mathbb{N}$, hyperbolic cross order $n\in \mathbb{N}$, RePU power $\ell \in \mathbb{N}$, with $\ell \geq 2$, and probability of failure $\varepsilon \in (0,1)$, there exist:
\begin{enumerate}[(i)]
    \item a class of neural networks $\mathcal{N}$ of the form  described in \S\ref{sec:PINNs} with $d$-dimensional input and $1$-dimensional output layers, $\textnormal{RePU}_\ell$ or linear activations, complex-valued weights and biases, and such that, for all $\psi \in \mathcal{N}$, 
\begin{align}
\label{eq:PET_width_bound}
    \textnormal{width}(\psi) & \leq c_{\ell}^{(1)} \cdot \min\left\{4n^516^d,e^2n^{2+\log_2d}\right\} \cdot d \cdot \min\{2^d, n\},\\
    \label{eq:PET_depth_bound}
    \textnormal{depth}(\psi) & \leq c^{(2)}\cdot \left(\log_2(n)+\min\{\log_2 d, n\}\right),  
\end{align}
\item a regularization function $\mathcal{R} : \mathcal{N} \to [0,\infty)$, 
\item a choice of tuning parameter $\lambda$ depending only on $a$, $\rho$ and $s$,
\end{enumerate}
such that the following holds with probability $1-\varepsilon$. For all $a:\mathbb{T}^d \to \mathbb{R}$ and $\rho \in \mathbb{R}$ satisfying \eqref{eq:suff_cond_a_rho_f}, \eqref{eq:diff_expansion} and \eqref{eq:suff_cond_PET}, let 
\begin{equation}
\label{eq:PET_sample_complexity}
    m\geq c^{(3)}_{a,\rho} \cdot s \cdot \log^2\left(c^{(4)}_{a,\rho}\cdot s\right)\cdot \left(\min\{\log(n)+d,\log(2n)\log(2d)\}+\log(\varepsilon^{-1})\right),
\end{equation}
$\Lambda =\Lambda^{\mathrm{HC}}_{d,n}$ as in \eqref{eq:def_HC}, and consider collocation points $\bm{x}_1, \ldots \bm{x}_m$ randomly and independently sampled from the uniform measure on $\mathbb{T}^d$. 
Then, every minimizer $\hat{\psi}$ of the training program 
\begin{equation}
\label{eq:MSE+regularization}
\min_{\psi \in \mathcal{N}} \sqrt{\frac 1m \sum_{i=1}^m \left|\mathscr{L}[\psi](\bm{x}_i) - f(\bm{x}_i)\right|^2} + \lambda \mathcal{R}(\psi)
\end{equation}
satisfies 
\begin{align}
\label{eq:PET_first_error_bound}
    \|u-\hat{\psi}\|_{L^2} + \|(\Delta-\rho)(u-\hat{\psi})\|_{L^2}
    & \leq C^{(1)}_{a,\rho}\cdot\frac{\sigma_s(\bm{c}_\Lambda)_1}{\sqrt{s}} + C^{(2)}_{a,d,\rho} \cdot\left(\frac{\|u-u_{\Lambda}\|_{W^{2,\infty}}}{\sqrt{s}} + \|u-u_{\Lambda}\|_{H^2}\right).
\end{align}
Moreover, if $\rho < 1$, we also have
\begin{align}
\label{eq:PET_H2_bound}
    \|u - \hat{\psi}\|_{H^2}
\leq C^{(3)}_{a,\rho}\cdot \frac{\sigma_s(\bm{c}_\Lambda)_1}{\sqrt{s}}+C^{(4)}_{a, d, \rho}\cdot \left(\frac{\|u-u_{\Lambda}\|_{W^{2,\infty}}}{\sqrt{s}}+\|u-u_{\Lambda}\|_{H^2}\right).
\end{align}
Here each constant depends only on the subscripted parameters. Moreover, the dependence of each constant on $d$ is at most linear (when present).
\end{theorem}

\begin{proof}\textbf{sketch.}
A complete proof of Theorem~\ref{thm:PET} is given in \S\ref{sec:proof_PET}. It leverages the convergence theory of CFC (presented and proved subsequently in \S\ref{sec:CFC_theory}) thanks to an \emph{ad hoc} construction of the network class $\mathcal{N}$. The idea is to construct networks $\psi\in\mathcal{N}$ so as to explicitly replicate linear combinations of Fourier functions $\{F_{\bm{\nu}}\}_{\bm{\nu}\in\Lambda}$ supported on the hyperbolic cross $\Lambda$, recall \eqref{eq:def_u_Lambda}. In this construction, only the last layer, corresponding to the coefficients of the linear combination is trained. The rest of the network is explicitly constructed to replicate (a suitably rescaled version of) the Fourier basis functions. In this setting, minimizing the regularized loss \eqref{eq:MSE+regularization} where $\mathcal{R}(\psi)$ is the $1$-norm of the weights in the last layer of the network $\psi$ is equivalent to solving a sparse regularization problem, specifically, a square-root LASSO problem \citep{adcock2019correcting, belloni2011square}, which allows to rigorously connect periodic PINNs' training with the CFC convergence theory (see Theorem~\ref{thm:CFC}).
\end{proof}

We conclude this section by highlighting some important features of Theorem~\ref{thm:PET}.
\begin{enumerate}[(i)]
\item Theorem~\ref{thm:PET} is called a \emph{practical} existence theorem since, as opposed to standard neural network existence results such as universal approximation theorems, see, e.g., \citep{elbrachter2021} and references therein, it not only guarantees the existence of neural networks (in this case, periodic PINNs) with favorable approximation properties, but also establishes that such networks can be computed by training a regularized RMSE loss and provides a sufficient condition on the minimum number of samples (i.e., collocation points) for the training process to be successful.
\item One of the key benefits of Theorem~\ref{thm:PET} is that the minimum number of training samples $m$ needed to successfully train a periodic PINN from the class $\mathcal{N}$ to solve a $d$-dimensional reaction-diffusion problem scales logarithmically (when $d \gg n$) or, at worst, linearly (when $n \gg d$) in $d$. This indicates that periodic PINNs are provably able to alleviate the curse of dimensionality.
\item As established by the error bound \eqref{eq:PET_first_error_bound}, the periodic PINN approximation accuracy guaranteed by Theorem~\ref{thm:PET} is controlled by the best $s$-term approximation error $\sigma_s(\bm{c}_\Lambda)_1$ and by the truncation error $u-u_{\Lambda}$ (measured with respect to the $W^{2, \infty}$- and the $H^2$-norm). This accuracy is inherited by the CFC convergence result that Theorem~\ref{thm:PET}'s proof relies on (see Theorem~\ref{thm:CFC}). In this paper we do not assume that $u$ belongs to a specific function class, but we note that the best $s$-term approximation and the hyperbolic cross truncation error could be bounded for functions satisfying suitable mixed regularity conditions, see \citep{dung2018hyperbolic, temlyakov2018multivariate} and the discussion in \cite[\S2.3]{wang20222}. In general, $\sigma_s(\bm{c}_\Lambda)_1$ can be estimated via Stechkin's inequality, see, e.g., \cite[Lemma 3.5]{adcock2022sparse}.
\item Theorem~\ref{thm:PET} also provides explicit bounds on the architecture of the periodic PINNs from class $\mathcal{N}$. The networks' depth scales logarithmically 
 in the dimension $d$ and the networks' width scales polynomially in $d$ (note that $n^{\log_2 d} = d^{\log_2 n}$), for sufficiently large $d$.
\end{enumerate}
The main limitations of Theorem~\ref{thm:PET} and related open problems are discussed in \S\ref{sec:conclusion}. The numerical impact of some of the theoretical ingredients of Theorem~\ref{thm:PET} is investigated in Appendix~\ref{sec:additional_numerics}.

\section{Compressive Fourier Collocation (CFC)}
\label{sec:CFC}

In this section, we illustrate the CFC method and its efficient numerical implementation via adaptive lower OMP. For more details, we refer to \cite{wang2022compressive}. Similarly to \S\ref{sec:main_result}, we consider a finite multi-index set $\Lambda \subset \mathbb{Z}^d$, the rescaled Fourier basis $\{\Psi_{\bm{\nu}}\}_{\bm{\nu}\in\mathbb{Z}^d}$ defined in \eqref{eq:def_Psi} and a finite-dimensional expansion $u_\Lambda$ of the form \eqref{eq:def_u_Lambda}.
Then, similarly to the PINN approach, we collocate the diffusion-reaction equation \eqref{eq:diffusion_eq_periodic_1} by means of Monte Carlo sampling. Hence, we randomly generate $m$ i.i.d.\ uniform points
$\bm{x}_1,\ldots,\bm{x}_m \in \mathbb{T}^d$, as in \eqref{eq:def_collocation_points}.
Letting $N=|\Lambda|$, we assume to have an ordering for the multi-indices in 
$\Lambda = \{\bm{\nu}_1, \ldots,\bm{\nu}_N\}$ (e.g., the lexicographic ordering). The PDE collocation process leads to the linear system
\begin{equation}
\label{eq:SC_system}
A \bm{z} = \bm{b},
\end{equation}
where $A \in \mathbb{C}^{m \times N}$ and $\bm{b} \in \mathbb{C}^m$ are defined by
\begin{equation}
\label{eq:def_A_b}
A_{ij}  = \frac{1}{\sqrt{m}}[-\nabla \cdot (a \nabla \Psi_{\bm{\nu}_j})+\rho\Psi_{\bm{\nu}_j}](\bm{x}_i) 
\quad \text{and} \quad 
b_{i}  = \frac{1}{\sqrt{m}} f(\bm{x}_i), \quad \forall i \in [m], j \in [N].
\end{equation}
We refer to $A$ as the CFC matrix. This collocation method is \emph{compressive} since we choose $m\ll N$. This makes the linear system \eqref{eq:SC_system} underdetermined.

\paragraph{Computing the CFC solution.}
Following the approach in \cite{wang2022compressive}, a CFC approximation $\hat{u}$ to $u$  can be computed by approximately solving the underdetermined linear system \eqref{eq:SC_system} via sparse recovery techniques. Therefore, we need to determine (i) a suitable truncation multi-index set $\Lambda$ and (ii) a sparse recovery method to approximately solve the linear system \eqref{eq:SC_system}. In this paper, the truncation set $\Lambda$ is chosen as a \emph{hyperbolic cross} of order $n$, i.e., $\Lambda = \Lambda^{\mathrm{HC}}_{d,n}$, recall \eqref{eq:def_HC}.
The hyperbolic cross offers a twofold advantage. First, as discussed in \S\ref{sec:main_result}, the cardinality of $\Lambda^{\mathrm{HC}}_{d,n}$ grows moderately in $n$ and $d$ when compared to other standard multi-index families (e.g., tensor product and total degree set). Second, a hyperbolic cross can be characterized as the union of \emph{lower} (or, equivalently, \emph{downward closed} or \emph{monotone}) sets of a given cardinality (\emph{cp}.\ Definition~\ref{def:lower}). 

After fixing $\Lambda$, we compute an approximate solution $\hat{\bm{c}}=(\hat{c}_{\bm{\nu}})_{\bm{\nu} \in \Lambda} \in \mathbb{C}^N$ to the underdetermined linear system \eqref{eq:SC_system}. To do so, we employ sparse recovery techniques, such as \emph{Orthogonal Matching Pursuit} (\emph{OMP}) or $\ell^1$ minimization \citep{foucart2013mathematical}. Finally, we define the CFC approximation as
\begin{equation}
\label{eq:def_uhat}
\hat{u} = \sum_{\bm{\nu} \in \Lambda} \hat{c}_{\bm{\nu}} \Psi_{\bm{\nu}}.
\end{equation}
In \cite{wang2022compressive} it was shown that $\hat{u}$ is an accurate and stable approximation to $u$ for high-dimensional diffusion equations, under sufficient conditions on the diffusion term and for a number of collocation points that scales only logarithmically with the ambient dimension $d$. In Theorem~\ref{thm:CFC}, we will extend the CFC convergence analysis from \cite{wang2022compressive} to diffusion-reaction problems. From the computational viewpoint, if one wants to compute a very sparse CFC approximation to the solution $u$, then OMP typically offers a faster reconstruction than solving an $\ell^1$ minimization problem via a convex optimization solver.

\subsection{Adaptive lower Orthogonal Matching Pursuit (OMP)}
\label{sssec:adaptive_lower_OMP}

For very high-dimensional domains (say, $d > 20$), the cardinality of the hyperbolic cross, despite being moderate with respect to other multi-index choices, becomes considerably large. This can make storing $A$ and computing $\hat{\bm{c}}$ genuinely challenging. To deal with higher dimensions, we need a more efficient recovery procedure that does not rely on choosing a large \emph{a priori} truncation set $\Lambda$, but constructs it iteratively.
This can be achieved by considering a reconstruction strategy called \emph{adaptive lower Orthogonal Matching Pursuit (OMP)}, that we now illustrate. The adaptive lower OMP algorithm presented here is based on analogous techniques employed in adaptive high-dimensional approximation, in particular in sparse grids \citep{gerstner2003dimension} and least squares methods \citep{migliorati2014adaptive,migliorati2019adaptive}.

First, we define the notion of lower set. Lower sets are an important class of multi-index sets in approximation theory and we refer to, e.g., \cite{cohen2018multivariate} or \cite[\S1.5]{adcock2022sparse} and references therein for further reading. Note that for polynomial approximations, lower sets are typically defined in $\mathbb{N}_0^d$. However, since we employ the complex Fourier basis here we consider lower sets in $\mathbb{Z}^d$. Before defining lower sets, we introduce a convenient notation to compare multi-indices of $\mathbb{Z}^d$. For $\bm{\mu},\,\bm{\nu} \in \mathbb{Z}^d$, we use  $\bm{\mu}\preceq\bm{\nu}$ to indicate that $|\mu_i| \leq |\nu_i|$ for every $i \in [d]$. Moreover,  if $\bm{\mu} \preceq \bm{\nu}$ and there exists an $i \in [d]$ such that $|\mu_i|<|\nu_i|$, then we say that $\bm{\mu}\prec\bm{\nu}$.
\begin{definition}[Lower set of $\mathbb{Z}^d$]
\label{def:lower}
A multi-index set $\Lambda \subseteq \mathbb{Z}^d$ is said to be \emph{lower} if the following holds for all $\bm{\nu}\in \Lambda$ and $\bm{\mu} \in \mathbb{Z}^d$: if  $\bm{\mu}\preceq \bm{\nu}$, then $\bm{\mu}\in \Lambda$.
\end{definition}
An example of a lower set is given in Fig.~\ref{fig:greedy_step} (blue dots). An equivalent condition for $\Lambda \subseteq \mathbb{Z}^d$ to be a lower set is the following: if $\bm{\nu} \in \Lambda$, then the box $\prod_{k=1}^d[-|\nu_k|, |\nu_k|] \subseteq \Lambda$. Note that lower sets of $\mathbb{Z}^d$ are symmetric with respect to all coordinate hyperplanes $\{\bm{\nu}\in \mathbb{Z}^d : \nu_k = 0\}$, with $k \in [d]$. This symmetry can be justified as follows. For real-valued solutions $u :\mathbb{T}^d \to \mathbb{R}$ one could employ a real Fourier expansion with respect to basis functions of the form $\prod_{k=1}^d \xi_k(2 \pi \nu_k x_k)$, where $\xi_k \in \{\sin, \cos\}$ and $\bm{\nu} \in \mathbb{N}_0^d$. It is a simple exercise to verify that the complex Fourier expansion of each of these real-valued trigonometric functions is supported on a set symmetric with respect to every coordinate hyperplane. 

In order to create a mechanism able to iteratively enlarge a multi-index set while preserving the lower set structure, we introduce the notion of reduced margin.
\begin{definition}[Reduced margin]
\label{def:margin}
The \emph{reduced margin} of a (nonempty) lower set $\Lambda \subseteq \mathbb{Z}^d$ is the multi-index set
$
\mathcal{R}(\Lambda) := \{\bm{\nu} \in \mathbb{Z}^d :\,\bm{\nu} \notin \Lambda \text{ and } \forall \bm{\mu} \prec \bm{\nu}, \, \bm{\mu} \in \Lambda \}.
$
Moreover, we let $\mathcal{R}(\emptyset) := \{\bm{0}\}$.
\end{definition}
We provide an illustration of the reduced margin in Fig.~\ref{fig:greedy_step} (red dots).
A key property of the reduced margin is that if $\Lambda \subseteq \mathbb{Z}^d$ is a lower set and $\bm{\nu} \in \mathcal{R}(\Lambda)$ then $\Lambda \cup \{\bm{\mu}\in \mathbb{Z}^d :  |\mu_i| = |\nu_i|, \; \forall i \in [d]\}$ is also a lower set.

The adaptive lower OMP algorithm is a structured variant of OMP, see, e.g., \citep[\S3.2]{foucart2013mathematical} and references therein, with the only difference that the greedy search is restricted to iteratively enlarged lower sets. The algorithm generates a nested sequence of lower sets $\emptyset=\Lambda^{(0)} \subseteq \Lambda^{(1)} \subseteq \ldots \subseteq \Lambda^{(K)}$ and then computes a sequence of least-squares solutions $\bm{z}^{(n)}$ to the linear system \eqref{eq:SC_system} such that $\mathrm{supp}(\bm{z}^{(n)}) \subseteq \Lambda^{(n)}$. At each iteration, adaptive lower OMP extends the existing multi-index set $\Lambda_n$ by picking an element $\bm{\nu}^{(n)} \in \mathcal{R}(\Lambda^{(n)})$ corresponding to the largest absolute residual $\left|(A^*(\bm{b}-A \bm{z}^{(n)}))_{\bm{\nu}}\right|$.
Then, it also adds all reflections of $\bm{\nu}^{(n)}$ with respect to the coordinate hyperplanes $\{\bm{\nu} \in \mathbb{Z}^d : \nu_k = 0\}$ to preserve the lower structure. 
The greedy selection criterion of adaptive lower OMP (Line~\ref{step:greedy} in Algorithm~\ref{algo:lower_OMP}) is illustrated in Fig.~\ref{fig:greedy_step}.
\begin{figure}
\centering
    \includegraphics[width = 7cm]{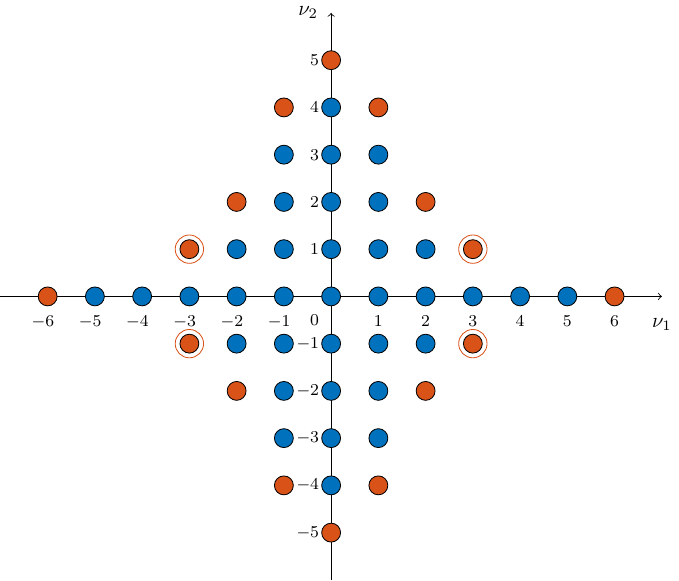}
    \caption{\label{fig:greedy_step} Illustration of the greedy selection criterion of adaptive lower OMP (Line~\ref{step:greedy} in Algorithm~\ref{algo:lower_OMP}). Considering a given lower set $\Lambda^{(n)} \subset \mathbb{Z}^2$ (blue), its reduced margin $\mathcal{R}(\Lambda^{(n)})$ is drawn in red. If the quantity $|A^*(\bm{b}-A\bm{z}^{(n)})_{\bm{\nu}}|$ is minimized at $\bm{\nu} = (3,1)$, then $\Lambda^{(n+1)}$ is constructed by adding the red circled dots to  $\Lambda^{(n)}$.}
\end{figure}
The method is summarized in Algorithm~\ref{algo:lower_OMP}.
\begin{algorithm}[H] 
\caption{Adaptive lower Orthogonal Matching Pursuit (OMP)\label{algo:lower_OMP}} 
\begin{algorithmic}[1] 
\REQUIRE{$A\in \mathbb{C}^{m \times N}$, $\bm{b}\in \mathbb{C}^m$, and number of iterations $K \in \mathbb{N}$}
\ENSURE{A $K$-sparse vector $\hat{\bm{c}} \in \mathbb{C}^N $, approximately solving \eqref{eq:SC_system}}
\STATE{$\Lambda^{(0)}=\emptyset,\ \bm{z}^{(0)}=\bm{0},$}
\FOR{$n=0,\dots,K-1$}
    \STATE{Compute the columns of $A$ corresponding to $\mathcal{R}(\Lambda^{(n)})$ (see Definition~\ref{def:margin})}
    \STATE{$d_{\bm{\nu}} \gets 1/\sqrt{\sum^m_{i=1}|A_{i\bm{\nu}}|^2}$, $A_{i\bm{\nu}} \gets A_{i\bm{\nu}}d_{\bm{\nu}}$, for all $\bm{\nu} \in \mathcal{R}(\Lambda^{(n)})$ ($\ell^2$-normalize the new columns of $A$)}
    \STATE{$\bm{\nu}^{(n)} \gets \displaystyle \argmax_{\bm{\nu} \in  \mathcal{R}(\Lambda^{(n)})} \left|(A^*(\bm{b}-A \bm{z}^{(n)}))_{\bm{\nu}}\right| $ (greedy multi-index selection) \label{step:greedy}}
    \STATE{$\Lambda^{(n+1)} \gets  \ \Lambda^{(n)} \cup \{\bm{\mu} \in \mathbb{Z}^d :  |\mu_i| = |\nu_i^{(n)}|, \, \forall i \in [d]\}$}
    \STATE{$\bm{z}^{(n+1)} \gets \displaystyle\arg\min_{\bm{z}\in \mathbb{C}^N} \left\{ \left\| \bm{b} - A \bm{z}\right\|_2^2,\,\mbox{supp}(\bm{z}) \subseteq \Lambda^{(n+1)} \right\}$}
\ENDFOR
\STATE{$\hat{\bm{c}} \gets D \bm{z}^{(K)}$, with $D = \text{diag}(\bm{d})$} 
\end{algorithmic}
\end{algorithm}

\section{A numerical study of CFC and periodic PINNs}
\label{sec:numerics}

In this section, we present numerical experiments on the high-dimensional periodic diffusion-reaction equation  \eqref{eq:diffusion_eq_periodic_1}. After illustrating the numerical setup in \S\ref{Sec:numerical_setting}, we conduct tests on PINNs with periodic layer in \S\ref{Sec:NN_performance} and the (adaptive) CFC method in \S\ref{Sec:LowerOMP_result}. Then, we compare the two approaches in \S\ref{Sec:CS_vs_NN}.

\subsection{Numerical setup} \label{Sec:numerical_setting}

We start by introducing the setup of our numerical experiments.

\paragraph{Measurement of errors.} In all numerical experiments, we use the relative $L^2$-error to measure the approximation error. It is defined as
$$
\text{relative}\ L^2\text{-error} = \frac{\|u-\hat{u}\|_{L^2}}{\|u\|_{L^2}},
$$
where $\hat{u}$ is the computed approximation to an exact solution $u$. Computing the $L^2$-norms is a challenge in high dimension, therefore the norms $\|\cdot\|_{L^2}$ are approximated using Monte Carlo integration, i.e.,
$$
\|u\|_{L^2} \approx \sqrt{\frac{1}{M}\sum^{M}_{i=1} |u(\bm{x}_i)|^2},
$$
where the $\bm{x}_1, \ldots, \bm{x}_M$ are $M \gg 1$ random independent points uniformly distributed in ${\mathbb{T}^d}$. Throughout this section, we choose $M=10000$.

\paragraph{Diffusion coefficient.}
As shown in \cite{wang2022compressive}, the choice of the type of diffusion coefficient (e.g., constant, sparse, or non-sparse) does not impact the numerical results for the compressive Fourier collocation method solving high-dimensional PDEs. Thus, in all experiments, we use the following diffusion coefficient:
$$
a(\bm{x}) = 1+ 0.25\sin{(2\pi x_1)}\sin{(2\pi x_2)}.
$$

\paragraph{Exact solutions.}
In our experiments, we will consider the problems having the following three high-dimensional functions as their solution:
\begin{align}
u_{1}(\bm{x}) & = \sin(4\pi x_1) \sin(2\pi x_2), & \text{(Example 1)} \label{eq:exact1}
\\
u_{2}(\bm{x}) & = \exp \left(\sin(2\pi x_1)+ \sin(2\pi x_2)\right), & \text{(Example 2)} \label{eq:exact2}
\\
u_{3}(\bm{x}) &= \exp \left(  \sum^d_{k=1} \frac{1}{k^2}\sin(2\pi x_k)\right).
& \text{(Example 3)} \label{eq:exact3}
\end{align}
In the following, we will use Examples 1, 2, and 3 to refer to diffusion-reaction equations with exact solutions given by \eqref{eq:exact1}, \eqref{eq:exact2}, and \eqref{eq:exact3}, respectively. Note that for any given coefficients $a$ and $\rho$, we will enforce the PDE to have a prescribed exact solution by suitable choice of the forcing term $f$.
Here, $u_{1}$ is defined in terms of only one real-valued trigonometric function. This corresponds to a 4-sparse solution with respect to the complex Fourier system. Hence, in this case we have $\sigma_s(\bm{c}_{\Lambda})=0$ and $u-u_{\Lambda}=0$ as soon as the truncation set $\Lambda$ contains the four multi-indices $(\pm 2,\pm 1,0,0,\dots)$. Note, though, that these four multi-indices do not form a lower set.
The solution $u_{2}$ is a smooth periodic function active only in the variables $x_1$ and $x_2$. Hence, it exhibits a highly anisotropic behaviour.
Finally, $u_{3}$
is another smooth solution defined using all variables from $x_1, \ldots, x_d$. The behaviour of this solution is also anisotropic and the solution is uniformly bounded by $\exp (\sum_{k=1}^\infty 1/{k^2})=\exp(\pi^2/6)$.

\paragraph{Randomization of experiments and visualization.} Due to the random nature of collocation points in the CFC and the periodic PINN method, we consider 25 random runs for each test. However, for the experiments in \S\ref{Sec:architecture}, where we examine the effect of the architecture on performance, we use 10 random runs for each width-depth combination. In our plots, the main curves represent the sample geometric mean of the error, and the size of the lightly shaded areas corresponds to its corrected geometric standard deviation. Note that calculating the geometric mean and standard deviation of the errors is equivalent to compute the classical mean and standard deviation of the log transformed errors. For more details on the visualization strategy, we refer to \cite[\S A.1.3]{adcock2022sparse}.

\subsection{Physics-Informed Neural Networks} \label{Sec:NN_performance}
We run tests using PINNs to solve the PDE defined in \eqref{eq:diffusion_eq_periodic_1} with the various exact solutions from \eqref{eq:exact1}-\eqref{eq:exact3}. For the experiments in \S\ref{sssec:perf_num_samples} and \S\ref{sssec:perf_dimension}, we use DNNs with $l=11$ nodes in the periodic layer, depth $h=3$ hidden layers, and width to depth ratio $r=10$ so that the networks have $30$ nodes per hidden layer. Moreover, we consider a tanh activation in every layer. In \S\ref{Sec:architecture}, we examine the impact of the architecture on performance by comparing results obtained by varying the number of nodes on the periodic layer and on the hidden layers of the fully connected part of the network, see Fig.~\ref{Fig:NN_architecture}. In this section, we train periodic PINNs without regularization, i.e., we let $\lambda =0$ in \eqref{eq:def_loss}. Moreover, we simply minimize the MSE loss (which is equivalent to minimizing the RMSE in the non-regularized setting). Additional experiments involving a regularized loss function can be found in Appendix~\ref{sec:additional_numerics}.

\subsubsection{Performance and the impact of the number of samples}
\label{sssec:perf_num_samples}

We test the PINNs with the exact solutions defined in \eqref{eq:exact1}-\eqref{eq:exact3} and refer to them as Example 1, 2, 3, respectively. Fig.\,\ref{Fig:PINN_number_samples} presents the performance of the PINNs over 30000 epochs for each exact solution, with parameters $d=6$, $\rho =0.5$. 
\begin{figure}[!t]
\centering
  \includegraphics[width=.32\textwidth]{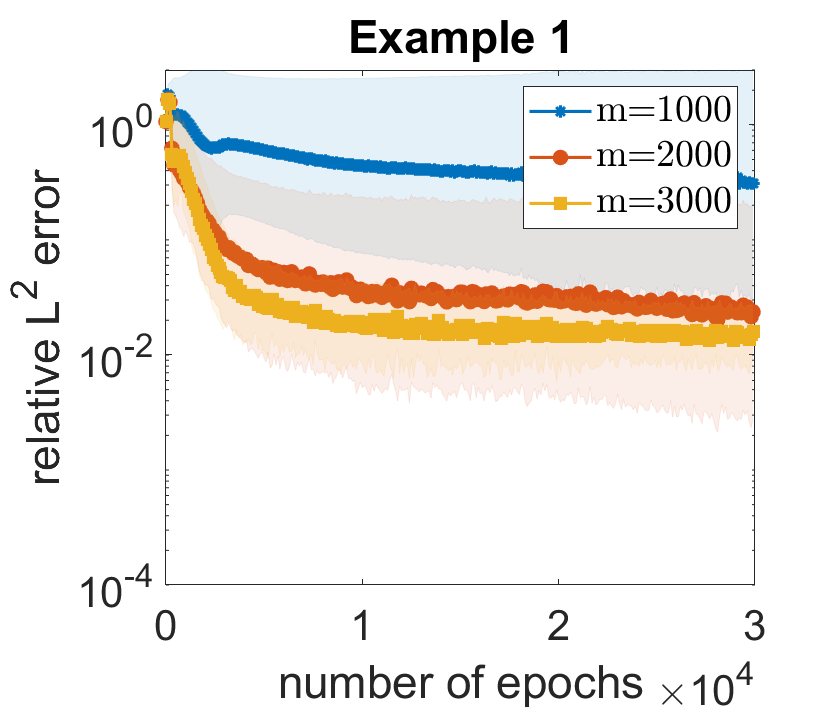}
  \includegraphics[width=.32\textwidth]{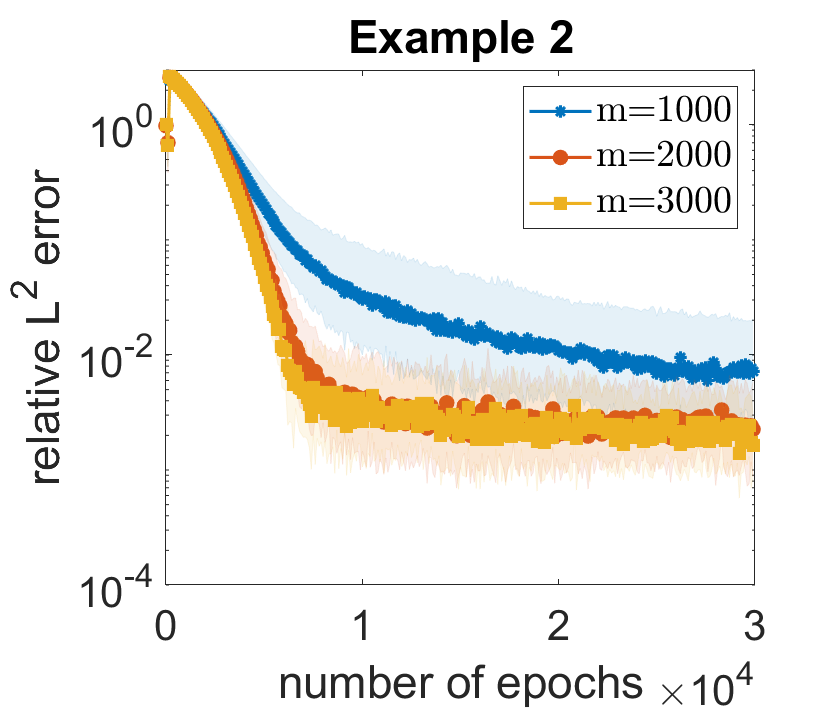}
  \includegraphics[width=.32\textwidth]{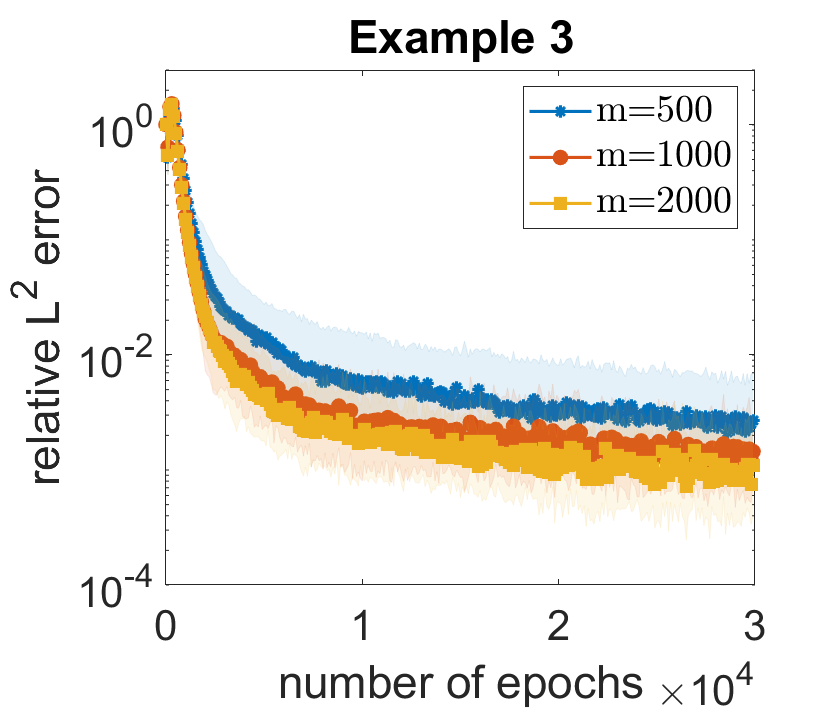}
\caption{(Impact of the number of samples) Relative $L^2$-error versus the number of epochs for approximating the exact solutions defined in \eqref{eq:exact1}-\eqref{eq:exact3} with $d=6$, where $m$ is the number of sample points.} \label{Fig:PINN_number_samples}
\end{figure}
The PINNs have 11 nodes on the periodic layer and three non-periodic layers with 30 nodes each. We choose the commonly used $\tanh$ activation function and the Adam optimizer \citep{kingma2017adam} with early stopping. We use this neural network to conduct all numerical tests except those in \S\ref{Sec:architecture}. As shown in Fig.~\ref{Fig:PINN_number_samples}, the PINN model approximates all three exact solutions with $L^2$-error on the order of $10^{-2}$ to $10^{-3}$ within the $30000$ epoch training budget. Notably, the error decreases rapidly in the first $10000$ epochs and can be seen to saturate for the remaining epochs. 
Also, in Examples 1 and 3, we see some small improvement in the error as the size of the training set is increased, while for Example 2, we note that moving from $m=2000$ to $m=3000$ points does not noticeably improve the error.

These differences in results obtained for our examples illustrate that the number of samples needed to reach a minimum error varies depending on the complexity of the exact solution. This difference in behaviour is not directly explained by Theorem~\ref{thm:PET}. In fact, Example~1 is the the easiest one to approximate in the Fourier basis as it is only a linear combination of 4 Fourier basis functions, but it corresponds to the worst performance for periodic PINNs. The difficulty in approximating Example~1 could be heuristically explained by the presence of a higher frequency $4 \pi$ (as opposed to $2\pi$) associated with the variable $x_1$. Yet, a rigorous explanation of this phenomenon remains an open problem.

\subsubsection{Impact of the dimension}
\label{sssec:perf_dimension}

We now consider the performance of PINNs in higher-dimensional settings. We present the results for exact solutions with $d=6, 10, 20$ in Fig.\,\ref{Fig:PINN_dimensionality}. 
\begin{figure}[!t]
\centering
  \includegraphics[width=.35\textwidth]{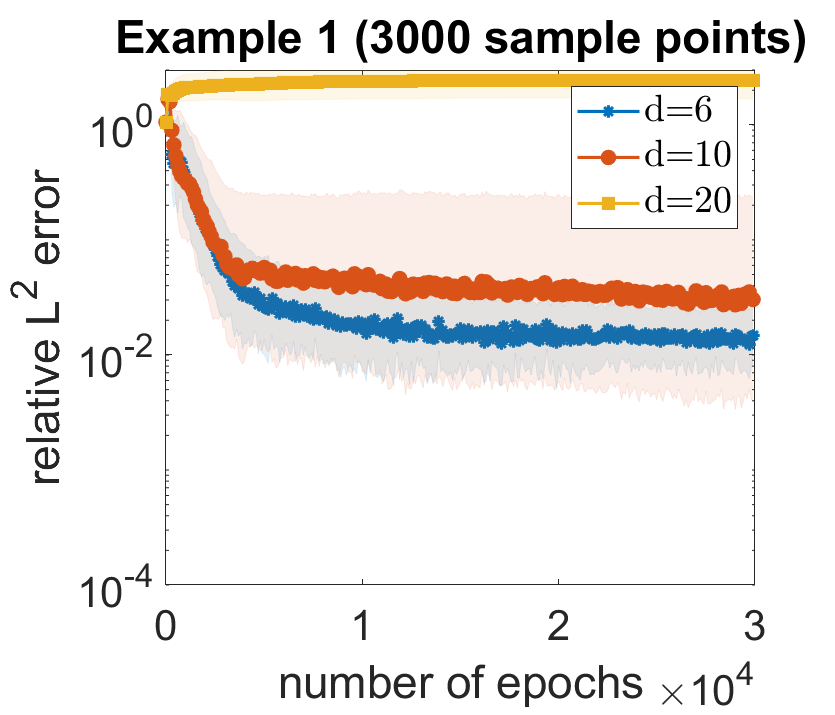}
  \includegraphics[width=.35\textwidth]{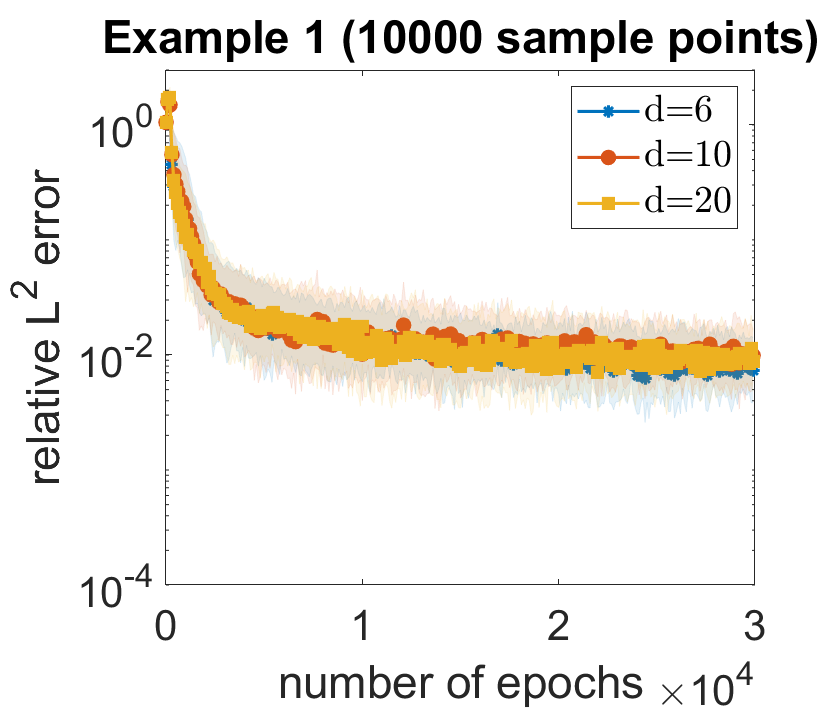}\\
  \includegraphics[width=.35\textwidth]{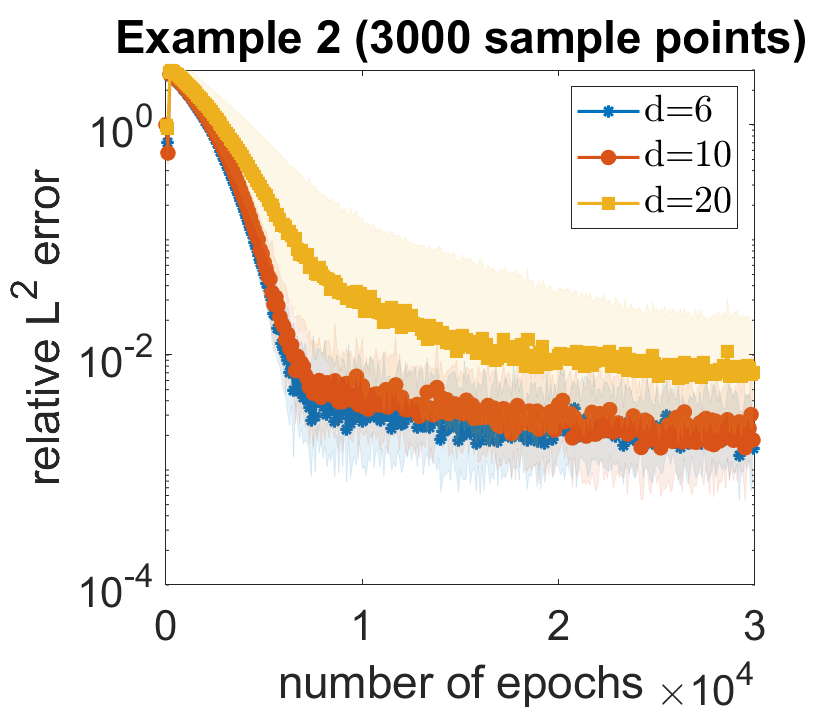}
  \includegraphics[width=.35\textwidth]{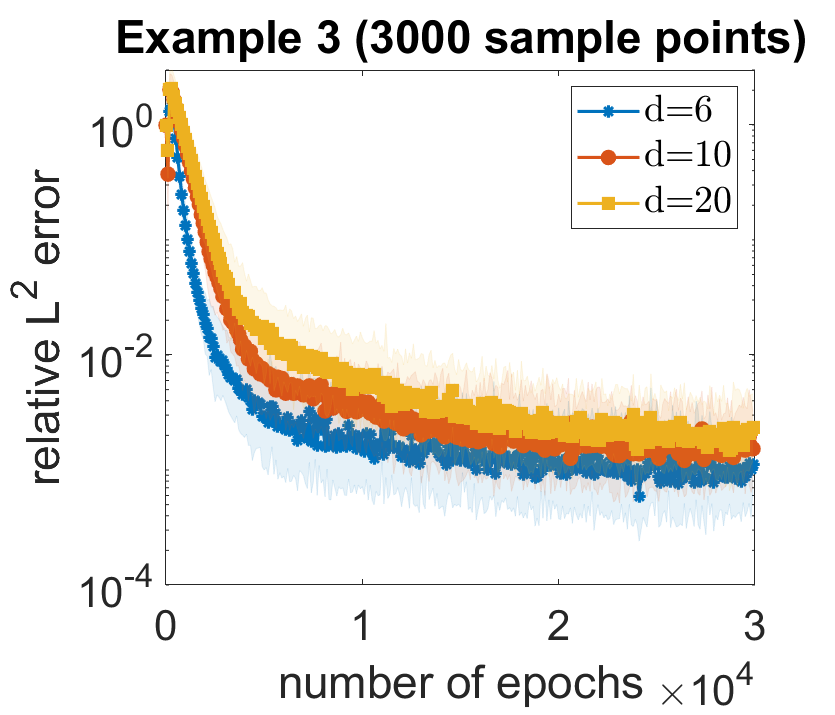}
\caption{(Impact of the dimension) Relative $L^2$-error 
versus the number of epochs for approximating the exact solutions defined in \eqref{eq:exact1}-\eqref{eq:exact3} with (top left, bottom left, bottom right) $m=3000$ and (top right) $m=10000$ samples, in $d = 6, 10, 20$ dimensions.} \label{Fig:PINN_dimensionality}
\end{figure}
With 3000 sample points, the PINNs approximate Example 1's exact solution with relative $L^2$-error on the order of $10^{-2}$ when $d=6$.
However, we observe that the error increases with increasing dimension with the error for $d=10$ above that for $d=6$, while for $d=20$ the PINNs do not converge at all within the budget of epochs. 
Increasing the number of samples to 10000 (result in the top right of Fig.\,\ref{Fig:PINN_dimensionality}), the PINNs can be seen to be converging to the exact solution, saturating at an identical error around $10^{-2}$ in dimensions $d=6,10$, and $20$ after 30000 epochs.
This highlights the dependence on the problem dimension in the number of samples required to achieve a given accuracy.
This can also be observed for Examples 2 and 3, where we see that with fixed $m=3000$ samples the error increases as the dimension increases.

As discussed in the introduction, we seek methods to solve high-dimensional PDEs which overcome or substantially mitigate the curse of dimensionality. The experiments in Fig.\,\ref{Fig:PINN_sample_dimension} aim to illuminate the relationship between the dimension of the problem and the number of samples required to achieve a given error. 
\begin{figure}[!t]
\centering
  \includegraphics[width=.32\textwidth]{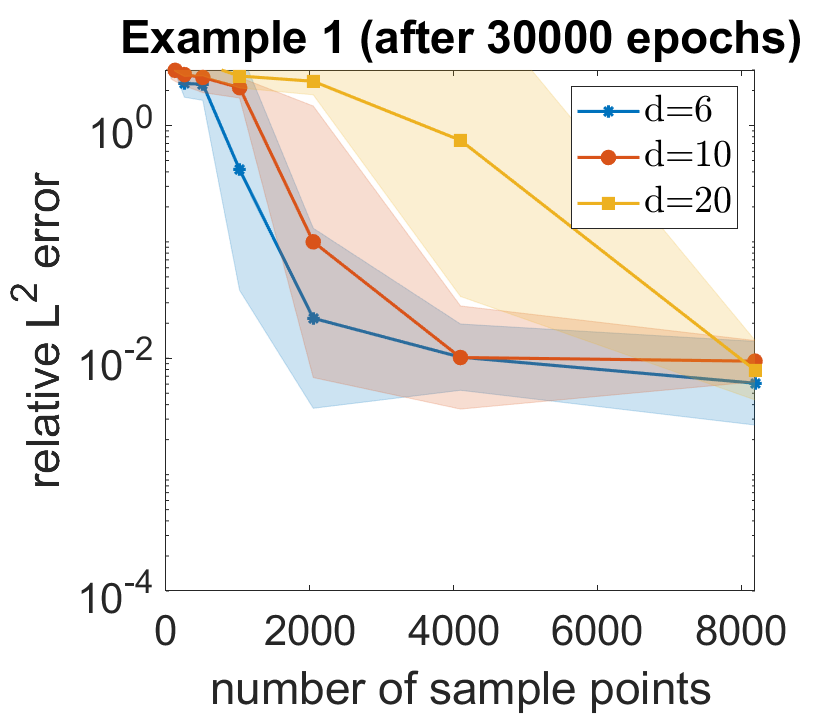}
  \includegraphics[width=.32\textwidth]{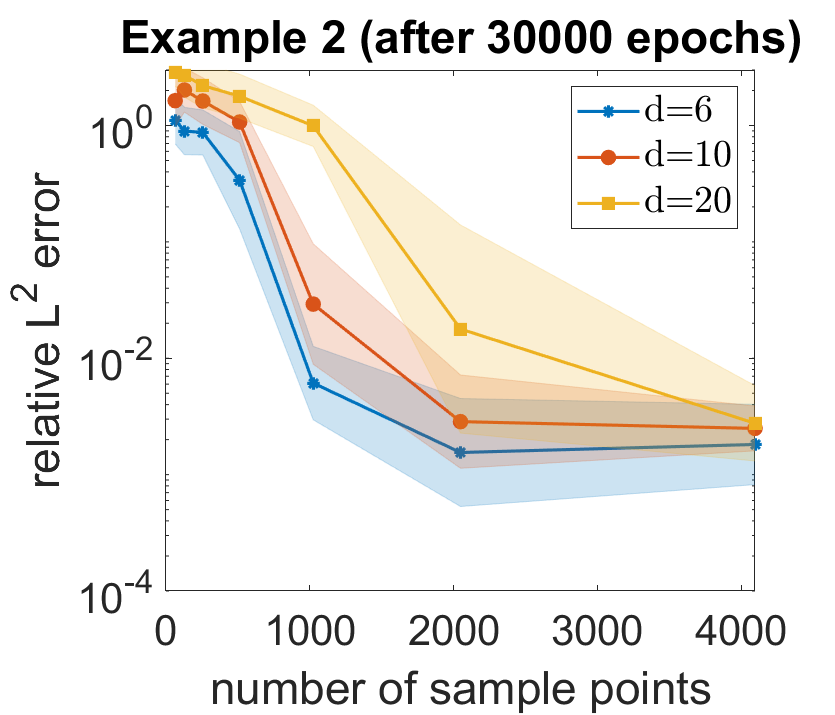}
  \includegraphics[width=.32\textwidth]{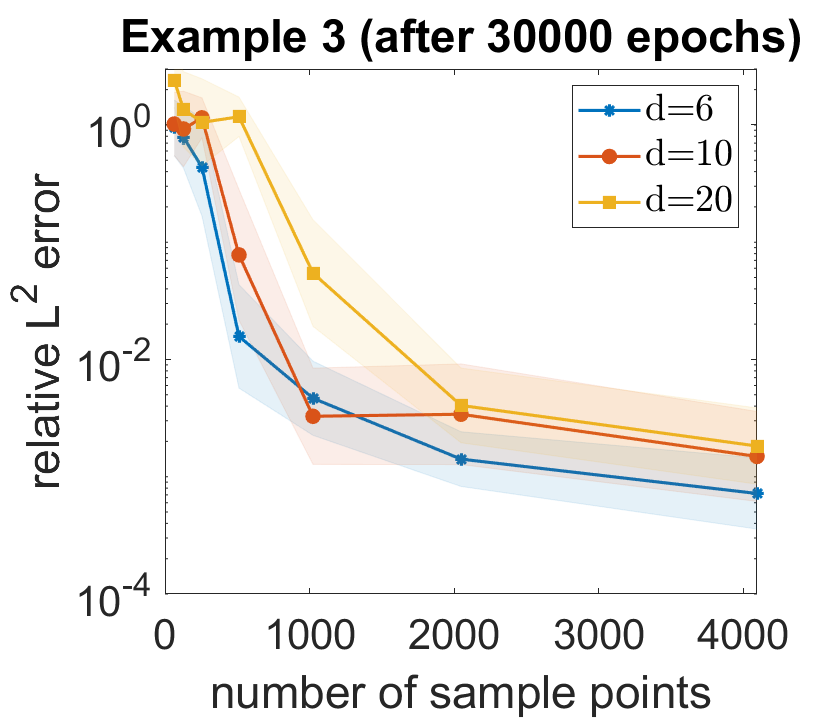}

\caption{(Dimensionality and number of samples) Relative $L^2$-error after 30000 epochs versus number of sample points $m$ for the exact solutions defined in Eq.\,\eqref{eq:exact1}-\eqref{eq:exact3}, where $d= 6, 10, 20$ are the dimension of the problem. Left: Example 1. Middle: Example 2. Right: Example 3.} \label{Fig:PINN_sample_dimension}
\end{figure}
There we plot the relative $L^2$-error vs.\ the number of sample points in $d=6,10$, and $20$ dimensions for each of our examples. We observe in all three cases that the number of samples required to achieve a given error increases with the dimension. However, focusing on the point at which the PINNs begin to saturate in their relative $L^2$-error, this scaling does not appear to be exponential in the dimension which would certainly be the case if the method was not capable of mitigating the curse. Rather, the scaling in the required number of samples appears to be approximately linear. Comparing the results for all three examples when $d=6, 10$ and $20$, we observe that the number of samples required for the PINNs to reach saturation at relative $L^2$-error around $10^{-2}$ roughly doubles moving from $6$ to $10$ dimensions and again moving from $10$ to $20$ dimensions. This empirical observation is in accordance with Theorem~\ref{thm:PET}, which establishes that $m$ scales at worst linearly in $d$ (see \eqref{eq:PET_sample_complexity}). It is however worth stressing that PINNs are not able to reach very high accuracy (an $L^2$-error of around $10^{-2}$ might not be considered fully satisfactory by numerical PDE experts). This fact could depend on the optimizer or the choice of architecture (one could consider more sophisticated options, such as networks where hidden layers do not have constant width, or networks with residual layers).

\subsubsection{Impact of the architecture} \label{Sec:architecture}
We also examine the impact of the neural network's architecture on performance. Specifically, we consider varying the number of nodes in the periodic layer and the width-depth ratio $r = w/h$ of the hidden layers.
In Fig.\,\ref{Fig:PINN_number_samples} and \ref{Fig:PINN_dimensionality}, we see that the neural network approach produces accurate convergence results in all three examples. 
We then naturally raise a question: can a better choice of the hyper-parameters $l$, $h$, and $r=w/h$ further improve the accuracy of the neural network approach? To answer this question, we design experiments by changing these hyper-parameters and testing the performance in Example 3. We change the number of nodes in the periodic layer of the neural network in the left plot of Fig.\,\ref{Fig:PINN_architecture}. 
\begin{figure}[!t]
\centering
  \includegraphics[width=.32\textwidth]{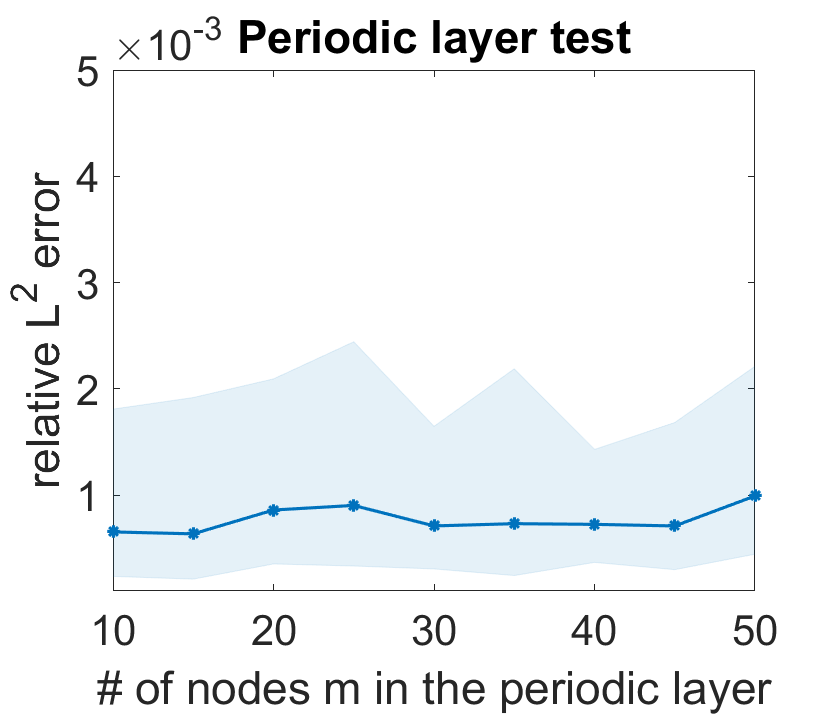}
  \includegraphics[width=.32\textwidth]{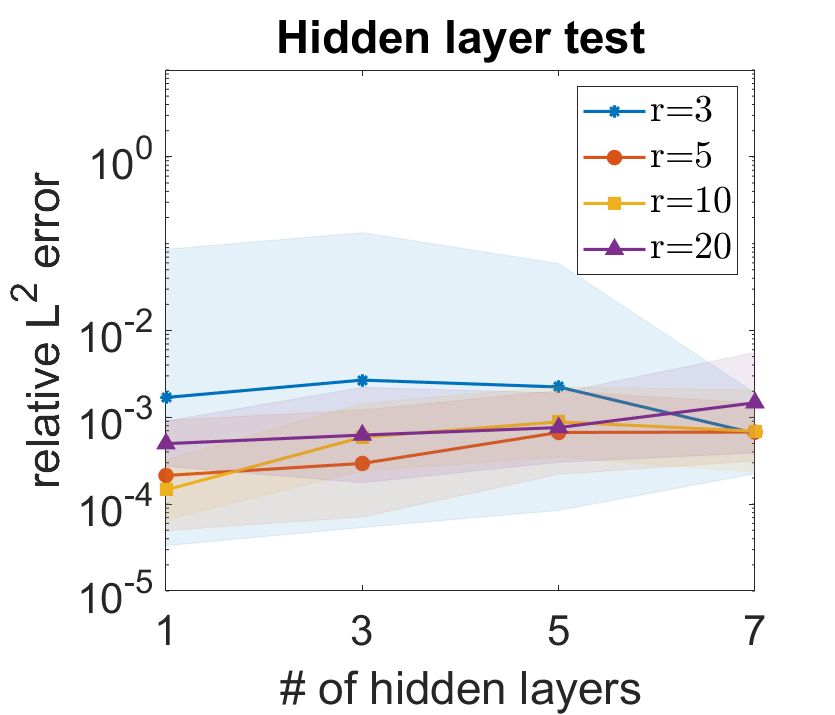}
\caption{(Impact of the architecture) Left: Relative $L^2$-error versus number of nodes in the periodic layer. Right: Relative $L^2$-error versus the number of hidden layers. Different curves represent the different width-depth ratios of the hidden layers. For both figures, we test the neural networks on Example 3, in $d=6$ dimensions using 5000 sample points, and the relative $L^2$-error are measured after 30000 epochs.} \label{Fig:PINN_architecture}
\end{figure}
The results show that increasing the number of nodes beyond 10 in the first layer does not improve the $L^2$-error. The results for tuning the number of layers and the number of nodes per layer for non-periodic layers are illustrated in the right plot of Fig.\,\ref{Fig:PINN_architecture}. We observe that adding more hidden layers to the fully-connected part of the network does not improve the accuracy of the periodic-NN solution. However, we do observe larger standard deviation for the choice of $r=3$ for the width-depth ratio in comparison to the choices of $r=5,10,20$.
We also observe diminishing returns in further increasing the ratio to $r=20$ as the cost of training increases due to the larger width of the network. Since we observe that the relative $L^2$-error is not sensitive to these hyper-parameters, we choose hyper-parameters as stated at the beginning of \S\ref{Sec:NN_performance}. Analogous results hold for higher dimensions (we have run experiments in dimensions $d = 10$ and $20$), but they are omitted for the sake of brevity.

\subsection{Adaptive lower OMP results}
\label{Sec:LowerOMP_result}
In this subsection, we represent numerical results for the adaptive lower OMP method (Algorithm~\ref{algo:lower_OMP}), including the convergence results and a comparison with the non-adaptive Fourier compressive collocation method described in \cite{wang2022compressive}.
\subsubsection{Number of iterations}
The lower OMP method is an adaptive method. On each iteration, the method adaptively changes the size of the lower index set and outputs an approximation to the PDE solution. As in the previous sections, we use the relative $L^2$-error to measure accuracy. Here the size of the index set directly relates to the computational cost of the lower OMP method. Hence, we study the relative $ L^2$-error and the cardinality of the lower index set over the number of iterations to measure the performance of the method. From the numerical experiments in Fig.\,\ref{Fig:Adaptive_OMP_error_card}, we observe  that for the different examples considered, the convergence rate and the increments in the size of the index set vary substantially.
\begin{figure}
\centering
  \includegraphics[width=.32\textwidth]{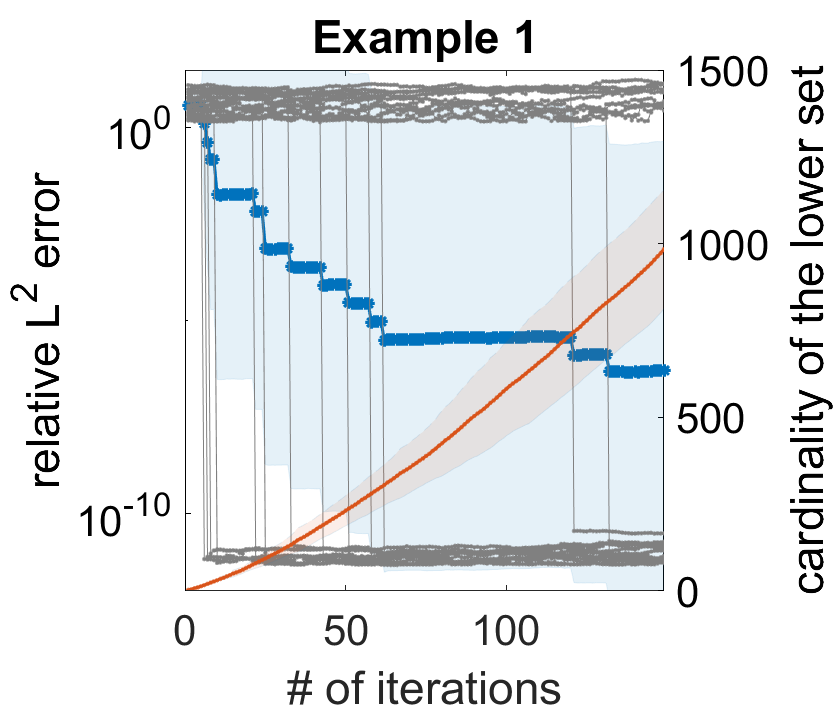}
  \includegraphics[width=.32\textwidth]{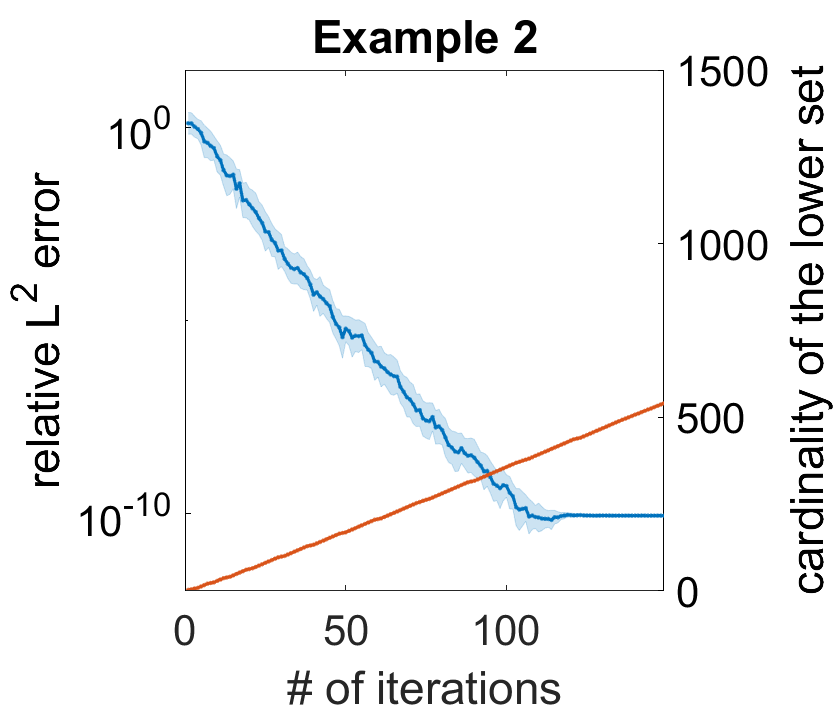}
  \includegraphics[width=.32\textwidth]{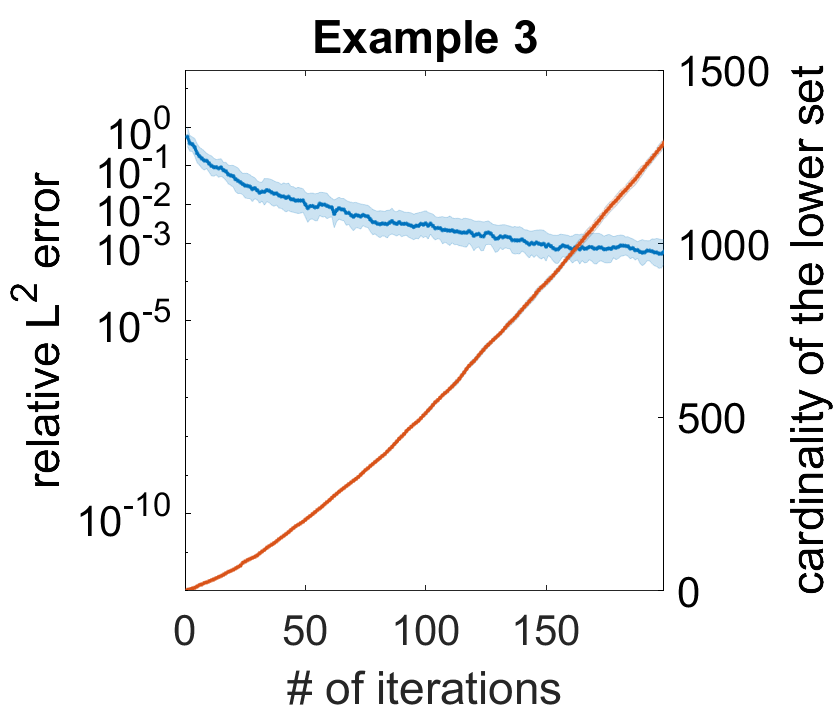}
\caption{(Number of iterations) Relative $L^2$-error (blue curve, left y-label) and cardinality (red curve, right y-label) of the lower set versus the number of iterations for the exact solutions defined in Eq.\,\eqref{eq:exact1}-\eqref{eq:exact3}. Other parameters include dimension $d= 6$ and sample points $m=3000$. } \label{Fig:Adaptive_OMP_error_card}
\end{figure}

In the case of Example 1, once the algorithm discovers the index for the complex Fourier basis function $\sin{4\pi x_1} \sin{2\pi x_2}$, then the $L^2$-error decays to less than $10^{-10}$, otherwise, the approach does not converge. In this example, the visualization using sample geometric mean can not fully describe the convergence behavior, so we also plot each run's relative $L^2$-error using black data points. As the number of iterations increases, we observe the algorithm is more likely to discover the correct Fourier basis function (i.e., more runs reach relative $L^2$-error below $10^{-10}$). However, even after 150 iterations, not all trials have converged to the exact solution as we see that some data points are still near the top of the figure. This suboptimal behaviour is not surprising since the multi-index set corresponding to the basis function $\sin{4\pi x_1} \sin{2\pi x_2}$, i.e., $\{\pm2 \bm{e}_1 \pm \bm{e}_2\}$, is not lower. Hence, Example 1 is not ideally suited for the adaptive lower OMP setting.

The situation is different for the other two examples. Example 2 has only two activated variables, and the coefficients satisfy a lower structure. The adaptive lower OMP method successfully finds the appropriate two-dimensional lower set and converges to the exact solution after 100 iterations for all trials. The cardinality of the lower set increases at a slower rate in this example because the lower set only expands in the direction of $x_1$ and $x_2$. For Example 3, the $L^2$-error converges to approximately $10^{-3}$ within 200 iterations. In this example we also observe that the index set's cardinality significantly increases because the index set is being adaptively and anisotropically extended in all active dimensions.

\subsubsection{Comparison with traditional OMP method}

We now compare the performance for both the lower OMP and traditional OMP methods in moderate dimension $d=6$. We note that the traditional OMP method is not applicable in higher dimensional problems due to the requirement of using an enormous index set. On the other hand, the lower OMP method can produce results in problems as high as $d=30$ dimensions because it adaptively searches for the best index set (see \S\ref{Sec:CS_vs_NN} for more lower OMP results with $d=30$). Fig.\,\ref{Fig:Comparison_OMP} compares the $L^2$-error of the traditional and the lower OMP methods with the same solution sparsity. Traditional OMP has a better performance when the Fourier expansion of the exact solution is fully captured by the traditional OMP ambient space (see, e.g., the left plot for Example 1). Otherwise, the lower OMP method adaptively extends the index set using the existing residual data and achieves better accuracy after more iterations (see, e.g., the middle and right plots for Examples 2 and 3). 
\begin{figure}[!t]
\centering
  \includegraphics[width=.32\textwidth]{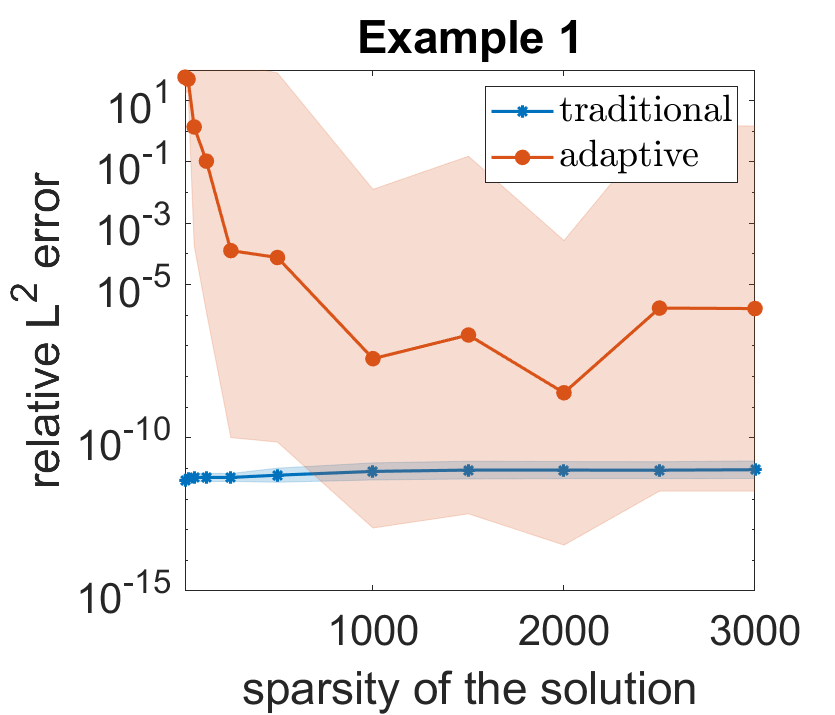}
  \includegraphics[width=.32\textwidth]{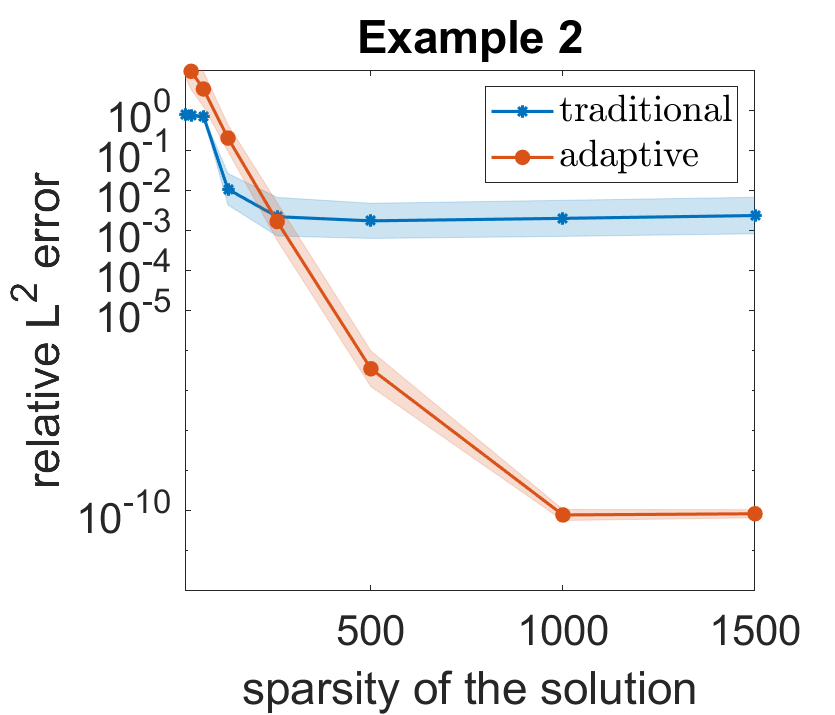}
  \includegraphics[width=.32\textwidth]{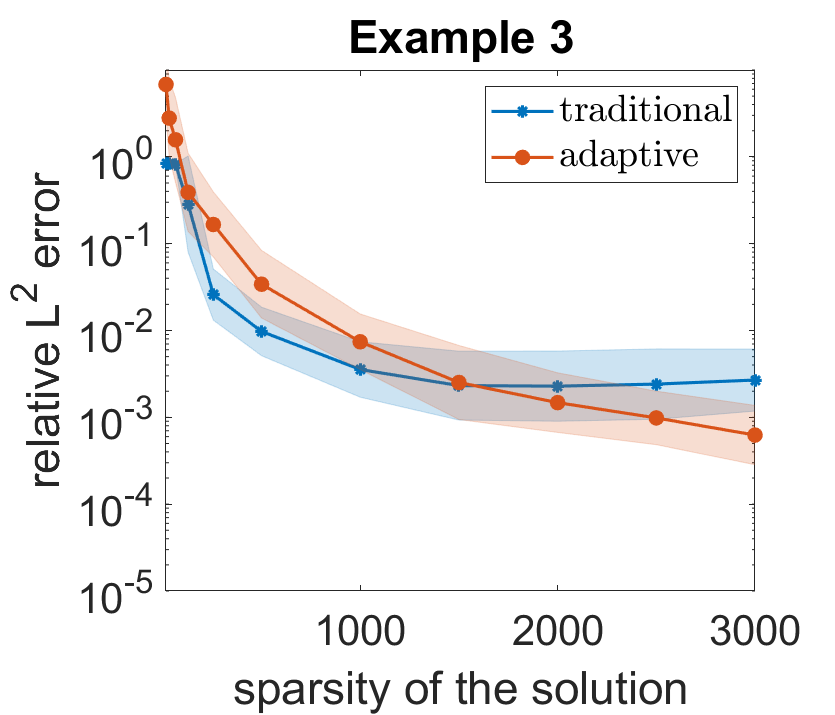}
\caption{(Comparison OMP) Relative $L^2$-error versus the sparsity of the OMP solution for the exact solutions defined in Eq.\,\eqref{eq:exact2}-\eqref{eq:exact3}, Here dimension $d= 6$, and we choose the hyperbolic cross multi-index set $\Lambda^{\mathrm{HC}}_{d,n}$ with $n=18$ (cardinality $=3418$) as ambient space for the traditional OMP. }\label{Fig:Comparison_OMP}
\end{figure}

\subsection{Periodic PINNs vs.\ CFC} \label{Sec:CS_vs_NN}

In this section we compare the performance of periodic PINNs and CFC with adaptive lower OMP recovery on high-dimensional test problems. 
To make a fair comparison, we use identical sets of sample points as the training data for both methods. 
For the neural network, we use 30000 epochs to examine the performance of the PINNs, which we empirically observe allows for the networks to saturate on the training data. In the adaptive lower OMP algorithm, the stopping criterion is set to be the size of the index set (which we recall is also an indicator of the computational cost). We choose this size to be greater than half of the number of sample points. As illustrated in Fig.\,\ref{Fig:Comparison}, the adaptive lower OMP method has advantages in discovering the underlying anisotropy of the high-dimensional functions, see, e.g., the results in Examples 2 and 3 where the function is inactive in most variables or when the coefficients satisfy a lower structure. 
The adaptive lower OMP method captures the most important terms in Example 2 as the number of samples reaches 1000. In all three cases, the neural network accurately approximates the solution as the number of sample points increases and reaches error approximately $10^{-2}$ (Example 1) and $10^{-3}$ (Examples 2 and 3). In Example 3, the error of the periodic PINN reaches a plateau after approximately 2500 samples (see also \S\ref{Sec:NN_performance}). However, similar to Example 1 in Fig.\,\ref{Fig:Comparison_OMP}, the lower OMP method struggles to converge for Example 1, while the periodic neural network provides slightly better average relative $L^2$-error at 8000 sample points with tighter spread in the standard deviation over the trials. In summary, CFC can achieve much higher accuracy than periodic PINNs (gaining from 1 to 8 orders of magnitude), but its performance depends on the sparsity properties of $u$. On the other hand, periodic PINNs are able to achieve a consistent accuracy level ($10^{-2}$ relative $L^2$-error) on all examples. Note that this phenomenon is not explained by Theorem~\ref{thm:PET}. While a possible cause for this behavior is the use of stochastic optimizers for PINNs' training, a rigorous understanding of it remains an interesting open problem.
\begin{figure}[!t]
\centering
  \includegraphics[width=.32\textwidth]{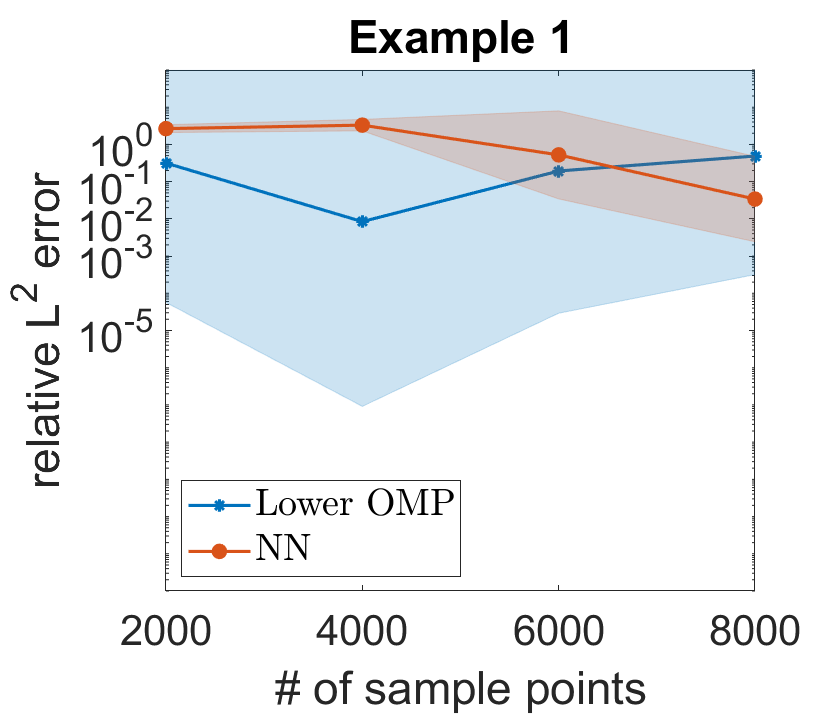}
  \includegraphics[width=.32\textwidth]{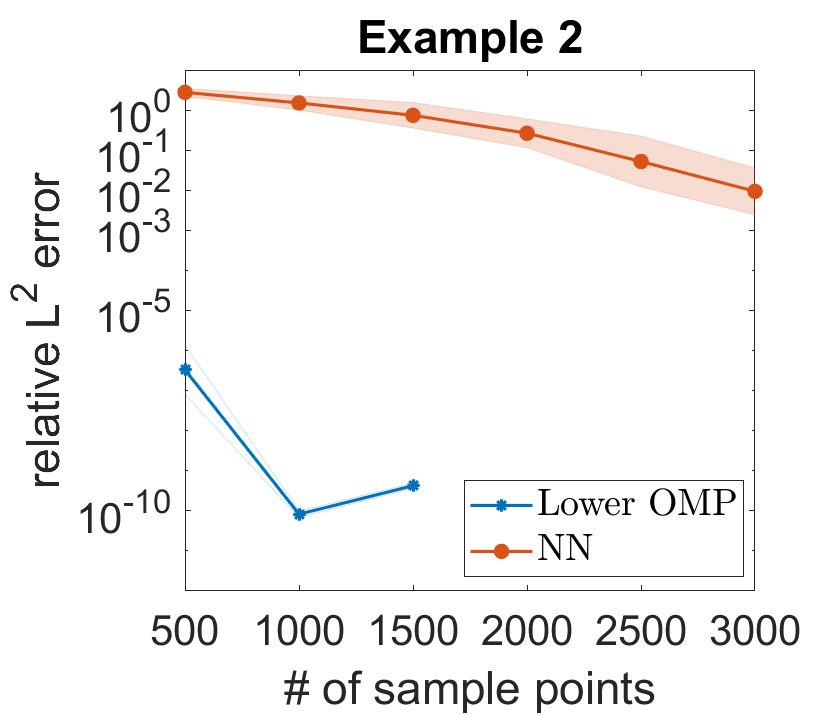}
  \includegraphics[width=.32\textwidth]{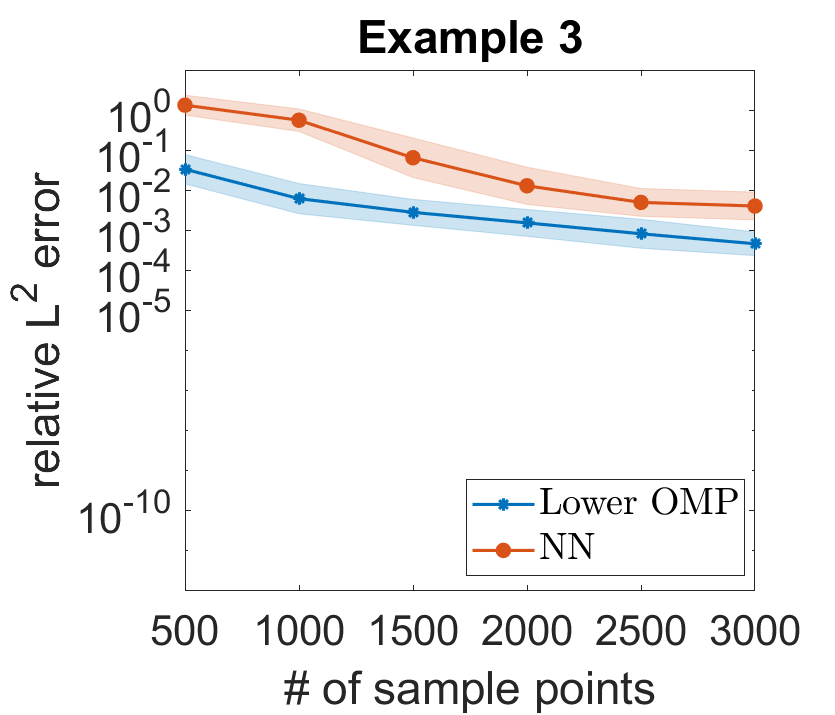}
\caption{(Comparison) Relative $L^2$-error versus the number of sample points $m$ for the exact solutions defined in Eq.\,\eqref{eq:exact2}-\eqref{eq:exact3}, Here dimension $d= 30$, number of sample points $m=3000$. The neural networks have the same structure as in \S\ref{Sec:NN_performance} and the relative error is measured after 30000 epochs. The cardinality of the index sets in the Lower adaptive OMP is constrained to be less than or equal to $m/2$} \label{Fig:Comparison}
\end{figure}

\section{Proof of the practical existence theorem}
\label{sec:proofs}
In order to prove Theorem \ref{thm:PET}, we first extend the CFC Theorem of \cite{wang2022compressive} to handle periodic diffusion equations with a reaction term in \S\ref{sec:CFC_theory}. This requires first outlining the theory of bounded Riesz systems and forming a connection with the CFC matrix and its corresponding Gram matrix. After illustrating this CFC convergence result, we will show a proof of Theorem~\ref{thm:PET} in \S\ref{sec:proof_PET}.

\subsection{Compressive Fourier collocation for diffusion-reaction problems}
\label{sec:CFC_theory}

In this subsection we extend the CFC convergence analysis of \cite{wang2022compressive} from diffusion to diffusion-reaction problems of the form \eqref{eq:diffusion_eq_periodic_1}. The general proof strategy is to present sufficient conditions for the PDE coefficients $a$ and $\rho$ such that the CFC matrix \eqref{eq:def_A_b} is a random sampling matrix from a bounded Riesz system \citep{brugiapaglia2021sparse}. This step is the most technical aspect of the proof. The rest uses the existing framework of sparse recovery in bounded Riesz systems whereby a sufficient condition on the sampling complexity is chosen such that the CFC satisfies the robust null space property with high-probability. From compressive sensing theory, we then obtain recovery guarantees for the CFC approximation.

\subsubsection{Some preliminary facts}
\label{sec:facts_about_bounded_riesz_systems}

We start by recalling the definition of the bounded Riesz system. We restrict our attention to Riesz systems in $L^2(\mathbb{T}^d)$, although the definition can be extended to general Hilbert spaces. Note in the following definition that $\ell^2(\Lambda;\mathbb{C})$ is the set of $\ell^2$-integrable complex sequences, indexed by $\Lambda$, i.e.,
$
\ell^2(\Lambda;\mathbb{C})=\left\{\bm{z}=(z_{\bm{\nu}})_{\bm{\nu}\in\Lambda} : \sum_{\bm{\nu}\in\Lambda}|z_{\bm{\nu}}|^2<\infty\right\}.
$

\begin{definition}[Bounded Riesz System]
\label{def:Riesz}
Let $\Lambda \subseteq \mathbb{Z}^d$, $0<b_\Phi \leq B_\Phi < \infty$, and let $\ell^2(\Lambda;\mathbb{C})$ denote the space of sequences $\bm{z}=(z_{\bm{\nu}})_{\bm{\nu}\in\Lambda}$ with $\|\bm{z}\|_2<\infty$. A set of functions $\{\Phi_{\bm{\nu}}\}_{\bm{\nu} \in \Lambda} \subset L^2(\mathbb{T}^d)$ is a \emph{Riesz system} with constants $b_{\Phi}$ and $B_{\Phi}$ if 
$$
b_\Phi\|\bm{z}\|^2_2 
\leq \left\| \sum_{\bm{\nu} \in \Lambda} z_{\bm{\nu}}
\Phi_{\bm{\nu}}\right\|^2_{L^2} 
\leq B_\Phi \|\bm{z}\|^2_2, 
\quad \forall \bm{z}=(z_{\bm{\nu}})_{\bm{\nu} \in \Lambda} \in \ell^2(\Lambda;\mathbb{C}).
$$
The constants $b_{\Phi}$ and $B_{\Phi}$ are called \emph{lower and upper Riesz constants}, respectively. Moreover, the system $\{\Phi_{\bm{\nu}}\}_{\bm{\nu} \in \Lambda}$ is \emph{bounded} if there exists a constant $0< K_{\Phi} < \infty$ such that
$$
\|\Phi_{\bm{\nu}}\|_{L^\infty} \leq K_{\Phi}, \quad \forall \bm{\nu} \in \Lambda.
$$
\end{definition}
Note that any $L^2$-orthonormal system is a Riesz system with $b_{\Phi} = B_{\Phi} = 1$. In particular, the Fourier system $\{F_{\bm{\nu}}\}_{\bm{\nu}\in\mathbb{Z}^d}$ is a bounded Riesz system with $b_{\Phi} = B_{\Phi} = K_{\Phi} = 1$

Recalling the definition \eqref{eq:def_operator_L} of $\mathscr{L}$, we define 
\begin{equation}
\label{eq:def_Phi}
\Phi_{\bm{\nu}} = \mathscr{L}[\Psi_{\bm{\nu}}], \quad \forall  \bm{\nu} \in \Lambda,
\end{equation}
where $\{\Psi_{\bm{\nu}}\}_{\bm{\nu}\in\mathbb{Z}^d}$ is the renormalized Fourier system given in \eqref{eq:def_Psi}. To show that the system $\{\Phi_{\bm{\nu}}\}_{\bm{\nu} \in \Lambda}$ defined in \eqref{eq:def_Phi} is a Riesz system, it is convenient to consider its \emph{Gram matrix} $G \in \mathbb{C}^{N\times N}$, where $N = |\Lambda|$, defined by
\begin{equation}
\label{eq:def_G}
    G_{\bm{\nu}\bm{\mu}}=\left< \Phi_{\bm{\nu}}, \Phi_{\bm{\mu}}\right>, \quad \forall \bm{\nu},\bm{\mu}\in \Lambda.
\end{equation}
We note in passing that, thanks to the normalization factor $1/\sqrt{m}$ in \eqref{eq:def_A_b}, we have $\mathbb{E}[A^*A] = G$ (this follows from a direct computation and the fact that the random collocation points $\bm{y}_1,\ldots,\bm{y}_m$ are independently and uniformly distributed over $\mathbb{T}^d$). The significance of the Gram matrix $G$ relies on the fact that it yields the following norm equivalence:
\begin{equation}
\label{eq:norm_equivalence}    
\left\| \sum_{\bm{\nu} \in \Lambda} c_{\bm{\nu}} \Phi_{\bm{\nu}} \right\|_{L^2}^2 
= \left< \sum_{\bm{\nu} \in \Lambda} c_{\bm{\nu}} \Phi_{\bm{\nu}},\sum_{\bm{\nu} \in \Lambda} c_{\bm{\nu}} \Phi_{\bm{\nu}}\right>
=\bm{c}^{T}G\bm{c}, \quad \forall \bm{c} \in \mathbb{C}^N.
\end{equation}
Note that $G$ is a Hermitian positive semidefinite matrix. Hence, it has only real nonnegative eigenvalues. The Courant–Fischer–Weyl min-max principle implies that, if $0 < b_{\Phi} \leq B_{\Phi} <\infty$ are such that
\begin{equation}
\label{eq:two-sided_bound}
    b_\Phi \leq \lambda_{\min}(G)\leq \lambda_{\max}(G) \leq B_\Phi,
\end{equation}
then $\{\Phi_{\bm{\nu}}\}_{\bm{\nu}\in\Lambda}$ is a Riesz system with constants $b_{\Phi}$ and $B_{\Phi}$.
Hence, estimating the lower and upper Riesz constants of $\{\Phi_{\bm{\nu}}\}_{\bm{\nu}\in\Lambda}$ corresponds to finding two-sided spectral bounds for the Gram matrix $G$. To obtain this type of spectral bounds, we will employ Gershgorin's circle theorem, see, e.g., \citep[Theorem 6.1.1]{horn2012matrix}, of which Lemma~\ref{lem:expansion_G} is a direct consequence. Here we state a useful corollary of the Gershgorin circle theorem.

\begin{lemma}[Gershgorin's circle theorem for Hermitian matrices]
\label{lem:Gershgorin_Hermitian_mat}
Let $A\in\mathbb{C}^{N\times N}$ be a Hermitian matrix. Then, all eigenvalues of $A$ lie in the real interval 
\begin{equation*}
    \left[\min_{i\in[N]}\left\{A_{ii}-\sum_{j\neq i}|A_{ij}|\right\},\max_{i\in[N]}\left\{A_{ii}+\sum_{j\neq i}|A_{ij}|\right\}\right].
\end{equation*}
\end{lemma}
Using the Fourier expansion \eqref{eq:diff_expansion} of $a$, it is possible to compute an explicit formula for the entries of the Gram matrix $G$.

\begin{lemma}[Explicit formula for the Gram matrix]\label{lem:expansion_G}
Let $\Lambda\subseteq\mathbb{Z}^d$, $r_{\bm{\nu}} \in \mathbb{C}$, with $\bm{\nu}\in\Lambda$, be generic rescaling constants for the Fourier system such that $\Psi_{\bm{\nu}}=r_{\bm{\nu}}F_{\bm{\nu}}$, and consider a diffusion coefficient $a\in C^1(\mathbb{T}^d)$ having Fourier expansion \eqref{eq:diff_expansion} with $T = \mathbb{Z}^d\setminus\{\bm{0}\}$. Then, the elements of $G$ admit the following explicit formula in terms of the Fourier coefficients $(a_{\boldsymbol{\tau}})_{\boldsymbol{\tau}\in\mathbb{Z}^d}$ of a and the reaction term $\rho\in \mathbb{R}$:
\begin{align*}
    G_{\bm{\nu}\bm{\mu}} 
    &  = r_{\bm{\nu}}\bar{r}_{\bm{\mu}}\bigg(16\pi^4\sum_{\bm{\tau}\in \mathbb{Z}^d}\sum_{\bm{\tau'}\in \mathbb{Z}^d}  (\bm{\tau} \cdot \bm{\nu} + \| \bm{\nu} \|_2^2) (\bm{\tau'} \cdot \bm{\mu} + \| \bm{\mu} \|_2^2) a_{\bm{\tau}} \bar{a}_{\bm{\tau'}}
    \delta_{\bm{\tau}+\bm{\nu},\bm{\tau'}+\bm{\mu}} + \rho ^2
    \delta_{\bm{\nu},\bm{\mu}}\\
    & \quad + 4\pi^2\rho \sum_{\bm{\tau}\in \mathbb{Z}^d}  (\bm{\tau} \cdot \bm{\nu} +  \| \bm{\nu} \|_2^2) a_{\bm{\tau}} 
    \delta_{\bm{\tau}+\bm{\nu},\bm{\mu}} 
    + 4\pi^2 \rho \sum_{\bm{\tau'}\in \mathbb{Z}^d}  (\bm{\tau'} \cdot \bm{\mu} +  \| \bm{\mu} \|_2^2)  \bar{a}_{\bm{\tau'}} 
    \delta_{\bm{\nu},\bm{\tau'}+\bm{\mu}}\bigg),
\end{align*}
for each $\bm{\nu}, \bm{\mu}$ in $\Lambda\setminus\{\bm{0}\}$ and where $\delta_{\bm{\nu},\bm{\mu}}$ denotes the Kronecker delta. Moreover, 
\begin{equation}
\label{eq:G_00_G_nu0_G_0nu}
    G_{\bm{0}\bm{0}}=  r_{\bm{0}}^2\rho ^2 \quad \text{and} \quad G_{\bm{\bm{\nu}}\bm{0}} = G_{\bm{\bm{0}\bm{\nu}}}=0, \quad \forall \bm{\nu} \in \Lambda\setminus\{\bm{0}\}.
\end{equation}
\end{lemma}

\begin{proof} Before proving the identity, we note that gradients and Laplacians of the Fourier basis functions defined in \eqref{eq:Fourier_system} can be easily computed as
\begin{equation}
\label{eq:diff_Fourier}
\nabla F_{\bm{\nu}} = (2\pi i\bm{\nu}) F_{\bm{\nu}} \quad \text{and} \quad
\Delta F_{\bm{\nu}} = - 4\pi^2 \|\bm{\nu}\|_2^2 F_{\bm{\nu}}, \quad 
\forall  \bm{\nu} \in \mathbb{Z}^d.
\end{equation}
Moreover, 
\begin{equation}
\label{eq:prod_Fourier}
    F_{\bm{\nu}}F_{\bm{\mu}} = F_{\bm{\nu} +\bm{\mu}},\quad\forall \bm{\nu},\bm{\mu} \in \mathbb{Z}^d.
\end{equation}

To prove the desired formula for $G_{\bm{\nu},\bm{\mu}}$, we expand the inner product in \eqref{eq:def_G}. Using the above properties, the $L^2$-orthonormality of the Fourier basis $\{F_{\bm{\nu}}\}_{\bm{\nu} \in \mathbb{Z}^d}$, and recalling the expansion \eqref{eq:diff_expansion} of $a$, if $\rho$ is a constant, we see that
\begin{align*}
G_{\bm{\nu} \bm{\mu}} 
& = \langle \nabla a \cdot \nabla \Psi_{\bm{\nu}}
+ a \Delta \Psi_{\bm{\nu}} - \rho \Psi_{\bm{\nu}},
 \nabla a \cdot \nabla \Psi_{\bm{\mu}}
+ a \Delta \Psi_{\bm{\mu}} - \rho \Psi_{\bm{\mu}}
\rangle \\
& = \langle \nabla a \cdot \nabla \Psi_{\bm{\nu}}
+ a \Delta \Psi_{\bm{\nu}}, \nabla a \cdot \nabla\Psi_{\bm{\mu}}
+ a \Delta \Psi_{\bm{\mu}} \rangle + 
 \rho ^2 \langle \Psi_{\bm{\nu}} , \Psi_{\bm{\mu}} \rangle \\
& \quad - \rho \langle \nabla a \cdot \nabla \Psi_{\bm{\nu}}
+ a \Delta \Psi_{\bm{\nu}}, \Psi_{\bm{\mu}} \rangle  - \rho \langle \Psi_{\bm{\nu}} , \nabla a \cdot \nabla\Psi_{\bm{\mu}}
+ a \Delta \Psi_{\bm{\mu}} \rangle\\
& = \sum_{\bm{\tau}}\sum_{\bm{\tau'}}
\langle((2 i \pi \bm{\tau}) \cdot (2 i \pi \bm{\nu}) - 4 \pi^2
\| \bm{\nu} \|_2^2)  a_{\bm{\tau}} F_{\bm{\tau}}\Psi_{\bm{\nu}},  ((2 i \pi \bm{\bm{\tau}}') \cdot (2 i \pi \bm{\mu}) - 4 \pi^2 \| \bm{\mu} \|_2^2) 
a_{\bm{\tau'}} F_{\bm{\tau'}}\Psi_{\bm{\mu}}\rangle \\
& \quad + \rho ^2 \langle \Psi_{\bm{\nu}} , \Psi_{\bm{\mu}}\rangle - \rho \sum_{\bm{\tau}}
\langle((2 i \pi \bm{\tau}) \cdot (2 i \pi \bm{\nu}) - 4 \pi^2
\| \bm{\nu} \|_2^2)  a_{\bm{\tau}} F_{\bm{\tau}}\Psi_{\bm{\nu}}, \Psi_{\bm{\mu}} \rangle \\
&\quad - \rho \sum_{\bm{\tau'}}
\langle \Psi_{\bm{\nu}},  ((2 i \pi \bm{\bm{\tau}}') \cdot (2 i \pi \bm{\mu}) - 4 \pi^2 \| \bm{\mu} \|_2^2) 
a_{\bm{\tau'}} F_{\bm{\tau'}}\Psi_{\bm{\mu}}\rangle  \\
& = 16\pi^4r_{\bm{\nu}}\bar{r}_{\bm{\mu}}\sum_{\bm{\tau}}\sum_{\bm{\tau'}}  (\bm{\tau} \cdot \bm{\nu} + \| \bm{\nu} \|_2^2) (\bm{\tau'} \cdot \bm{\mu} + \| \bm{\mu} \|_2^2) a_{\bm{\tau}} \bar{a}_{\bm{\tau'}}
\langle  F_{\bm{\tau}} F_{\bm{\nu}},  F_{\bm{\tau'}} F_{\bm{\mu}}\rangle + \rho ^2 r_{\bm{\nu}}\bar{r}_{\bm{\mu}}
\langle F_{\bm{\nu}}, F_{\bm{\mu}} \rangle\\
& \quad + 4\pi^2\rho r_{\bm{\nu}} \bar{r}_{\bm{\mu}} \sum_{\bm{\tau}}  (\bm{\tau} \cdot \bm{\nu} +  \| \bm{\nu} \|_2^2) a_{\bm{\tau}} 
\langle F_{\bm{\tau}} F_{\bm{\nu}},  F_{\bm{\mu}}\rangle 
+ 4\pi^2\rho r_{\bm{\nu}} \bar{r}_{\bm{\mu}} \sum_{\bm{\tau'}}  (\bm{\tau'} \cdot \bm{\mu} +  \| \bm{\mu} \|_2^2)  \bar{a}_{\bm{\tau'}} 
\langle  F_{\bm{\nu}},  F_{\bm{\tau'}} F_{\bm{\mu}}\rangle \\
& = r_{\bm{\nu}}\bar{r}_{\bm{\mu}}\bigg(16\pi^4\sum_{\bm{\tau}}\sum_{\bm{\tau'}}  (\bm{\tau} \cdot \bm{\nu} + \| \bm{\nu} \|_2^2) (\bm{\tau'} \cdot \bm{\mu} + \| \bm{\mu} \|_2^2) a_{\bm{\tau}} \bar{a}_{\bm{\tau'}}
\delta_{\bm{\tau}+\bm{\nu},\bm{\tau'}+\bm{\mu}} + \rho ^2
\delta_{\bm{\nu},\bm{\mu}}\\
& \quad + 4\pi^2\rho \sum_{\bm{\tau}}  (\bm{\tau} \cdot \bm{\nu} +  \| \bm{\nu} \|_2^2) a_{\bm{\tau}} 
\delta_{\bm{\tau}+\bm{\nu},\bm{\mu}} 
+ 4\pi^2 \rho\sum_{\bm{\tau'}}  (\bm{\tau'} \cdot \bm{\mu} +  \| \bm{\mu} \|_2^2)  \bar{a}_{\bm{\tau'}} 
\delta_{\bm{\nu},\bm{\tau'}+\bm{\mu}}\bigg),
\end{align*}
where all the summations are over $\bm{\tau},\bm{\tau}' \in \mathbb{Z}^d$, and both $\bm{\nu}, \bm{\mu}$ not equal to $\bm{0}$. Moreover, for any $\bm{\nu} \neq \bm{0}$, we have
\begin{align*}
G_{\bm{\nu} \bm{0}} 
& = \langle \nabla a \cdot \nabla \Psi_{\bm{\nu}}
+ a \Delta \Psi_{\bm{\nu}} - \rho \Psi_{\bm{\nu}},
  - \rho \Psi_{\bm{0}} \rangle \\
& = \sum_{\bm{\tau}}
\langle((2 i \pi \bm{\tau}) \cdot (2 i \pi \bm{\nu}) - 4 \pi^2
\| \bm{\nu} \|_2^2)  a_{\bm{\tau}}  F_{\bm{\tau}}\Psi_{\bm{\nu}}, 
- \rho \Psi_{\bm{0}} \rangle \\
& = 4\pi^2\rho\sum_{\bm{\tau}} (\bm{\tau} \cdot \bm{\nu} +  \| \bm{\nu} \|_2^2)a_{\bm{\tau}} \rho\langle F_{\bm{\tau}} F_{\bm{\nu}} , F_{\bm{0}}\rangle \\
& = 4\pi^2\rho\sum_{\bm{\tau}} (\bm{\tau} \cdot \bm{\nu} +  \| \bm{\nu} \|_2^2)a_{\bm{\tau}} \rho\delta_{\bm{\tau}+\bm{\nu},\bm{0}} = 0.
\end{align*}
Similarly, one can show that $G_{\bm{0}\bm{\mu}} = 0$, and $G_{\bm{0}\bm{0}} = \rho^2\langle r_{\bm{0}}F_{\bm{0}}, r_{\bm{0}}F_{\bm{0}} \rangle = r_{\bm{0}}^2\rho ^2$.
\end{proof}

We also need an auxiliary result about norm equivalencies in $H^2(\mathbb{T}^d)$ that will be necessary to derive the final error bound of Theorem \ref{thm:CFC}. Letting $\vertiii{u}^2 :=\|u\|_{L^2}^2+\|\Delta u\|_{L^2}^2$, the following lemma shows that this norm is equivalent to the canonical $H^2$-norm. 
\begin{lemma}
\label{lem:H2_norm_equivalence}
    The norms $\normIII{\cdot}$ and $\|\cdot\|_{H^2}$ are equivalent. Specifically, $\sqrt{2/3}\|u\|_{H^2}\leq \normIII{u} \leq \|u\|_{H^2}$ for every $u \in H^2(\mathbb{T}^d)$.
\end{lemma}
\begin{proof}
Let $u(\bm{x}) = \sum_{\bm{\nu} \in \mathbb{Z}^d} {c}_{\bm{\nu}} \exp(2 \pi i \bm{\nu} \cdot \bm{x})$
and recall that $
\|u\|_{H^2}^2 = \|u\|_{L^2}^2 + \|\nabla u\|_{L^2}^2 + \|\nabla^2 u\|_{L^2}^2. 
$
Using Parseval's identity, we have
\begin{align*}
\|\nabla u\|_{L^2}^2 
 = (2 \pi)^2\int_{\mathbb{T}^d} \sum_{k=1}^d\left|\sum_{\bm{\nu} \in \mathbb{Z}^d} c_{\bm{\nu}} \nu_k \exp(2 \pi i \bm{\nu} \cdot \bm{x})\right|^2 d \bm{x}
 & = (2 \pi)^2\sum_{k=1}^d \sum_{\bm{\nu} \in \mathbb{Z}^d} |c_{\bm{\nu}}|^2 | \nu_k|^2 \\
 & = (2 \pi)^2 \sum_{\bm{\nu} \in \mathbb{Z}^d} |c_{\bm{\nu}}|^2 \| \bm{\nu}\|_2^2.
\end{align*}
Moreover,  recalling that $\|\nabla^2 u\|_{L^2}^2 = \int_{\mathbb{T}^d} \|\nabla^2u(\bm{x})\|_F^2 d \bm{x}$ and using Parseval's identity again, we obtain
\begin{align*}
\|\nabla^2 u\|_{L^2}^2 
& = (2 \pi)^4 \int_{\mathbb{T}^d}\sum_{k=1}^d \sum_{l=1}^d \left|\sum_{\bm{\nu} \in \mathbb{Z}^d} c_{\bm{\nu}}  \nu_k \nu_l \exp(2 \pi i \bm{\nu} \cdot \bm{x})\right|^2 d \bm{x}\\
& = (2 \pi)^4 \sum_{k=1}^d \sum_{l=1}^d \sum_{\bm{\nu} \in \mathbb{Z}^d} |c_{\bm{\nu}}|^2  (\nu_k \nu_l)^2 
 = (2 \pi)^4 \sum_{\bm{\nu} \in \mathbb{Z}^d} |c_{\bm{\nu}}|^2  \|\bm{\nu}\|_2^4.
\end{align*}
Now, observe that
$\|u\|_{L^2} = \|\bm{c}\|_2$.
Thus,
$$
\|u\|_{H^2}^2 = \sum_{\bm{\nu} \in \mathbb{Z}^d} |c_{\bm{\nu}}|^2 (1 + (2\pi)^2 \|\bm{\nu}\|_2^2 + (2\pi)^4 \|\bm{\nu}\|_2^4) = \bm{c}^* D_1 \bm{c},
$$
with $D_1 = \text{diag}((1 + (2\pi)^2 \|\bm{\nu}\|_2^2 + (2\pi)^4 \|\bm{\nu}\|_2^4)_{\bm{\nu} \in \mathbb{Z}^d})$. Moreover, 
\begin{align*}
\|\Delta u\|_{L^2}^2 
& = (2 \pi)^4 \int_{\mathbb{T}^d} \left|\sum_{\bm{\nu} \in \mathbb{Z}^d} c_{\bm{\nu}}  \|\bm{\nu}\|_2^2 \exp(2 \pi i \bm{\nu} \cdot \bm{x})\right|^2 d \bm{x}
 = (2 \pi)^4 \sum_{\bm{\nu} \in \mathbb{Z}^d} |c_{\bm{\nu}}|^2  \|\bm{\nu}\|_2^4.
\end{align*}
Therefore, 
$$
\normIII{u}^2 = \sum_{\bm{\nu} \in \mathbb{Z}^d} |c_{\bm{\nu}}|^2  (1 + (2 \pi)^4 \|\bm{\nu}\|_2^4)
= \bm{c}^* D_2 \bm{c},
$$
where $D_2= \text{diag}((1 + (2\pi)^4 \|\bm{\nu}\|_2^4)_{\bm{\nu} \in \mathbb{Z}^d})$.
Hence, 
$$
\frac{\|u\|_{H^2}^2}{\normIII{u}^2}
= \frac{\bm{c}^* D_1 \bm{c}}{\bm{c}^* D_2 \bm{c}}
= \frac{\bm{d}^* D_3 \bm{d}}{\bm{d}^*\bm{d}}
=\frac{\sum_{\bm{\nu} \in \mathbb{Z}^d} (D_3)_{\bm{\nu}, \bm{\nu}} |d_{\bm{\nu}}|^2}{\sum_{\bm{\nu} \in \mathbb{Z}^d} |d_{\bm{\nu}}|^2}
\leq \sup_{\bm{\nu} \in \mathbb{Z}^d}((D_3)_{\bm{\nu}, \bm{\nu}}),
$$
where we made the change of variable $\bm{d} = D_2^{\frac12} \bm{c}$ and where 
$$
D_3 = D_2^{-\frac12}D_1D_2^{-\frac12}
= \text{diag}\left(\left(\frac{1 + (2\pi)^2 \|\bm{\nu}\|_2^2 + (2\pi)^4 \|\bm{\nu}\|_2^4}{1 + (2\pi)^4 \|\bm{\nu}\|_2^4}\right)_{\bm{\nu} \in \mathbb{Z}^d}\right).
$$
In particular,
$$
\sup_{\bm{\nu} \in \mathbb{Z}^d}((D_3)_{\bm{\nu}, \bm{\nu}})
= \sup_{\bm{\nu} \in \mathbb{Z}^d}\frac{1 + (2\pi)^2 \|\bm{\nu}\|_2^2 + (2\pi)^4 \|\bm{\nu}\|_2^4}{1 + (2\pi)^4 \|\bm{\nu}\|_2^4} \leq \frac{3}{2}
$$
The last bound is obtained by maximizing $x \mapsto (1+ x^2 + x^4 )/(1 + x^4)$ over $\mathbb{R}$. 
Similarly, we obtain the second inequality in the statement of equivalence from
$$
\frac{\normIII{u}^2}{\|u\|_{H^2}^2}
\leq \sup_{\bm{\nu} \in \mathbb{Z}^d}\frac{1 + (2\pi)^4 \|\bm{\nu}\|_2^4}{1 + (2\pi)^2 \|\bm{\nu}\|_2^2 + (2\pi)^4 \|\bm{\nu}\|_2^4} \leq 1.
$$
This concludes the proof.
\end{proof} 

\subsubsection{CFC convergence}

We are now ready to illustrate our CFC convergence result. Its proof is based on the same tools employed in \cite{wang2022compressive} and it is therefore an extension of the analysis presented therein. Keeping the CFC setup of \S\ref{sec:CFC} in mind, in order to recover a compressible solution $\hat{\bm{c}}$ to the linear system \eqref{eq:SC_system}, we utilize the \emph{Square-Root LASSO} (in short, SR-LASSO), see \citep{adcock2019correcting,belloni2011square} and references therein, a modified version of the original LASSO (Least Absolute Shrinkage and Selection Operator), which lacks a power of 2 on the data fidelity term. The SR-LASSO is defined as
\begin{equation}\label{eq:SRLasso}
    \min_{\bm{z} \in \mathbb{C}^N} \left\{ \|A\bm{z}-\bm{b}\|_2 + \lambda \|\bm{z}\|_1 \right\}.
\end{equation} 
The main benefit of this formulation is that the optimal choice of the tuning parameter $\lambda>0$ is independent of the noise level, see \citep[\S6.6.2]{adcock2022sparse}.

\begin{theorem}[Convergence of CFC for diffusion-reaction problems]\label{thm:CFC}
Given a dimension $d\in \mathbb{N}$, target sparsity $s\in\mathbb{N}$, hyperbolic cross order $n\in \mathbb{N}$ and probability of failure $\varepsilon \in (0,1)$, let $\Lambda = \Lambda^{\textnormal{HC}}_{d,n}\subset\mathbb{Z}^d$ and suppose $a \in C^1(\mathbb{T}^d)$ and $\rho \in \mathbb{R}$ satisfy \eqref{eq:suff_cond_a_rho_f} and \eqref{eq:diff_expansion}--\eqref{eq:suff_cond_PET}. Then, the system $\{\Phi_{\bm{\nu}}\}_{\bm{\nu}\in\Lambda}$ defined in \eqref{eq:def_Phi} is a bounded Riesz system in the sense of Definition \ref{def:Riesz} with constants
\begin{align} 
    b_{\Phi} & = a_{\bm{0}}^2  - \left( 2  a_{\bm{0}} + \frac{\rho}{2 \pi^2} \right) \beta - \beta^2>0,\\
    B_{\Phi} & = \left\|a \right\|_{H^1}^2 + \frac{\rho^2}{16 \pi^4} + \frac{ a_{\bm{0}}\rho }{2\pi^2} + \left( 2  a_{\bm{0}} + \frac{\rho}{2 \pi^2} \right) \beta+ \beta^2,
\end{align}
where $\beta = \sqrt{|T|}\|a-a_{\bm{0}}\|_{H^1}$, and 
$
K_{\Phi}=  a_{\bm{0}} + \beta + \frac{\rho}{4\pi^2}.
$
Moreover, the CFC solution $\hat{u}$ in \eqref{eq:def_uhat} approximating the high-dimensional periodic diffusion-reaction equation \eqref{eq:diffusion_eq_periodic_1}, whose coefficients $\hat{\bm{c}}$ are computed by solving the SR-LASSO problem \eqref{eq:SRLasso} with $A$ and $\bm{b}$ defined as in \eqref{eq:def_A_b} and with tuning parameter $d^{(1)}_{a,\rho}\sqrt{B_\Phi/s}<\lambda\leq d^{(2)}_{a,\rho}\sqrt{B_\Phi/s}$, satisfies the following with probability at least $1-\varepsilon$: if
\begin{equation}
\label{eq:sample_complexity}
    m\geq c^{(3)}_{a,\rho} \cdot s \cdot \log^2\left(c^{(4)}_{a,\rho}\cdot s\right)\cdot \left(\min\{\log(n)+d,\log(2n)\log(2d)\}+\log(\varepsilon^{-1})\right),
\end{equation}
then 
\begin{align}
\label{eq:CFC_first_error_bound}
    \|u-\hat{u}\|_{L^2} + \|(\Delta-\rho)(u-\hat{u})\|_{L^2}
    & \leq C^{(1)}_{a,\rho}\cdot\frac{\sigma_s(\bm{c}_\Lambda)_1}{\sqrt{s}} + C^{(2)}_{a,d,\rho} \cdot\left(\frac{\|u-u_{\Lambda}\|_{W^{2,\infty}}}{\sqrt{s}} + \|u-u_{\Lambda}\|_{H^2}\right).
\end{align}
Moreover, if $\rho < 1$, we also have
\begin{align}
\label{eq:CFC_H2_bound}
    \|u - \hat{u}\|_{H^2}
\leq C^{(3)}_{a,\rho}\cdot \frac{\sigma_s(\bm{c}_\Lambda)_1}{\sqrt{s}}+C^{(4)}_{a, d, \rho}\cdot \left(\frac{\|u-u_{\Lambda}\|_{W^{2,\infty}}}{\sqrt{s}}+\|u-u_{\Lambda}\|_{H^2}\right).
\end{align}
Here each constant depends only on the subscripted parameters. Moreover, the dependence of each constant on $d$ is at most linear (when present).
\end{theorem}

\begin{remark}
The above theorem holds for OMP recovery as well, under a sufficient condition on the ratio $b_{\Phi}/B_{\Phi}$, see \citep[Theorem 3]{wang2022compressive}.
\end{remark}

\begin{remark}
There is a gap between Theorem~\ref{thm:CFC} and the corresponding result in the diffusion setting \citep[Theorem 3.5]{wang2022compressive}. Namely, in the latter the diffusion coefficient $a$ can be nonsparse. Extending Theorem~\ref{thm:CFC} to allow for more general diffusion coefficients is an open problem.
\end{remark}

\begin{proof}
Given sparse diffusion and reaction terms $a$ and $\rho$ as in \eqref{eq:suff_cond_a_rho_f} and \eqref{eq:diff_expansion}--\eqref{eq:suff_cond_PET}, we determine that $\left\{\Phi_{\bm{\nu}}\right\}_{\bm{\nu}\in\Lambda}$ is a bounded Riesz system using the explicit form of the Gram matrix in Lemma \ref{lem:expansion_G}.

\paragraph{Step 1: Riesz property.} We find lower and upper Riesz constants $b_{\Phi}$ and $B_{\Phi}$ by establishing a two-sided spectral bound for the Gram matrix $G$, recall equation \eqref{eq:def_G}. First, observing that the first row (and column) of $G$ only has one nonzero entry (namely, $G_{\bm{00}}$), one eigenvalue of $G$ is $G_{\bm{00}} = a_{\bm{0}}^2$. To estimate the remaining eigenvalues, we use a special version of Gershgorin's circle theorem for Hermitian matrices, presented in Lemma \ref{lem:Gershgorin_Hermitian_mat}, see, e.g., \citep[Theorem 6.1.1]{horn2012matrix} for the general statement. Using Lemma~\ref{lem:expansion_G} with the generic normalization, the diagonal entries $G_{\bm{\nu} \bm{\nu}}$ with $\bm{\nu} \in \Lambda\setminus\{\bm{0}\}$ are given by
\begin{align*}
    G_{\bm{\nu} \bm{\nu}} & = 16\pi^4|r_{\bm{\nu}}|^2\sum_{\bm{\tau}\in T\cup\{\bm{0}\}}(\bm{\tau}\cdot\bm{\nu}+\|\bm{\nu}\|_2^2)^2|a_{\bm{\tau}}|^2+\rho^2|r_{\bm{\nu}}|^2+8\pi^2|r_{\bm{\nu}}|^2\|\bm{\nu}\|_2^2\rho a_{\bm{0}}\\
    & = 16\pi^4\|\bm{\nu}\|_2^4|r_{\bm{\nu}}|^2\left[\left(  a_{\bm{0}} + \frac{\rho}{4 \pi^2 \| \bm{\nu} \|_2^2} \right)^2  + 
    \sum_{\bm{\tau} \in T}\left( \frac{\bm{\tau} \cdot \bm{\nu} +  \| \bm{\nu} \|_2^2}{\| \bm{\nu} \|_2^2} \right)^2 \left|a_{\bm{\tau}}\right|^2\right].
\end{align*} 
Using the Cauchy-Schwarz inequality and inputting the actual normalization used in \eqref{eq:def_Psi}, we see that, for all $\bm{\nu} \in \Lambda\setminus\{\bm{0}\}$,  
$$
|r_{\bm{\nu}}|^2\left(a_{\bm{0}} +\frac{\rho}{4\pi^2\|\bm{\nu}\|_2^2}\right)^2
\leq  \frac{\lvert G_{\bm{\nu} \bm{\nu}} \rvert}{16\pi^4\|\bm{\nu}\|_2^4} 
\leq  |r_{\bm{\nu}}|^2\left[a_{\bm{0}}^2 + \frac{\rho^2}{16 \pi^4} + \frac{ |a_{\bm{0}}|\rho}{2\pi^2} + \sum_{\bm{\tau} \in T } \left(\| \bm{\tau} \|_2 + 1 \right)^2 \lvert a_{\bm{\tau}} \rvert^2\right]\nonumber 
$$
which, in turn, implies
\begin{equation}
\label{eq:G_diag}
    a_{\bm{0}}^2\leq\lvert G_{\bm{\nu} \bm{\nu}} \rvert \leq\left\| a \right\|^2_{H^1} + \frac{\rho^2}{16 \pi^4} + \frac{ |a_{\bm{0}}|\rho}{2\pi^2}.
\end{equation}
The inequality defining the upper bound can be proved as follows. Using the definition of $H^1$-norm, the differentiation properties \eqref{eq:diff_Fourier}, and the fact that the Fourier system $\{F_{\bm{\nu}}\}_{\bm{\nu} \in \mathbb{Z}^d}$ is $L^2$-orthonormal, we obtain 
\begin{align}
\label{eq:LowerBoundOnH1NormForDiffCoeff}
     \left\| a \right\|^2_{H^1} 
     & =   \left\| a \right\|^2_{L^2}
 + \sum^{d}_{l=1}  \left\| \frac{\partial a}{\partial x_l} \right\|^2_{L^2}\nonumber
     = \sum_{\bm{\tau}\in T \cup \{\bm{0}\}}|a_{\bm{\tau}}|^2 + \sum^{d}_{l=1}\sum_{\bm{\tau}\in T \cup \{\bm{0}\}}|2\pi a_{\bm{\tau}} \bm{\tau}_l|^2\nonumber\\
     & =  |a_{\bm{0}}|^2 + \sum_{\bm{\tau} \in T}\left(1+(2\pi)^2 \left\|\bm{\tau}\right\|_2^2\right)|a_{\bm{\tau}}|^2
      \geq  |a_{\bm{0}}|^2+ \sum_{\bm{\tau} \in T} \left(\| \bm{\tau} \|_2 + 1 \right)^2 \lvert a_{\bm{\tau}} \rvert^2, 
\end{align}
which proves \eqref{eq:G_diag}.
To apply Gershgorin's circle theorem, we now bound the sum of all off-diagonal entries in the $\bm{\nu}$-th row of $G$. Using Lemma~\ref{lem:expansion_G} again, the definition of the Kronecker delta, \eqref{eq:G_00_G_nu0_G_0nu}  and the fact that $|r_{\bm{\nu}}| \leq 1/(4\pi^2 \|\bm{\nu}\|_2^2) \leq 1/4\pi^2$ for every $\bm{\nu} \neq \bm{0}$, we obtain
\begin{align*}
\sum_{\bm{\mu} \in \Lambda \setminus \{\bm{\nu}\}} \lvert G_{\bm{\nu} \bm{\mu}} \rvert
& = \sum_{\bm{\mu} \in \Lambda \setminus \{\bm{\nu}\}} \bigg| r_{\bm{\nu}}\bar{r}_{\bm{\mu}}\bigg(16\pi^4\sum_{\bm{\tau}\in \mathbb{Z}^d} (\bm{\tau} \cdot \bm{\nu} + \| \bm{\nu} \|_2^2) \left((\bm{\tau} + \bm{\nu} - \bm{\mu}) \cdot \bm{\mu} + \| \bm{\mu} \|_2^2 \right) a_{\bm{\tau}} \bar{a}_{\bm{\tau} + \bm{\nu} - \bm{\mu}}
     \\
    & \quad + 4\pi^2\rho \left( (\bm{\mu} - \bm{\nu}) \cdot \bm{\nu} +  \| \bm{\nu} \|_2^2 \right) a_{\bm{\mu} - \bm{\nu}}  
    + 4\pi^2 \rho  \left((\bm{\nu} - \bm{\mu}) \cdot \bm{\mu} +  \| \bm{\mu} \|_2^2 \right)  \bar{a}_{\bm{\nu} - \bm{\mu}} 
    \bigg) \bigg| \\
    & \leq \sum_{\bm{\tau}\in \mathbb{Z}^d} \left(\frac{|\bm{\tau} \cdot \bm{\nu}|}{\| \bm{\nu} \|_2^2} + 1 \right) \lvert a_{\bm{\tau}} \rvert \sum_{\bm{\mu} \in \Lambda \setminus\{ \bm{\nu}\}} \left(\frac{ |\left(\bm{\tau} +\bm{\nu} -\bm{\mu} \right) \cdot \bm{\mu}|}{\| \bm{\mu} \|_2^2} + 1 \right) \lvert a_{\bm{\bm{\tau} +\bm{\nu} -\bm{\mu}}} \rvert \\
    & \quad + \frac{\rho}{4\pi^2}\sum_{\bm{\mu} \in \Lambda \setminus\{ \bm{\nu}\}} \left( \frac{| (\bm{\mu} - \bm{\nu}) \cdot \bm{\nu} |}{\|\bm{\nu}\|_2^2} + 1 \right)|a_{\bm{\mu} - \bm{\nu}}| \\
    &  \quad + \frac{\rho}{4\pi^2}\sum_{\bm{\mu} \in \Lambda \setminus\{ \bm{\nu}\}} \left( \frac{| (\bm{\nu} - \bm{\mu}) \cdot \bm{\mu} |}{\|\bm{\mu}\|_2^2} + 1 \right)|\bar{a}_{\bm{\nu} - \bm{\mu}}|.
\end{align*}
Using the Cauchy-Schwarz inequality $|\bm{\tau} \cdot \bm{\nu}| \leq \|\bm{\tau}\|_2 \cdot \|\bm{\nu}\|_2$ and the fact that $\|\bm{\mu}\|_2 \geq 1$ and $\|\bm{\nu}\|_2 \geq 1$ for all $\bm{\mu},\bm{\nu} \neq \bm{0}$, we see that, for all $\bm{\nu} \in \Lambda\setminus\{\bm{0}\}$,
\begin{align*}
\sum_{\bm{\mu} \in \Lambda \setminus \{\bm{\nu}\}} \lvert G_{\bm{\nu} \bm{\mu}} \rvert
& \leq \sum_{\bm{\tau} \in T \cup \{\bm{0}\}} \left(\| \bm{\tau} \|_2 + 1 \right) \lvert a_{\bm{\tau}} \rvert \sum_{\bm{\mu} \in \Lambda \setminus\{ \bm{\nu}\}} \left(\| \bm{\tau} +\bm{\nu} -\bm{\mu}\|_2 + 1 \right) \lvert a_{\bm{\bm{\tau} +\bm{\nu} -\bm{\mu}}} \rvert \\
& \quad +  \frac{\rho}{4\pi^2}\sum_{\bm{\tau} \in T} \left(\| \bm{\tau} \|_2 + 1 \right) \lvert a_{\bm{\tau}} \rvert +  \frac{\rho}{4\pi^2}\sum_{\bm{\tau} \in T} \left(\| \bm{\tau} \|_2 + 1 \right) \lvert \bar{a}_{\bm{\tau}} \rvert.
\end{align*}
Substituting $\bm{\tau}'=\bm{\tau}+\bm{\nu}-\bm{\mu}$ (which implies $\bm{\tau}' \neq \bm{\tau}$), recalling that $(a_{\bm{\nu}})_{\bm{\nu} \in \mathbb{Z}^d}$ is supported on $T \cup \{\bm{0}\}$, and separating the $a_{\bm{0}}$ term, we obtain
\begin{align*}
    \sum_{\bm{\mu} \in \Lambda \setminus\{ \bm{\nu}\}} \lvert G_{\bm{\nu} \bm{\mu}} \rvert
    & \leq \sum_{\bm{\tau} \in T \cup \{\bm{0}\}} \left(\| \bm{\tau} \|_2 + 1 \right) \lvert a_{\bm{\tau}} \rvert \sum_{\bm{\tau'}  \in T \cup \{\bm{0}\} \setminus\{ \bm{\tau}\}} \left(\| \bm{\tau'} \|_2 + 1 \right) \lvert a_{\bm{\tau'}} \rvert + \frac{\rho}{2 \pi^2}\sum_{\bm{\tau} \in T} \left(\| \bm{\tau} \|_2 + 1 \right) \lvert a_{\bm{\tau}} \rvert\\
    & =  |a_{\bm{0}}| \sum_{\bm{\tau'}  \in T } \left(\| \bm{\tau'} \|_2 + 1 \right) \lvert a_{\bm{\tau'}} \rvert + \sum_{\bm{\tau}  \in T } \left(\| \bm{\tau} \|_2 + 1 \right) \lvert a_{\bm{\tau}} \rvert \sum_{\bm{\tau'}  \in T \cup \{\bm{0}\}\setminus\{ \bm{\tau}\}} \left(\| \bm{\tau'} \|_2 + 1 \right) \lvert a_{\bm{\tau'}} \rvert \\
    & \quad + \frac{\rho}{2 \pi^2}\sum_{\bm{\tau} \in T} \left(\| \bm{\tau} \|_2 + 1 \right) \lvert a_{\bm{\tau}} \rvert\\
    & \leq 2  |a_{\bm{0}}| \sum_{\bm{\tau'}  \in T } \left(\| \bm{\tau'} \|_2 + 1 \right) \lvert a_{\bm{\tau'}} \rvert + \sum_{\bm{\tau}  \in T} \left(\| \bm{\tau} \|_2 + 1 \right) \lvert a_{\bm{\tau}} \rvert \sum_{\bm{\tau'} \in T} \left(\| \bm{\tau'} \|_2 + 1 \right) \lvert a_{\bm{\tau'}} \rvert \\
    & \quad +\frac{\rho}{2 \pi^2}\sum_{\bm{\tau} \in T} \left(\| \bm{\tau} \|_2 + 1 \right) \lvert a_{\bm{\tau}} \rvert\\
    & = \left( 2  a_{\bm{0}}+ \frac{\rho}{2 \pi^2} \right)\sum_{\bm{\tau'}  \in T } \left(\| \bm{\tau'} \|_2 + 1 \right) \lvert a_{\bm{\tau'}} \rvert + \left( \sum_{\bm{\tau}  \in T } \left(\| \bm{\tau} \|_2 + 1 \right) \lvert a_{\bm{\tau}} \rvert \right)^2.
\end{align*}
Applying the Cauchy–Schwarz inequality, denoting $t=|T|$, and separating the $a_{\bm{0}}$ term, we obtain the bound 
\begin{equation}
\label{eq:Bound_beta}
    \sum_{\bm{\tau} \in T } \left(\| \bm{\tau} \|_2 + 1 \right) \lvert a_{\bm{\tau}} \rvert 
    \leq \sqrt{t}
    \left( \sum_{\bm{\tau}\in T} \left(\| \bm{\tau} \|_2 + 1 \right)^2 \lvert a_{\bm{\tau}} \rvert ^2 \right)^{1/2}
     \leq \sqrt{t} \left( \left\| a \right\|^2_{H^1} -   a_{\bm{0}} ^2 \right)^{1/2} := \beta.
\end{equation}
Combining the above inequalities yields
\begin{equation} \label{eq:G_off_diag}
    \sum_{\bm{\mu} \in \Lambda \setminus \{\bm{\nu}\} } \lvert G_{\bm{\nu} \bm{\mu}} \rvert 
    \leq \left( 2  a_{\bm{0}} + \frac{\rho}{2 \pi^2} \right) \beta+\beta^2.
\end{equation}
Finally, applying the Gershgorin circle theorem on $G$ combining \eqref{eq:G_diag} and \eqref{eq:G_off_diag} and recalling that $\rho^2$ is an eigenvalue of $G$, we obtain the Riesz constants 
\begin{align} \label{eq:eta_sparse_conclusion}
b_{\Phi} & = a_{\bm{0}}^2  - \left( 2  a_{\bm{0}} + \frac{\rho}{2 \pi^2} \right) \beta - \beta^2>0, \\
B_{\Phi} & = \left\|a \right\|_{H^1}^2 + \frac{\rho^2}{16 \pi^4} + \frac{ a_{\bm{0}}\rho }{2\pi^2} + \left( 2  a_{\bm{0}} + \frac{\rho}{2 \pi^2} \right) \beta+ \beta^2.
\end{align}
We note that the positivity of $b_\Phi$ is required for equation \eqref{eq:sample_complexity} to be well-defined and this is assured by the sufficient condition \eqref{eq:suff_cond_PET}.

\paragraph{Step 2: Boundedness.} To bound the essential supremum of the Riesz system, we use properties \eqref{eq:diff_Fourier} and \eqref{eq:prod_Fourier} 
\begin{align*}
\| \Phi_{\bm{\nu}}\|_{L^{\infty}}
&=\|\mathscr{L}\left[\Psi_{\bm{\nu}}\right]\|_{L^{\infty}}
=\|-\nabla\cdot\left(a\nabla\Psi_{\bm{\nu}}\right)+\rho\Psi_{\bm{\nu}}\|_{L^{\infty}}\\
&\leq\|\nabla a\cdot\nabla\Psi_{\bm{\nu}}\|_{L^{\infty}}+\| a\Delta\Psi_{\bm{\nu}}\|_{L^{\infty}}+\|\rho\Psi_{\bm{\nu}}\|_{L^{\infty}}\\
&\leq\left\|\sum_{\bm{\tau}\in T\cup\{\bm{0}\}}2\pi i\bm{\tau}a_{\bm{\tau}}F_{\bm{\tau}}\right\|_{L^{\infty}}\cdot\left\|\frac{2\pi i\bm{\nu}}{4\pi^2\|\bm{\nu}\|^2_2}F_{\bm{\nu}}\right\|_{L^{\infty}}+\left\|\sum_{\bm{\tau}\in T\cup\{\bm{0}\}}a_{\bm{\tau}}F_{\bm{\tau}}\right\|_{L^{\infty}}+\frac{\rho}{4\pi^2\|\bm{\nu}\|^2_2}\\
&\leq\sum_{\bm{\tau}\in T\cup\{\bm{0}\}}\frac{\lvert\bm{\tau}\cdot\bm{\nu}\rvert}{\|\bm{\nu}\|^2_2}|a_{\bm{\tau}}|+\sum_{\bm{\tau}\in T\cup\{\bm{0}\}}|a_{\bm{\tau}}|+\frac{\rho}{4\pi^2\|\bm{\nu}\|^2_2}\\
&=\sum_{\bm{\tau}\in T\cup\{\bm{0}\}}\left(\frac{\lvert\bm{\tau}\cdot\bm{\nu}\rvert}{\|\bm{\nu}\|^2_2}+1\right)|a_{\bm{\tau}}|+\frac{\rho}{4\pi^2\|\bm{\nu}\|^2_2}
\end{align*}
Finally, applying the Cauchy inequality, equation \eqref{eq:LowerBoundOnH1NormForDiffCoeff} and the fact that $\|\bm{\nu}\|_2\geq 1$, we obtain 
\begin{align}
\label{eq:EssSup_noTail}
    \|\Phi_{\bm{\nu}}\|_{L^\infty}&\leq\sum_{\bm{\tau}\in T\cup\{\bm{0}\}}\left(\frac{|\bm{\tau}\cdot\bm{\nu}|}{\|\bm{\nu}\|^2}+1\right)|a_{\bm{\tau}}|+\frac{ \rho }{4\pi^2\|\bm{\nu}\|_2^2}\nonumber\\
    &\leq  a_{\bm{0}} +\sum_{\bm{\tau}\in T}\left(\|\bm{\tau}\|+1\right)|a_{\bm{\tau}}|+\frac{ \rho }{4\pi^2}\leq   a_{\bm{0}} + \beta + \frac{\rho}{4\pi^2}.
\end{align}
This proves that $\{\Phi_{\bm{\nu}}\}_{\bm{\nu}\in\Lambda}$ is a bounded Riesz system (Definition~\ref{def:Riesz}). Now, it remains to show accurate and stable recovery guarantees for the problem \eqref{eq:SRLasso}. The machinery for this is enabled by recent advances in sparse recovery for bounded Riesz systems \citep{brugiapaglia2021sparse}.

\paragraph{Step 3: Bounded Riesz property $\Longrightarrow$ error bound.}
The strategy is to pick a lower bound on the sample complexity, such that $A$ satisfies the robust null-space property with high probability, and then appeal to the recovery bounds already supplied by the literature surrounding SR-LASSO. We begin with a simple re-normalization of the compressive Fourier collocation matrix, by letting $\tilde{A}=A/\sqrt{B_\phi}$. We thus consider a rescaled version of the SR-LASSO problem
\begin{equation}
    \hat{\bm{c}}\in\argmin_{\bm{z} \in \mathbb{C}^N}\|A\bm{z}-\bm{b}\|_2+\lambda\|\bm{z}\|_1,
\end{equation}
defined as
\begin{equation}
     \hat{\tilde{\bm{c}}}\in 
    \argmin_{\tilde{\bm{z}}\in \mathbb{C}^N}\|\tilde{A}\tilde{\bm{z}}-\bm{b}\|_2+ \tilde{\lambda}\|\tilde{\bm{z}}\|_1,
\end{equation}
where $\tilde{\bm{z}}=\sqrt{B_{\Phi}}\bm{z}$ and $\tilde{\lambda} = \lambda/\sqrt{B_{\Phi}}$. The two minimizers are such that $\sqrt{B_{\Phi}} \hat{\bm{c}} = \hat{\tilde{\bm{c}}}$. An inspection of the proof of \cite[Theorem 2.6]{brugiapaglia2021sparse} reveals that condition 
\begin{equation}
    \label{eq:suff_cond_m}
    m\geq c_0\left(\frac{\max\{1,B_\Phi\}}{b_\Phi}\right)^2K_\Phi^2s\log^2\left(sK_\Phi^2\frac{\max\{1,B_\Phi\}}{b_\Phi}\right)\log(eN),
\end{equation}
where $c_0>0$ is a universal constant, is sufficient for $\tilde{A}$ to satisfy the rNSP with constants $\rho=1/2$ and $\gamma=2B_{\Phi}/b_{\Phi}$ and with probability at least $1-\varepsilon/2$.\footnote{Observe that the probability of failure \cite[Theorem 2.6]{brugiapaglia2021sparse} is $4\exp{-c_0'(b_{\Phi}/\max\{1,B_{\Phi}\})^2m/(sK_\Phi^2)}$, for some constant $c_0'>0$, and which is bounded by $\varepsilon$ thanks to \eqref{eq:sample_complexity}.} 

Note that condition \eqref{eq:suff_cond_m} is implied by \eqref{eq:sample_complexity} with
\begin{align*}
c^{(3)}_{a,\rho}  = c_0\left(\frac{\max\{1,B_\Phi\}}{b_\Phi}\right)^2K_\Phi^2
\quad \text{and}\quad
c^{(4)}_{a,\rho}  = K_\Phi^2\frac{\max\{1,B_\Phi\}}{b_\Phi},
\end{align*}
and thanks to the upper bound \eqref{eq:card_bound_HC} on the cardinality of the hyperbolic cross $N = |\Lambda|$.

Hence, we can apply \cite[Theorem 6.29]{adcock2022sparse} with $\bm{w}=1$ and $\bm{b}=\tilde{A}\tilde{\bm{c}}_\Lambda+\bm{e}=A\bm{c}_\Lambda+\bm{e}$, where 
$$
\bm{e} = \frac{1}{\sqrt{m}}\left(\mathscr{L}[u-u_{\Lambda}](\bm{x}_i)\right)_{i\in[m]} \in \mathbb{C}^m,
$$
yielding the following coefficient recovery guarantee:
\begin{equation}
    \|\hat{\tilde{\bm{c}}}-\tilde{\bm{c}}_\Lambda\|_2\leq c_1\frac{\sigma_s(\tilde{\bm{c}}_\Lambda)_1}{\sqrt{s}}+\frac{1}{2}\left(\frac{c_1}{\sqrt{s}\tilde{\lambda}}+c_2\right)\|\bm{e}\|_2,
\end{equation}
for constants $c_1 = 2(1+\rho)^2/(1-\rho)= 9$ and $c_2 = 2(3+\rho)\gamma/(1-\rho) = 28 B_\Phi/b_\Phi$, see also Eq.~(6.15) from \citep{adcock2022sparse}. 
Note also that Eq.~(6.47) \citep{adcock2022sparse} allows us to pick a specific range for $\tilde{\lambda}$ and, consequently, for $\lambda$:
\begin{equation}
    \sqrt{B_\Phi}\frac{d^{(1)}_{a,\rho}}{\sqrt{s}}\leq\lambda\leq\sqrt{B_\Phi}\frac{d^{(2)}_{a,\rho}}{\sqrt{s}},
\end{equation}
where $d^{(1)}_{a,\rho}$ and $d^{(2)}_{a,\rho}$ are such that $0<d^{(1)}_{a,\rho}\leq d^{(2)}_{a,\rho}\leq\frac{1+\rho}{(3+\rho)\gamma}=\frac{3b_\Phi}{14B_\Phi}$. Finally, changing back to the original variables, and inserting values for constants, we obtain
\begin{align}\label{eq:CoeffBound}
    \|\hat{\bm{c}}-\bm{c}_\Lambda\|_2&\leq\frac{c_1\sigma_s(\bm{c}_\Lambda)_1}{\sqrt{s}}+\frac{1}{2\sqrt{B_\Phi}}\left(\frac{c_1\sqrt{B_\Phi}}{\sqrt{s}\lambda}+c_2\right)\|\bm{e}\|_2\nonumber\\
    &=\frac{9\sigma_s(\bm{c}_\Lambda)_1}{\sqrt{s}}+\frac{1}{2}\left(\frac{9}{\sqrt{s}\lambda}+\frac{28\sqrt{B_\Phi}}{b_\Phi}\right)\|\bm{e}\|_2\nonumber\\
    &\leq\frac{9\sigma_s(\bm{c}_\Lambda)_1}{\sqrt{s}}+\frac{1}{2}\left(\frac{9}{\sqrt{B_{\Phi}}d_1}+\frac{28\sqrt{B_\Phi}}{b_\Phi}\right)\|\bm{e}\|_2.
\end{align}

Thanks to the rescaling used in \eqref{eq:def_Psi}, we have $(-\Delta + \rho)\Psi_{\bm{\nu}} = F_{\bm{\nu}}$ for all $\bm{\nu} \in \mathbb{Z}^d$. Hence, for any $v\in H^2(\mathbb{T}^d)$ such that $v = \sum_{\bm{\nu}\in \mathbb{Z}^d} d_{\bm{\nu}} \Psi_{\bm{\nu}}
$, using Parseval's identity, we obtain 
\begin{equation}
\label{eq:consequence_Parseval}
\|(\Delta - \rho)v\|_{L^2} = \|\bm{d}\|_2
\quad
\text{and}
\quad
\|v\|_{L^2} = \|R \bm{d}\|_2 \leq (a_{\bm{0}}/\rho) \| \bm{d}\|_2,
\end{equation}
where $R = \text{diag}((r_{\bm{\nu}})_{\bm{\nu}\in\mathbb{Z}^d})$ with $r_{\bm{\nu}}=1/(4\pi^2 \|\bm{\nu}\|_2^2 + \rho/a_{\bm{0}})$ being the rescaling factors, and where we used the fact that $r_{\bm{\nu}}\leq a_{\bm{0}}/\rho$ for every $\bm{\nu}\in\mathbb{Z}^d$.
Using the triangle inequality, we see that
\begin{align*}
\|(\Delta-\rho)(u-\hat{u})\|_{L^2} & \leq \|(\Delta-\rho)(u-u_{\Lambda})\|_{L^2} + \|(\Delta-\rho)(u_{\Lambda}-\hat{u})\|_{L^2}\\
\|u-\hat{u}\|_{L^2} & \leq \|u-u_{\Lambda}\|_{L^2} + \|u_{\Lambda}-\hat{u}\|_{L^2}.
\end{align*}
These, combined with \eqref{eq:CoeffBound} and \eqref{eq:consequence_Parseval} (with $v = \hat{u}-u_{\Lambda}$ and $\bm{d} = \hat{\bm{c}} - \bm{c}_{\Lambda}$), imply the following error bounds:
\begin{align*}
    \|(\Delta-\rho)(u-\hat{u})\|_{L^2}
    & \leq \|(\Delta-\rho)(u-u_\Lambda)\|_{L^2}+\frac{9}{\sqrt{s}}\sigma_s(\bm{c}_\Lambda)_1+\frac{1}{2}\left(\frac{9}{\sqrt{B_\Phi}d_1}+\frac{28\sqrt{B_\Phi}}{b_\Phi}\right)\|\bm{e}\|_2\\
    \|u-\hat{u}\|_{L^2}
    & \leq \|u-u_\Lambda\|_{L^2}+\frac{a_{\bm{0}}}{\rho}\left(\frac{9}{\sqrt{s}}\sigma_s(\bm{c}_\Lambda)_1+\frac{1}{2}\left(\frac{9}{\sqrt{B_\Phi}d_1}+\frac{28\sqrt{B_\Phi}}{b_\Phi}\right)\|\bm{e}\|_2\right).
\end{align*}
Summing these two inequalities and observing that, thanks to Lemma~\ref{lem:H2_norm_equivalence}, for any $v\in H^2(\mathbb{T}^d)$,
$$
\|(\Delta-\rho)v\|_{L^2} + \|v\|_{L^2} 
\leq \|\Delta u\|_{L^2} + (1 + \rho) \|v\|_{L^2} 
\leq \sqrt{2} (1 + \rho )  \vertiii{v}
\leq \sqrt{2} (1 + \rho)    \|v\|_{H^2}, 
$$
we see that
\begin{align}
\nonumber
\|u-\hat{u}\|_{L^2} + \|(\Delta-\rho)(u-\hat{u})\|_{L^2}
& \leq \sqrt{2} (1 + \rho)    \|u-u_\Lambda\|_{H^2}\\
& \quad + \left(1 + \frac{a_{\bm{0}}}{\rho}\right)\left(\frac{9}{\sqrt{s}}\sigma_s(\bm{c}_\Lambda)_1+\frac{1}{2}\left(\frac{9}{\sqrt{B_\Phi}d_1}+\frac{28\sqrt{B_\Phi}}{b_\Phi}\right)\|\bm{e}\|_2\right).
\label{eq:CFC_proof_first_error_bound}
\end{align}

\paragraph{Step 4: Bounding the truncation error $\|\bm{e}\|_2$ and further simplifications.}
Following the discussion of \cite[page 9]{wang2022compressive} and the notation of \cite[\S7.6.1]{adcock2022sparse}, we define 
\begin{equation*}
    \|\bm{e}\|_2=\sqrt{\frac{1}{m}\sum_{i=1}^m\left|\mathscr{L}[u-u_\Lambda](y_i)\right|^2}=:E_{\Lambda,\text{disc}}(\mathscr{L}[u]).
\end{equation*}
Now, \eqref{eq:sample_complexity} implies  $m\geq 2s\log(4/\varepsilon)$, giving the following bound for $E_{\Lambda,\text{disc}}(\mathscr{L}[u])$ with probability $1-\varepsilon/2$:
\begin{equation}
\label{eq:bounding_truncation_error}
    E_{\Lambda,\text{disc}}(\mathscr{L}[u])\leq\sqrt{2}\left(\frac{E_{\Lambda,\infty}(\mathscr{L}[u])}{\sqrt{s}}+E_{\Lambda,2}(\mathscr{L}[u])\right),
\end{equation}
thanks to \cite[Lemma 7.11]{adcock2022sparse}, where $E_{\Lambda,\infty}(\mathscr{L}[u])=\|\mathscr{L}[u-u_\Lambda]\|_{L^\infty}$ and $E_{\Lambda,2}(\mathscr{L}[u])=\|\mathscr{L}[u-u_\Lambda]\|_{L^2}$. Hence, \eqref{eq:sample_complexity} is sufficient so that both the rNSP and \eqref{eq:bounding_truncation_error} hold with probability $1-\varepsilon$ and we obtain
\begin{equation}
\label{eq:error_after_bounding_truncation_error}
    \|\bm{e}\|_2 \leq \sqrt{2}\left(\frac{\|\mathscr{L}[u-u_\Lambda]\|_{L^\infty}}{\sqrt{s}}+\|\mathscr{L}[u-u_\Lambda]\|_{L^2}\right),
\end{equation}
Now, we estimate the two error terms on the right-hand side. Note that, for $1\leq p \leq \infty$, and any $v\in H^2(\mathbb{T}^d)$,
\begin{align}
\label{eq:PDE_operator_trunc_bound}
    \|\mathscr{L}[v]\|_{L^p}
    & = \|-\nabla \cdot (a \nabla v) + \rho v\|_{L^p}\nonumber\\
    & = \|-\nabla a \cdot \nabla v - a \Delta v + \rho v\|_{L^p}\nonumber\\
    & \leq \|\nabla a \cdot \nabla v\|_{L^p} + \|a \Delta v\|_{L^p} + \|\rho v\|_{L^p}.
\end{align}
Using the inequality
$\|\bm{v}\|_1 \leq d^{1-\frac1p}\|\bm{v}\|_p$, for all $\bm{v} \in \mathbb{R}^d$,
we begin by analyzing the first two terms. First, we see that
\begin{align*}
    \|\nabla a \cdot \nabla v\|_{L^p}
    & =\left\|\sum_{k=1}^d\frac{\partial a}{\partial x_k}\frac{\partial v}{\partial x_k}\right\|_{L^p}
    \leq\sum_{k=1}^d\left\|\frac{\partial a}{\partial x_k}\frac{\partial v}{\partial x_k}\right\|_{L^p}
    \leq \max_{k\in\left[d\right]}\left\|\frac{\partial a}{\partial x_k}\right\|_{L^\infty}\sum_{k=1}^d\left\|\frac{\partial v}{\partial x_k}\right\|_{L^p}\\
    & =\max_{k\in\left[d\right]}\left\|\frac{\partial a}{\partial x_k}\right\|_{L^\infty}\left\|\left(\left\|\frac{\partial v}{\partial x_k}\right\|_{L^p}\right)_{k=1}^d\right\|_{1}
    \leq \|a\|_{W^{1,\infty}}\cdot d^{1-\frac{1}{p}}\left\|\left(\left\|\frac{\partial v}{\partial x_k}\right\|_{L^p}\right)_{k=1}^d\right\|_{p}\\
    & \leq d^{1-\frac{1}{p}} \|a\|_{W^{1,\infty}}\|v\|_{W^{2, p}}.
\end{align*}
Second, we have
\begin{align*}
    \|a\Delta v\|_{L^p} 
    & \leq \|a\|_{L^\infty}\left\|\sum_{k=1}^d\frac{\partial^2 v}{\partial x_k^p}\right\|_{L^p}\leq\|a\|_{L^\infty}\sum_{k=1}^d\left\|\frac{\partial^2 v}{\partial x_k^2}\right\|_{L^p}=\|a\|_{L^\infty}\left\|\left(\left\|\frac{\partial^2v}{\partial x_k^2}\right\|_{L^p}\right)_{k=1}^d\right\|_{1}\\
    & \leq \|a\|_{L^\infty}\cdot d^{1-\frac{1}{p}}\left\|\left(\left\|\frac{\partial^2v}{\partial x_k^2}\right\|_{L^p}\right)_{k=1}^d\right\|_{p}
    \leq d^{1-\frac{1}{p}}\|a\|_{L^{\infty}}\|v\|_{W^{2,p}}. 
\end{align*}
The third term is more easily estimated as
$
\|\rho v\|_{L^p}\leq \|\rho\|_{L^\infty}\|v\|_{L^p}\leq\|\rho\|_{L^\infty}\|v\|_{W^{2,p}}.
$
Altogether, when $p=\infty$, we have
\begin{equation}
\label{eq:trunc_error_infty}
    \|\mathscr{L}\left[v\right]\|_{L^\infty}\leq \|\nabla a\cdot \nabla v\|_{L^\infty}+\|a\Delta v\|_{L^\infty}+\|\rho v\|_{L^\infty}
    \leq \left(2d\|a\|_{W^{1,\infty}}+\|\rho\|_{L^\infty}\right) \|v\|_{W^{2,\infty}},
\end{equation}
and, for $p=2$,
\begin{equation}
\label{eq:trunc_error_L2}
    \|\mathscr{L}\left[v\right]\|_{L^2}
    \leq \|\nabla a\cdot \nabla v\|_{L^2}+\|a\Delta v\|_{L^2}+\|\rho v\|_{L^2}\leq \left(2\sqrt{d}\|a\|_{W^{1,\infty}}+\|\rho\|_{L^\infty}\right) \|v\|_{H^2}.
\end{equation}
Letting $v = u-u_{\Lambda}$ and combining \eqref{eq:bounding_truncation_error} with the above estimates yields
$$
\|\bm{e}\|_2 \leq \sqrt{2} \left(\|a\|_{W^{1,\infty}}+\|\rho\|_{L^\infty}\right) \left(\frac{\|u-u_{\Lambda}\|_{W^{2,\infty}}}{\sqrt{s}} + \|u-u_{\Lambda}\|_{H^2}\right).
$$
Plugging the above bound into \eqref{eq:CFC_proof_first_error_bound} leads to the recovery guarantee \eqref{eq:CFC_first_error_bound},  
where 
\begin{align*}
    C^{(1)}_{a, \rho} & = 9\left(1 + \frac{a_{\bm{0}}}{\rho}\right)\\
    C^{(2)}_{a,d,\rho} & = \frac{1}{2}\left(1 + \frac{a_{\bm{0}}}{\rho}\right)\left(\frac{9}{\sqrt{B_\Phi}d_1}+\frac{28\sqrt{B_\Phi}}{b_\Phi}\right) \sqrt{2} \left(\|a\|_{W^{1,\infty}}+\|\rho\|_{L^\infty} + 1\right) +\sqrt{2}(1+\rho).
\end{align*}

When $\rho < 1$, the bound \eqref{eq:CFC_H2_bound} can be obtained by invoking Lemma~\ref{lem:H2_norm_equivalence} again and observing that, for every $v\in H^2(\mathbb{T}^d)$,
$$
\|(\Delta-\rho)v\|_{L^2} + \|v\|_{L^2}
\geq \|\Delta v\|_{L^2} + (1-\rho) \|v\|_{L^2}
\geq (1-\rho) \vertiii{v}
\geq (1-\rho) \sqrt{\frac{2}{3}} \|v\|_{H^2}.
$$
Combining this inequality with \eqref{eq:CFC_proof_first_error_bound}, we obtain \eqref{eq:CFC_H2_bound}, where 
$$
C^{(3)}_{a,\rho} = \frac{C^{(1)}_{a,\rho}}{\sqrt{2/3}(1-\rho)}
\quad \text{and}\quad
C^{(4)}_{a,d,\rho} = \frac{C^{(2)}_{a,d,\rho}}{\sqrt{2/3}(1-\rho)}.
$$
This concludes the proof.
\end{proof}

\subsection{Proof of Theorem~\ref{thm:PET}}
\label{sec:proof_PET}

We proceed with a proof of the main result. For a summary of the general proof strategy, we refer to the proof sketch after the statement of Theorem~\ref{thm:PET}. Most of the technical efforts will be devoted to constructing the class of periodic PINNs, $\mathcal{N}$ and bounding the depth and width of its neural networks. This will allow us to leverage the CFC convergence theory (Theorem~\ref{thm:CFC}) to derive the desired error bounds \eqref{eq:PET_first_error_bound}--\eqref{eq:PET_H2_bound}.

\paragraph{Construction of the network class $\mathcal{N}$.} As explained in the proof sketch, we will construct networks $\psi \in\mathcal{N}$ able to produce linear combinations of the form \eqref{eq:def_u_Lambda}. Letting $N = |\Lambda|$, this class is of the form  
\begin{equation}
\label{eq:def_class_N}
\mathcal{N} = \left\{\psi : \mathbb{R}^d \to \mathbb{C} : \psi(\bm{x}) = \bm{z}^\top\psi_{\Lambda}(\bm{x}), \bm{z} \in \mathbb{C}^N\right\},
\end{equation}
where $\psi_{\Lambda}(\bm{x}) = (\psi_{\bm{\nu}}(\bm{x}))_{\bm{\nu} \in \Lambda}$ and $\psi_{\bm{\nu}}$ is a network replicating the $\bm{\nu}$th (rescaled) Fourier function, i.e., $\psi_{\bm{\nu}} \equiv \Psi_{\bm{\nu}}$, with $\Psi_{\bm{\nu}}$ defined as in \eqref{eq:def_Psi}. The main task is therefore to explicitly construct the networks $\psi_{\bm{\nu}}$. We will do so by composing different network modules, aimed at producing more and more complex functions of the input variables $\bm{x} = (x_j)_{j\in [d]}$, as follows:
\begin{enumerate}
    \item First, using a periodic layer $\bm{v}^{(2)}\circ \bm{q}^{(1)}$ with $\bm{v}^{(2)}(\bm{x}) = \bm{x}$ (see \S\ref{sec:PINNs}), we will generate univariate trigonometric functions $\sin(2\pi x_j)$ and $\cos(2\pi x_j)$, for $j\in [d]$.
    \item Then, through a network module $\bm{v}_{\text{trig} \to \text{multi-freq}}$, we will compute univariate trignometric functions corresponding to all the frequencies we need to generate functions in the system $\{\Psi_{\bm{\nu}}\}_{\bm{\nu}\in \Lambda}$. These are given by $\cos(2 \pi \nu x_j)$ and $\sin(2\pi \nu x_j)$, for $\nu \in \{ -\nu_{\max}, \ldots, \nu_{\max}\}$ and $j \in [d]$, where
    $$
    \nu_{\max} := \max \{ |\nu_j| : \bm{\nu} = (\nu_j)_{j=1}^d \in \Lambda\} = n-1,
    $$
    is the maximum absolute frequency of multi-indices in the hyperbolc cross $\Lambda = \Lambda^{\text{HC}}_{d,n}$. 
    
    \item Finally, using a module $\bm{v}_{\text{multi-freq} \to \Lambda}$, we will generate the rescaled Fourier functions $\Psi_{\bm{\nu}}(\bm{x})$ through tensorization.
\end{enumerate}
In summary, we have
$
\psi_{\Lambda} =  \bm{v}_{\text{multi-freq} \to \Lambda} \circ \bm{v}_{\text{trig} \to \text{multi-freq}} \circ \bm{q}^{(1)}.
$
After constructing $\psi_{\Lambda}$, we add a last linear layer with weights $\bm{z}^\top \in \mathbb{R}^{1\times N}$ to generate linear combinations of the form \eqref{eq:def_u_Lambda} (with $c_{\bm{\nu}} = z_{\bm{\nu}}$). Fig.~\ref{Fig:Summary_NN} summarizes our network construction. 
\begin{figure}[!t]
\centering
\includegraphics[width=\linewidth]{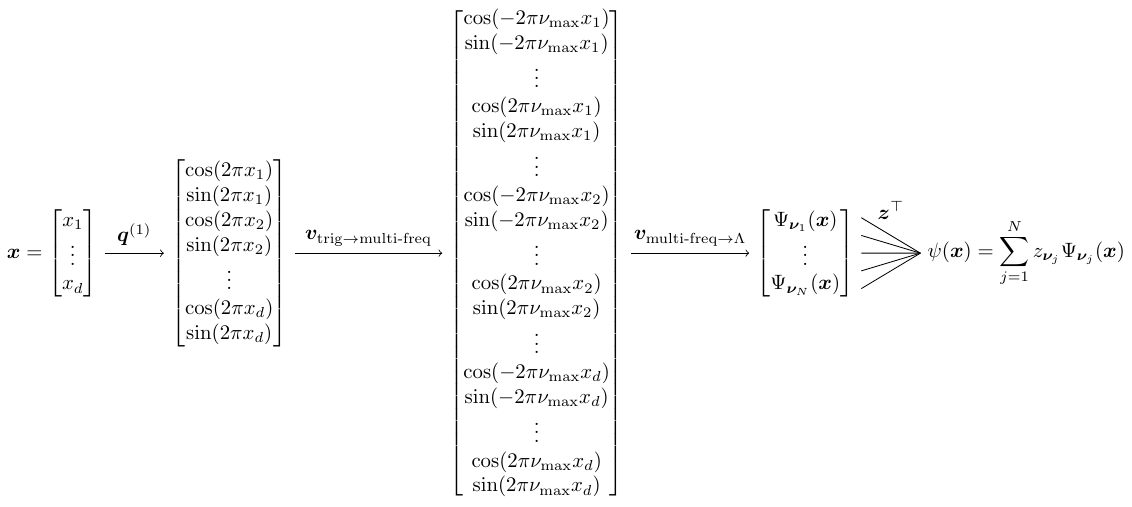}
\caption{\label{Fig:Summary_NN}Summary of a generic network $\psi\in \mathcal{N}$. }
\end{figure}

A key ingredient of the proof is the fact that neural networks with RePU activation can exactly replicate products of real numbers. Specifically, for any $\ell, k \in \mathbb{N}$ with $\ell\geq2$, \cite[Lemma 7.3]{adcock2023near}, which, in turn relies on techniques from \citep{opschoor2022exponential,schwab2019deep}, establishes the existence (along with an explicit construction) of a feedforward neural network $P^{(k
)}:\mathbb{R}^k \to \mathbb{R}$ with $\mathrm{RePU}_{\ell}$ activation able to exactly reproduce the product of $k$ numbers, i.e.,
$$
P^{(k)}(x_1, \ldots, x_k) = \prod_{j = 1}^k x_j,
\quad \forall \bm{x} = (x_j)_{j=1}^k \in \mathbb{R}^k.
$$
Moreover, the width and depth of this network are such that 
\begin{equation}
\label{eq:depth_width_bounds_product_net}
\text{width}(P^{(k)}) \leq p_{\ell}^{(1)} \cdot k \quad \text{and} \quad \text{depth}(P^{(k)}) \leq p^{(2)} \cdot \log_2(k),
\end{equation}
where $p_{\ell}^{(1)}>0$ depends on $\ell$ only and $p^{(2)}>0$ is a universal constant. Note that we will not track the dependence of $P^{(k)}$ on $\ell$ since this parameter is assumed to be fixed throughout the proof. Before showing in detail how to construct the various modules composing $\psi_{\Lambda}$, we establish a connection between the training of periodic PINNs in $\mathcal{N}$ and CFC approximation, which will yield the desired recovery guarantees.

\paragraph{CFC convergence theory $\Rightarrow$ error bounds \eqref{eq:PET_first_error_bound}--\eqref{eq:PET_H2_bound}.} To apply Theorem~\ref{thm:CFC}, we leverage two important facts: (i) the spectral basis \eqref{eq:def_Psi} is exactly replicated by the network $\psi_{\Lambda}$ and (ii) the weights of the last layer $\bm{z}^\top \in \mathbb{R}^{1\times N}$, corresponding to the coefficients of linear combinations of the form \eqref{eq:def_u_Lambda}, are the only trainable parameters of $\psi$. In this setting, training the network $\psi$ by minimizing the regularized RMSE loss in \eqref{eq:MSE+regularization} with regularization term
\begin{equation}
\label{eq:regulariz_term}
\mathcal{R}(\psi) = \|\bm{z}\|_1,
\end{equation}
is equivalent to solving the SR-LASSO problem \eqref{eq:SRLasso} with $A$ and $\bm{b}$ as in \eqref{eq:def_A_b}. Hence, we can simply apply Theorem~\ref{thm:CFC} to obtain the desired conclusion. \\

The rest of the proof is devoted to illustrating the construction of the periodic layer $\bm{q}^{(1)} \circ \bm{v}^{(2)}$ and the network modules $\bm{v}_{\text{trig} \to \text{multi-freq}}$ and $\bm{v}_{\text{multi-freq} \to \Lambda}$ in full detail.

\paragraph{Construction of $\psi_\Lambda$ (Step 1): the periodic layer.} Recalling the notation introduced in \eqref{eq:def_periodic_layer}, we let $l=2$, with $\phi_{i1} = 0$ and $\phi_{i1} = - \pi/2$, for $i\in [d]$. This leads to 
$$
\bm{x} \in \mathbb{R}^d \mapsto 
\bm{q}^{(1)}(\bm{x}) = 
\begin{bmatrix} \cos(2 \pi x_1)\\ 
\cos(2 \pi x_1 - \pi/2)\\ 
\vdots\\ 
\cos(2\pi x_d)\\
\cos(2 \pi x_d - \pi/2) 
\end{bmatrix} 
=
\begin{bmatrix} \cos(2 \pi x_1)\\ 
\sin(2 \pi x_1)\\ 
\vdots\\ 
\cos(2\pi x_d)\\
\sin(2 \pi x_d) 
\end{bmatrix}
=:
\begin{bmatrix}
    c_1\\
    s_1\\
    \vdots\\
    c_d\\
    s_d
\end{bmatrix}
\in \mathbb{R}^{2d},
$$
where we have used the short-hand notation $c_j = \cos(2 \pi x_j)$ and $s_j = \sin(2 \pi x_j)$. The second layer within the periodic layer is a linear (identity) layer, i.e. $\bm{v}^{(2)}(\bm{x}) = \bm{x}$. Hence, in summary, 
\begin{equation}
\label{eq:width_depth_periodic_layer}
\text{depth}(\bm{v}^{(2)} \circ \bm{q}^{(1)}) = 3 
\quad \text{and} \quad
\text{width}(\bm{v}^{(2)} \circ \bm{q}^{(1)}) = 2d.
\end{equation}

\paragraph{Construction of $\psi_\Lambda$ (Step 2): the module $\bm{v}_{\textnormal{trig} \to \textnormal{multi-freq}}$.} To generate sine and cosine functions at multiple frequencies given $c_i$, $s_i$, we first generate powers of $c_i$ via an inner network submodule module $\bm{v}_{\text{trig} \to \text{pow}_i}$ and then apply multi-angle trigonometric formulas (see, e.g., \cite{MultipleAngleFormulae}).

We start by presenting the inner module $\bm{v}_{\text{trig}\rightarrow\text{pow}_i}$, which constructs the powers of cosine for a fixed dimension $i$. This is given by
\begin{equation}
\label{eq:def_trig_pow_i}
\begin{bmatrix} 
    c_1 \\
    s_1\\
    \vdots\\ 
    c_d\\ 
    s_d
\end{bmatrix} 
\xrightarrow[\text{restriction}]{\text{linear}}
\begin{bmatrix} 
    c_i\\
    s_i
\end{bmatrix} 
\xrightarrow[\text{$\bm{x} \mapsto W\bm{x} + \bm{b}$}]{\text{affine layer}}
\begin{bmatrix}
    s_i\\
    1\\
    c_i\\
    c_i\\
    \vdots\\
    c_i
\end{bmatrix}
\xrightarrow[]{\text{powers}}
\begin{bmatrix}
    s_i\\
    1\\
    c_i\\
    c_i^2\\
    \vdots\\
    c_i^{\nu_{\text{max}}}
\end{bmatrix}
:= \text{pow}_i,
\end{equation}
where the weights and biases in $W \in \mathbb{R}^{(\nu_{\max}+2) \times 2}$ and $\bm{b}\in \mathbb{R}^{\nu_{\max}+2}$ and the network corresponding to $\xrightarrow{\text{powers}}$ are illustrated in Fig.~\ref{Fig:trig->pow}. Note that this network contains both RePU and linear activations, in accordance with \eqref{eq:def_hidden_layer}. 
\begin{figure}[!t]
\centering
\includegraphics[width=0.9\linewidth]{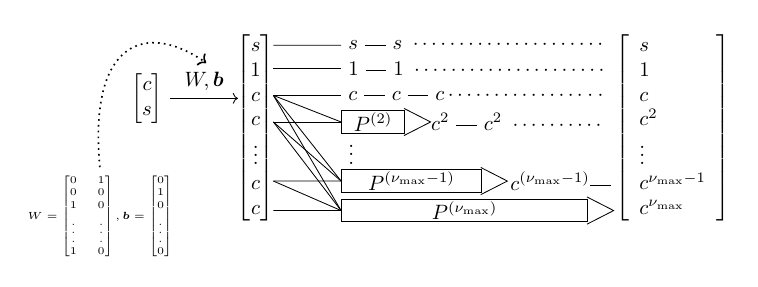}
\caption{\label{Fig:trig->pow} Part of the network submodule $\bm{v}_{\text{trig} \to \text{pow}_i}$ defined in \eqref{eq:def_trig_pow_i}, used to generate powers of cosine at dimension $i$.}
\end{figure}
Recalling \eqref{eq:depth_width_bounds_product_net} and observing that consecutive affine layers can be combined together without altering depth (since the composition of affine maps is an affine map), the depth and width of this module are bounded by 
\begin{align*}
    \text{width}(\bm{v}_{\text{trig} \to \text{pow}_i}) & 
    \leq   3 + \sum_{k=2}^{\nu_{\text{max}}}\text{width}\left(P^{(k)}\right)
    \leq C_1 \cdot \nu_{\text{max}}^2\cdot p_{\ell}^{(1)},\\
    \text{depth}(\bm{v}_{\text{trig} \to \text{pow}_i}) & \leq
    \text{depth}\left(P^{(\nu_{\text{max}})}\right) \leq p^{(2)}\cdot \log_2(\nu_{\text{max}}),
\end{align*}
where $C_1>0$ is a universal constant (recall that $p^{(1)}_\ell$ depends on the degree $\ell$ of the RePU activation, whereas $p^{(2)}$ does not).

Now, the output $\text{pow}_i$ of $\bm{v}_{\text{trig} \to \text{pow}_i}(\bm{x})$ is transformed using multi-angle formulae (see, e.g., \cite{MultipleAngleFormulae}) to obtain univariate sine and cosine functions at multiple frequencies. This corresponds to a linear layer
$$
\label{fig:powi->}
\text{pow}_i = \begin{bmatrix}
    s_i\\
    1\\
    c_i\\
    \vdots\\
    c_i^{\nu_{\text{max}}}
\end{bmatrix}
\xrightarrow[\text{formulae}]{\text{multi-angle}}
\begin{bmatrix}
    \cos(-\nu_{\text{max}}x_i)\\
    \sin(-\nu_{\text{max}}x_i)\\
    \vdots\\
    \cos(\nu_{\text{max}}x_i)\\
    \sin(\nu_{\text{max}}x_i)
\end{bmatrix}.
$$
For the sake of completeness, we quickly review how to apply the multi-angle formulae in this context. These formulae allow us to generate functions of the form $\cos(2 \pi \nu x)$ and $\sin(2 \pi \nu x)$ from linear combinations of $\sin(2\pi x)$ and powers of $\cos(2\pi x)$. In fact, for any $\nu \in \{-\nu_{\max}, \ldots, \nu_{\max}\}$ and $x \in \mathbb{R}$, we have 
\begin{equation}
\label{eq:multiple_angle_cosines}
\cos(2 \pi\nu x) = \sum_{i=0}^{\lfloor 
\nu/2 \rfloor} \sum_{j=0}^i (-1)^{i-j}\binom{\nu}{2i} \binom{i}{j} [\cos(2 \pi x)]^{\nu - 2(i-j)}
= \sum_{k=\nu - \lfloor 
\nu/2 \rfloor}^{\nu} a_k^{(\nu)} \cos^k(2 \pi x),
\end{equation}
for suitable values of $a_k^{(\nu)}$ and where we have used the reindexing $k=\nu - 2(i-j)$ in the second inequality and $\lfloor\cdot\rfloor$ denotes the floor function. 
We can use the following analogous formula for the construction of sine functions at multiple frequencies: 
$$
\sin(2\nu\pi x) = \sin(2\pi x) \cdot \sum_{i=0}^{\lfloor 
(\nu+1)/2 \rfloor} \sum_{j=0}^i (-1)^{i-j} \binom{\nu}{2i+1} \binom{i}{j} [\cos(2 \pi x)]^{\nu - 2(i-j)-1}.
$$
Note that the multi-angle formula for sines does not need powers of sines but only powers of cosines. The above formula can be replicated by combining a linear layer with a network $P^{(2)}$ performing the product.

Taking into account the linear layer corresponding to the multi-angle formulae and the fact that the generation of functions $\cos(2\pi \nu x_i)$ and $\sin(2\pi \nu x_i)$ must be repeated for each dimension $i\in [d]$, we obtain
\begin{align*}
    \text{width}(\bm{v}_{\text{trig} \to \text{multi-freq}}) & 
    = d \cdot \text{width}(\bm{v}_{\text{trig} \to \text{pow}_i})  
    \leq C_1\cdot p_{\ell}^{(1)} \cdot d \cdot \nu_{\text{max}}^2,\\
    \text{depth}(\bm{v}_{\text{trig} \to \text{multi-freq}}) & =
    \text{depth}(\bm{v}_{\text{trig} \to \text{pow}_i}) + \text{depth}(P^{(2)}) \leq C_2 \cdot   p^{(2)} \cdot \log_2(\nu_{\text{max}})
\end{align*}
for a universal constant $C_2 > 0$.

\paragraph{Construction of $\psi_\Lambda$ (Step 3): the module $\bm{v}_{\textnormal{multi-freq} \to \Lambda}$.}
The final module's objective is to generate the rescaled Fourier functions $\Psi_{\bm{\nu}}(\bm{x})$ in \eqref{eq:def_Psi} given the output of $\bm{v}_{\text{trig} \to \text{multi-freq}}(\bm{x})$. Recall that each Fourier function is a product of complex numbers: 
$$
F_{\bm{\nu}}(\bm{x}) 
= \exp(2 \pi i \bm{\nu} \cdot \bm{x})
= \prod_{j \in \text{supp}(\bm{\nu})} \exp(2\pi i \nu_j x_j)
= \prod_{j \in \text{supp}(\bm{\nu})} (\cos(2\pi i \nu_j x_j) + i \sin(2\pi i \nu_j x_j)). 
$$
In particular, the generation of $F_{\bm{\nu}}(\bm{x})$ involves the product of $\|\bm{\nu}\|_0$ complex numbers whose real and imaginary parts are given by the outputs of $\bm{v}_{\text{trig} \to \text{multi-freq}}(\bm{x})$. Moreover, since $\bm{\nu} \in \Lambda = \Lambda^{\text{HC}}_{d,n}$, we have
\begin{equation}
\label{eq:bound_on_norm_0_nu}
n \geq \prod_{j = 1}^d(|\nu_j| + 1) = \prod_{j \in \text{supp}(\bm{\nu})} (|\nu_j| + 1) \geq 2^{\|\bm{\nu}\|_0} \quad \Longrightarrow \quad \|\bm{\nu}\|_0 \leq \log_2 n.
\end{equation}
Therefore, computing each $F_{\bm{\nu}}(\bm{x})$ requires at most $\min\{d, \log_2 n\}$ products.

Now, we describe how to compute products of complex numbers using RePU networks. We would like to compute $\prod_{j=1}^k z_j$ for $z_j = x_j + i y_j\in \mathbb{C}$ for generic values $x_j, y_j \in \mathbb{R}$. To do so, we first find explicit formulas for the real and imaginary parts of this product:
\begin{align*}
\prod_{j=1}^k z_j 
& = \prod_{j=1}^k (x_j + i y_j) 
= \sum_{\bm{j} \in \{0,1\}^k}  \prod_{t=1}^k x_t^{j_t} (iy_t)^{1-j_t} 
= \sum_{\bm{j} \in \{0,1\}^k}  i^{k-\|\bm{j}\|_1} \prod_{t=1}^k x_t^{j_t} y_t^{1-j_t} \\
& = \underbrace{\left(\sum_{\substack{\bm{j} \in \{0,1\}^k\\ k-\|\bm{j}\|_1 \text{ even}}} (-1)^{\frac{k-\|\bm{j}\|_1}{2}} \prod_{t=1}^k x_t^{j_t} y_t^{1-j_t}\right)}_{=\Real\left(\prod_{j=1}^k z_j \right)} 
+ i \cdot \underbrace{\left(\sum_{\substack{\bm{j} \in \{0,1\}^k\\ k-\|\bm{j}\|_1 \text{ odd}}} (-1)^{\frac{k-\|\bm{j}\|_1-1}{2}} \prod_{t=1}^k x_t^{j_t} y_t^{1-j_t}\right)}_{=\Imag\left(\prod_{j=1}^k z_j \right)} .
\end{align*}
This shows that $\Real(\prod_{j=1}^k z_j )$ and $\Imag(\prod_{j=1}^k z_j)$ can be computed as linear combinations of $2^{k}$ products of $k$ real numbers. These correspond to the terms $\prod_{t=1}^k x_t^{j_t} y_t^{1-j_t}$, where we observe that 
$$
x_t^{j_t} y_t^{1-j_t} 
=
\begin{cases}
x_t & \text{if } j_t = 1,\\
y_t & \text{if } j_t = 0.
\end{cases}
$$
Hence, we can compute products of $k$ complex numbers $z_1,\ldots, z_k$ via a RePU network $P^{(k)}_{\mathbb{C}} : \mathbb{R}^{2k}\to \mathbb{C}$ that takes the real and imaginary parts of the complex numbers $z_j$ as inputs and constructed by stacking $2^k$ copies of $P^{(k)}$ atop, adding a linear layer realizing the linear combinations above, and a  final linear layer with complex weights $\begin{bmatrix}1 & i\end{bmatrix}$:
$$
\begin{bmatrix}
x_1\\ y_1 \\  \vdots \\ x_k \\ y_k
\end{bmatrix}
\xrightarrow[\text{of $P^{(k)}$}]{\text{$2^k$ copies}}
\left[
\prod_{t=1}^k x_t^{j_t} y_t^{1-j_t}
\right]_{\bm{j} \in \{0,1\}^k}
\xrightarrow[\text{layer}]{\text{linear}}
\begin{bmatrix}
\Real\left(\prod_{j=1}^k z_j \right) \\
\Imag\left(\prod_{j=1}^k z_j \right)
\end{bmatrix}
\xrightarrow[]{\begin{bmatrix}1 & i\end{bmatrix}}
\Real\left(\prod_{j=1}^k z_j \right) + i 
\Imag\left(\prod_{j=1}^k z_j \right).
$$
Therefore, recalling \eqref{eq:depth_width_bounds_product_net}, we obtain 
\begin{align*}
\text{width}(P^{(k)}_{\mathbb{C}}) & = 2^k \cdot \text{width}(P^{(k)}) \leq C_3 \cdot p^{(1)}_\ell \cdot  2^k \cdot k  \\
\text{depth}(P^{(k)}_{\mathbb{C}}) & = \text{depth}(P^{(k)}) + 1 \leq C_4 \cdot p^{(2)} \cdot \log_2{k},
\end{align*}
for come universal constants $C_3, C_4 >0$. To conclude, the module $\bm{v}_{\textnormal{multi-freq} \to \Lambda}$ is such that 
\begin{align*}
\bm{v}_{\textnormal{multi-freq} \to \Lambda}\left(\begin{bmatrix}\cos(2\pi\nu x_j)\\ \sin(2\pi\nu x_j)\end{bmatrix}_{j \in [d], \nu \in \{-\nu_{\max},\ldots, \nu_{\max}\}}\right) 
& = \left[r_{\bm{\nu}}\cdot P^{(\|\bm{\nu}\|_0)}_{\mathbb{C}}\left(\begin{bmatrix}\cos(2\pi\nu_j x_j)\\ \sin(2\pi\nu_j x_j)\end{bmatrix}_{j \in \text{supp}(\bm{\nu})}\right)\right]_{\bm{\nu}\in \Lambda}\\
& = r_{\bm{\nu}}F_{\bm{\nu}}(\bm{x}),
\end{align*}
where $r_{\bm{\nu}} = 1/(4\pi^2\|\bm{\nu}\|_2^2 + \rho/a_{\bm{0}})$ as desired. Note that the networks $P^{(\|\bm{\nu}\|_0)}_{\mathbb{C}}$,  with $\bm{\nu} \in \Lambda$, can be stacked using a construction analogous to that of Fig.~\ref{Fig:trig->pow}. Also observe that for $\bm{\nu}=\bm{0}$ we simply let $P^{(0)}_{\mathbb{C}} \equiv 1$. In addition, the multiplication by the rescaling factor $r_{\bm{\nu}}$ can be realized by modifying weights in the last layer (we don't need to add a $P^{(2)}$ module because $r_{\bm{\nu}}$ is independent of $\bm{x}$). In summary, recalling that $N = |\Lambda|$ and using \eqref{eq:bound_on_norm_0_nu} and the depth and width bounds for $P^{(n)}_{\mathbb{C}}$, the architecture bounds for the whole network module are
\begin{align*}
\text{width}(\bm{v}_{\textnormal{multi-freq} \to \Lambda}) & = \sum_{\bm{\nu} \in \Lambda} \text{width}(P^{(\|\bm{\nu}\|_0)}_{\mathbb{C}})  \leq N \cdot \text{width}(P^{(\min\{d, \log_2 n\})}_{\mathbb{C}}) \\
&\leq p^{(1)}_{\ell} \cdot C_3 \cdot N \cdot \min\{2^d, n\} \cdot \min\{d, \log_2 n\} \\
\text{depth}(\bm{v}_{\textnormal{multi-freq} \to \Lambda}) & = \max_{\bm{\nu} \in \Lambda} \text{depth}(P^{(\|\bm{\nu}\|_0)}_{\mathbb{C}})
= \text{depth}(P^{(\min\{d,2^n\})}_{\mathbb{C}})
\leq C_4 \cdot p^{(2)} \cdot \log_2(\min\{d,2^n\}).
\end{align*}

\paragraph{Final architecture bounds for $\psi$.} Combining the architecture bounds obtained for each network module in Steps 1, 2 and 3,  and recalling that $\psi(\bm{x}) = \bm{z}^\top \psi_{\Lambda}(\bm{x})$, we obtain
\begin{align*}
    \text{width}\left(\psi\right)
    &=\max\left\{\text{width}(\bm{v}^{(2)} \circ \bm{q}^{(1)}),\text{width}\left(\bm{v}_{\text{trig}\rightarrow\text{multi-freq}}\right),\text{width}\left(\bm{v}_{\text{multi-freq}\rightarrow\Lambda}\right)\right\}\\
    &\leq\max\left\{
    2d, \;
    C_1\cdot p_{\ell}^{(1)} \cdot d \cdot \nu_{\text{max}}^2, \;
    p^{(1)}_{\ell} \cdot C_3 \cdot N \cdot \min\{2^d, n\} \cdot \min\{d, \log_2 n\}
    \right\}\\
    &\leq \min\left\{4n^516^d,e^2n^{2+\log_2d}\right\}\cdot c_{(\ell)}^{(1)}\cdot d.
\end{align*}
Using the hyperbolic cross cardinality bound \eqref{eq:card_bound_HC}, we see that
$$
\text{width}\left(\psi\right)
\leq C_5 \cdot p_{\ell}^{(1)} \cdot \min\left\{4n^516^d,e^2n^{2+\log_2d}\right\} \cdot d \cdot \min\{2^d, n\},
$$
for some universal constant $C_5>0$. Moreover, 
\begin{align*}
    \text{depth}\left(\psi\right)
    & = \text{depth}(\bm{v}^{(2)} \circ \bm{q}^{(1)}) + \text{depth}\left(\bm{v}_{\text{trig}\rightarrow\text{multi-freq}}\right) + \text{depth}\left(\bm{v}_{\text{multi-freq}\rightarrow\Lambda}\right) + 1\\
    & \leq C_6 \cdot p^{(2)} \cdot \left(\log_2(n)+\min\{\log_2 d, n\}\right),
\end{align*}
where $C_6$ is a universal constant. Letting $c_{\ell}^{(1)} = C_5 \cdot p_{\ell}^{(1)}$ and $c^{(2)} = C_6 \cdot p^{(2)}$ yields \eqref{eq:PET_width_bound} and \eqref {eq:PET_depth_bound} and concludes the proof. \hfill $\blacksquare$

\section{Conclusions and open problems}
\label{sec:conclusion}

We have shown a new convergence result for PINNs (Theorem~\ref{thm:PET}) in the form of a practical existence theorem for the numerical solution of high-dimensional, periodic diffusion-reaction problems. This result establishes the existence of a class of periodic PINNs able to achieve the same accuracy as a sparse approximation method (namely, CFC) and using a number of training samples that scales only logarithmically or, at worst, linearly with the PDE domain's dimension $d$. The mild scaling of the sample complexity with respect to $d$ is numerically confirmed through experimentation (see Fig.~\ref{Fig:PINN_sample_dimension}). Our practical existence theorem relies on a new CFC convergence result for diffusion-reaction problems (Theorem~\ref{thm:CFC}) and an explicit construction of periodic PINNs able to replicate the Fourier basis. We have also experimentally confirmed the robustness of periodic PINNs to solve high-dimensional PDEs with respect to the network hyperparameters. Finally, we compared periodic PINNs with CFC, showing that the latter could achieve much higher (but, sometimes, worse) accuracy than the former depending on the sparsity properties of the PDE solution. On the other hand, the performance of periodic PINNs is numerically observed to be consistent across the three examples considered (see Fig.~\ref{Fig:Comparison}).

We conclude by mentioning some gaps between theory and practice and open problems for future research. First, Theorem~\ref{thm:PET} relies on RePU or linear activations and not on more standard ones such as ReLU or tanh. The main reason to work with RePU activations is that they allow to exactly replicate products and, hence, using the construction illustrated in \S\ref{sec:proof_PET}, Fourier functions. The argument of Theorem~\ref{thm:PET} could be generalized to more general activations, such as ReLU or tanh. In that case, however, products (and, hence, Fourier basis functions) could only be approximated and not exactly replicated. This issue could be handled, see \citep{adcock2024learning} and references therein, but it will introduce nontrivial technical difficulties due to the fact that one would have to deal with an approximate CFC matrix, whose error with respect to the true CFC matrix should be carefully controlled in the analysis.

Another important aspect is the presence of sufficient conditions \eqref{eq:diff_expansion} and \eqref{eq:suff_cond_PET} on the PDE coefficients $a$ and $\rho$ in Theorem~\ref{thm:PET}. These are inherited by the CFC convergence analysis (Theorem~\ref{thm:CFC}) and are an artifact of its proof. It was shown in the diffusion equation case that these conditions are sufficient but far from being necessary, see \citep[\S4.4]{wang2022compressive}. This can also be seen by observing that, for $\rho \neq 0$, problem \eqref{eq:diffusion_eq_periodic_1} with coefficients $(a, \rho, f)$ is equivalent to the same PDE with coefficients $(\rho^{-1}a, 1, \rho^{-1}f)$. This invariance to rescaling by $\rho$ is not inherited by condition \eqref{eq:suff_cond_PET}. This is likely to be the case in the diffusion-reaction case as well since the convergence analysis is based on the same argument.

The optimal network $\hat{\psi}$ of Theorem~\ref{thm:PET} is assumed to be trained by \emph{exactly} minimizing a regularized RMSE loss. Of course, this is not what happens in practice, where the loss is usually only approximately minimized using a stochastic gradient descent method. Taking into account the error introduced by the training algorithm in the analysis is an important open question. Moreover, only the last layer of $\hat{\psi}$ is trained, whereas the previous layers are explicitly constructed. This gap between theory and practice is an intrinsic limitation of the argument that Theorem~\ref{thm:PET} relies on. Other types of training loss could also be considered. The regularized loss \eqref{eq:def_loss} considered in this paper is of the square-root LASSO type, which has the advantage of having an optimal choice of tuning parameter $\lambda$ independent of the noise corrupting the samples. However, finding this optimal value in practice through e.g., cross-validation could be more challenging than using a more standard LASSO-type loss (obtained by replacing the RMSE data-fidelity term with an MSE one in \eqref{eq:def_loss})---see \cite{adcock2019correcting, berk2024square}. The theory could be extended to a LASSO-type loss, but the corresponding restrictions on $\lambda$ would be depending on $u$ and $\Lambda$ since the noise corrupting the samples contains the truncation error (recall Step~4 in the proof of Theorem~\ref{thm:CFC}).

Hence, differences exist between the theoretical setting of Theorem~\ref{thm:PET} and the practical implementation of periodical PINNs in \S\ref{sec:numerics}, for which we have followed a setup closer to what is commonly employed in the literature. These differences include the choice of activation, the network's width and depth (which do not saturate the bounds \eqref{eq:PET_width_bound}--\eqref{eq:PET_depth_bound} of Theorem~\ref{thm:PET}), the use of complex-valued weights and the presence of a regularization term in the loss function. It would be interesting to study whether a numerical implementation that follows the setting of Theorem~\ref{thm:PET} more closely  would give any practical benefit over the setup considered in \S\ref{sec:numerics}---see Appendix~\ref{sec:additional_numerics} for some initial results in this direction. On this note, it is also worth observing that the optimal network $\hat{\psi}$ of Theorem~\ref{thm:PET} is sparsely connected and its last layer is approximately sparse due to the presence of $\ell^1$ regularization. Understanding the potential practical benefits of sparsely-connected PINNs is also an interesting avenue of future work.

An important direction of future investigation aimed at bridging the gap with real-world applications is the study of time-dependent PDEs. In order to solve a time-dependent diffusion-reaction equation, a natural option would be to employ an implicit Euler discretization in time. This would, in turn, lead to solving a steady-state diffusion equation of the form \eqref{eq:diffusion_eq_periodic_1} at each time step. In the case of CFC, this scheme could be studied using tools from spectral methods analysis (see, e.g., \cite{canuto2006spectral}). Yet, building a rigorous connection between this approach and a PINN-based approximation of a time-dependent diffusion-reaction equation where the loss involves both space and time variables (as it is typically done) does not appear to be a straightforward task and deserves further research. Moving beyond reaction-diffusion problems is also a possible direction of future work. One could consider, e.g., more general advection-diffusion-reaction operators of the form 
$\mathcal{L}[u](\bm{x}) = - \nabla \cdot (a(\bm{x}) \nabla u (\bm{x})) + \bm{b}(\bm{x}) \cdot \nabla u(\bm{x}) + \rho(\bm{x}) u(\bm{x})$, 
with suitable restrictions on $a(\bm{x})$, $\bm{b}(\bm{x})$, and $\rho(\bm{x})$ to ensure ellipticity. The main task required to extend the theory to this case would be the estimation of the smallest and largest eigenvalues of the Gram matrix $G$ (see Step~1 in the proof of Theorem~\ref{thm:CFC}). While this can be done using the same tools employed in this paper, this is likely to lead to substantially more technically involved computations and to the introduction of extra assumptions on $a(\bm{x})$, $\bm{b}(\bm{x})$, and $\rho(\bm{x})$ for the Riesz property to hold.

An additional limitation of our setting is that it only covers periodic boundary conditions. From a practical viewpoint, PINNs can handle different types of boundary conditions in at least two ways: (i) through the addition of a penalty term in the loss function aimed at enforcing the specific boundary conditions considered; (ii) in the case of homogeneous Dirichlet boundary conditions, by considering modified network ansatzes of the form, e.g., $\psi(\bm{x}) = N(\bm{x}) \cdot \prod_{i=1}^d x_i(1-x_i)$, where $N(\bm{x})$ is a trainable network and the extra factor enforces zero boundary conditions on $[0,1]^d$. 
Compressive spectral collocation methods can also be implemented for non-periodic boundary conditions---see  \cite{brugiapaglia2020compressivespectral} for numerical and theoretical results on steady-state diffusion equations over $[0,1]^d$. Note, however, that the approach in \cite{brugiapaglia2020compressivespectral} is heavily affected by the curse of dimensionality. Similarly to the PINNs case, one could add a penalty term to enforce boundary conditions or modify the spectral basis functions via multiplication by an extra term. Although this is feasible in practice, the corresponding theoretical analysis is expected to be considerably harder than the Fourier case due to the lack of nice algebraic properties of the spectral basis. In summary, extending CFC to non-periodic high dimensional problems and establishing a rigorous link between CFC and PINNs in that setting would be a natural open question stemming from this work.

In summary, there are still important gaps between theory and practice and the convergence theory of PINNs remains a key area of active research. Nonetheless, Theorem~\ref{thm:PET} builds a first important bridge between PINNs' convergence analysis and sparse approximation theory. We hope it will enable further research advances in the theoretical analysis and numerical implementation of physics-informed deep learning in the coming years.

\section*{Acknowledgements}

SB acknowledges the support of the Natural Sciences and Engineering Research Council of Canada (NSERC) through grant RGPIN-2020-06766, the Fonds de Recherche du Qu\'ebec Nature et Technologies (FRQNT) through grants 313276 and 359708, the Faculty of Arts and Science of Concordia University, and Applied Math Lab of the Centre de Recherches Math\'ematiques (CRM). ND acknowledges the support of Florida State University through the CRC 2022-2023 FYAP grant program. WW acknowledges the support of Concordia University through the Horizon Postdoctoral Fellowship program (2021-2023). SB thanks Prof.~Alexander Shnirelman (Concordia University) for helpful discussions about norm equivalences in $H^2(\mathbb{T}^d)$. The authors acknowledge the Digital Research Alliance of Canada for providing the computational resources needed to run the numerical experiments. The authors would also like to acknowledge the anonymous reviewers, whose feedback substantially improved the paper's quality.

\appendix

\section{Additional numerical experiments}

\label{sec:additional_numerics}

This appendix contains additional numerical experiments aimed at testing the numerical impact of some of the constructions adopted in Theorem~\ref{thm:PET}. In particular, we investigate the impact of explicit sparse regularization on the last layer during training and of the use of Fourier features (as opposed to a trainable periodic layer).

\paragraph{The impact of sparse regularization.}

Theorem~\ref{thm:PET} shows the existence of a neural network with an approximately sparse final layer due to the regularization term $\mathcal{R}(\psi)$ employed during training (recall \eqref{eq:regulariz_term}). On the other hand, in \S\ref{sec:numerics}, we trained periodic PINNs without enforcing any sparsity and observed that this tends to produce solutions with relative $L^2$-error between $10^{-2}$ and $10^{-3}$ for all the examples considered (assuming to use enough collocation points). In the experiment of Fig.~\ref{Fig:impact_sparse_regularization}, we test whether explicitly promoting sparsity can further improve accuracy.
\begin{figure}[t]
\centering
  \includegraphics[width=.32\textwidth]{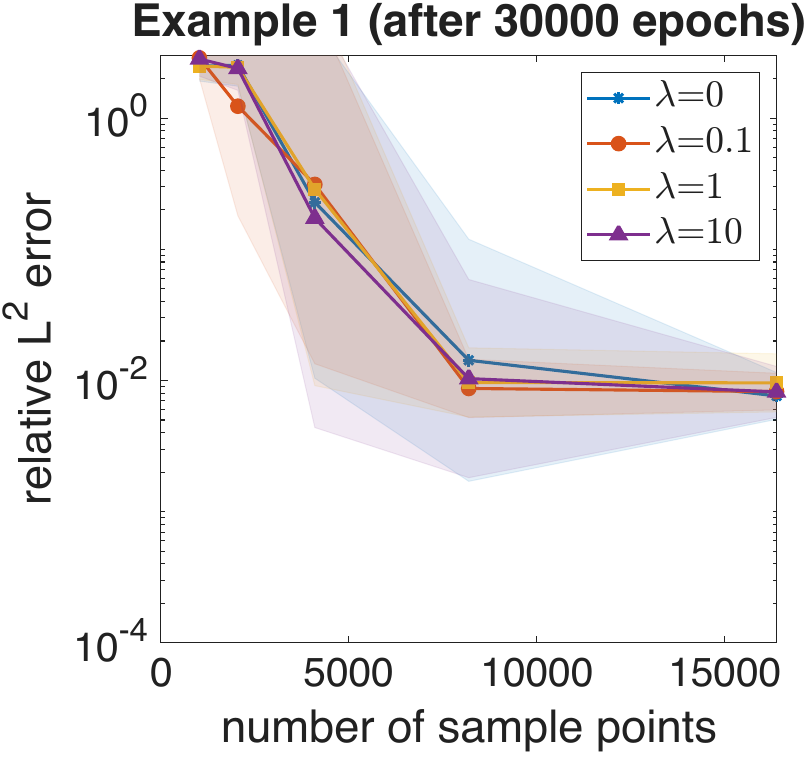}
  \includegraphics[width=.32\textwidth]{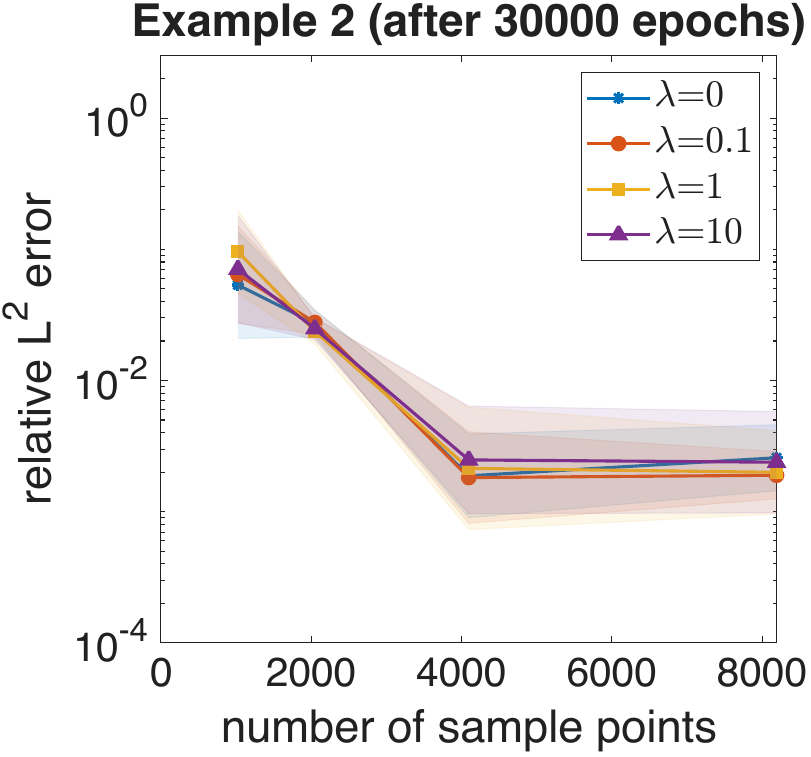}
  \includegraphics[width=.32\textwidth]{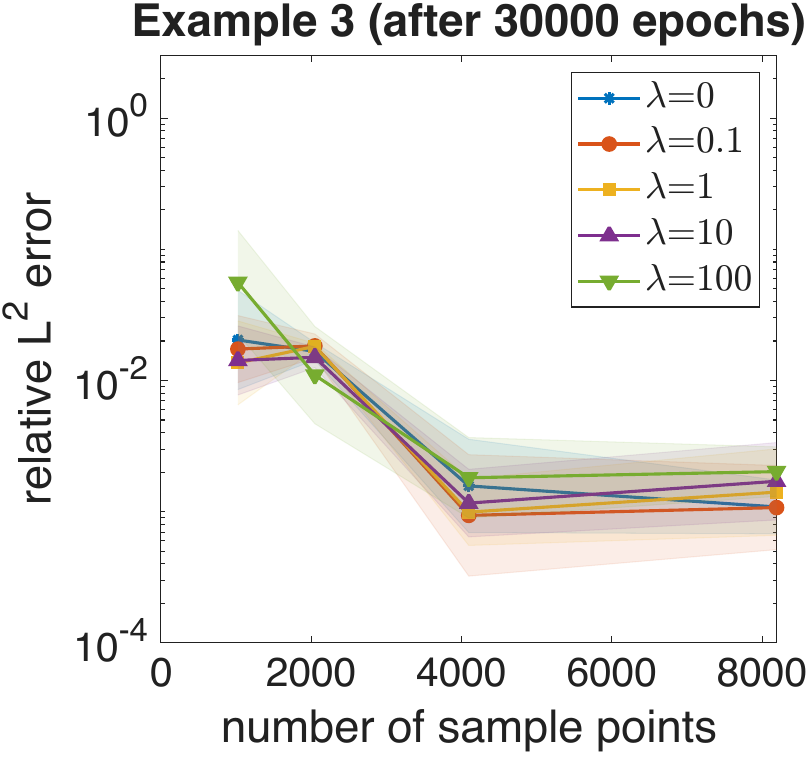}
\caption{(Impact of sparse regularization) Relative $L^2$-error versus training sample count for PINNs trained with with last layer $\ell^1$-regularization and tuning parameter of the form $\lambda=10^{k}$ or without regularization ($\lambda  =0$). Plots from left to right refer to Examples 1--3 ($d=10$), respectively, defined in \eqref{eq:exact1}--\eqref{eq:exact3}.} \label{Fig:impact_sparse_regularization}
\end{figure}

We take the same periodic PINN architecture and training procedure described in \S\ref{sec:numerics} and compare (i) a baseline model with no sparsity penalty, and (ii) the same model with an additional $\ell^1$ penalty term applied to the weights of the last layer. We consider values of the tuning parameter spanning several orders of magnitude, of the form $\lambda = 10^{k}$ for different integer values of $k$. For each $\lambda$, and for each number $m$ of collocation points, we train for a fixed budget of $30000$ epochs and show the corresponding relative $L^2$-error.

We observe two consistent trends. In the low-sample regime (smaller $m$), adding an $\ell_1$ penalty on the last layer yields accuracy comparable to unregularized baseline when the penalty weight $\lambda$ is appropriately tuned, suggesting that sparse regularization does not harm (and could help) data efficiency when the number of sample points is small. As $m$ increases, both the unregularized and $\ell_1$-regularized models converge toward a similar error floor, so the effect of regularization becomes negligible. Overall, across all tested settings the $\ell_1$ penalty with optimal choice of tuning parameter is never worse than the baseline for Examples 1-3.

\paragraph{Trainable periodic layer vs.\ Fourier features.}

Theorem~\ref{thm:PET} constructs a periodic PINN whose first few layers map the input $\bm{x} \in \mathbb{R}^d$ to fixed trigonometric features of the form $(\cos(2\pi \nu x_i), \sin(2\pi \nu x_i))$, for $i\in [d]$ and $\nu \in \{ -\nu_{\max}, \dots, \nu_{\max}\}$ (see Figure~\ref{Fig:NN_architecture}). This construction is the first step towards showing the existence of a class $\mathcal{N}$ of trainable networks with controlled width and depth that can accurately approximate the PDE solution from limited samples. 
A natural question is whether these Fourier features are only a theoretical device or whether it can also be numerically effective when implemented. To address this, we use a frozen set of Fourier features in the first layer instead of learning the periodic layer’s phase parameters as in the PINN tests in \S\ref{sec:numerics}, where the periodic layer in \eqref{eq:def_periodic_layer} was considered. 

The results are shown in Fig.~\ref{fig:non_trainable_periodic_layer}.
\begin{figure}[t]
\centering
  \includegraphics[width=.32\textwidth]{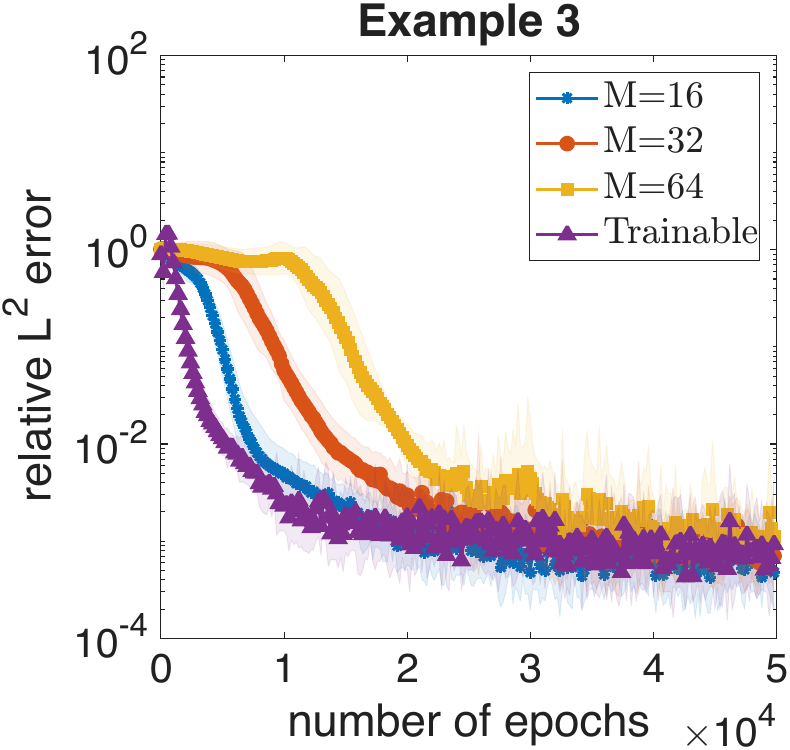}
\caption{(Trainable vs.\ non-trainable periodic layer) Comparison of the standard trainable periodic layer versus non-trainable first Fourier features layer $(\cos(2 \pi \nu x_i), \sin(2 \pi \nu x_i)))$ with $i \in [d]$ and $\nu \in [M]$ and frequency cap $M=\{16, 32, 64\}$. The plot shows the relative $L^2$-error as a function of the number of training epochs.} \label{fig:non_trainable_periodic_layer}
\end{figure}
We fix the first layer to a predetermined set of trigonometric features of the form $(\cos(2 \pi \nu x_i), \sin(2 \pi \nu x_i)))$ with integer frequencies up to a chosen cutoff $M$, i.e., $\nu \in [M]$. The subsequent layers and training are the same as in \S\ref{sec:numerics}. We then compare this ``frozen periodic layer'' model with the standard model where the first periodic layer is trainable by plotting the relative $L^2$ error versus training epochs. All runs use the same settings (exact solution from Example 3, $d=6$, number of sample points $m=10000$).

For a sufficiently large number of epochs, the non-trainable periodic layer model reaches essentially the same final relative $L^2$ error as the PINN with the trainable first periodic layer. We observe that architectures with larger $M$ are harder to train. This is likely due to the fact that the first layer has width equal to $2M$. The final error is slightly better than the baseline for $M = 16, 32$ and slightly worse than the baseline for $M=64$ (although the relative $L^2$-error seems to still be decreasing). This demonstrates that the explicit Fourier basis used in Theorem~\ref{thm:PET} is numerically effective and its performance is comparable to that of PINNs with a trainable periodic layer.

\bibliography{biblio}

\end{document}